\documentclass[10pt]{article}

\usepackage[utf8]{inputenc} 
\usepackage{amsmath,amsfonts,amssymb} 
\usepackage{graphicx} 
\usepackage{cite} 
\usepackage{hyperref} 
\usepackage[margin=0.75in]{geometry} 
\usepackage{titlesec} 
\usepackage{abstract} 
\usepackage{float} 

\usepackage[ruled,vlined,linesnumbered]{algorithm2e}

\SetCommentSty{mycommfont}

\usepackage{amsmath}
\usepackage{amssymb}
\usepackage{amsthm}
\usepackage{bbm}
\usepackage{xfrac} 
\usepackage{makecell} 
\usepackage{longtable} 

\usepackage{caption}
\usepackage{subcaption}
\usepackage{float}
\usepackage{url}

\usepackage{xcolor}
\usepackage{graphicx}

\usepackage[normalem]{ulem} 

\usepackage{tikz}
\usetikzlibrary{positioning}
\usetikzlibrary{arrows.meta}

\tikzset{
    mass/.style={draw, rectangle, minimum width=0.5cm, minimum height=0.5cm},
    spring/.style={thick, decorate, decoration={zigzag, segment length=6pt, amplitude=3pt}},
    damper/.style={thick,decoration={markings, mark connection node=dmp, mark=at position 0.5 with 
    {
        \node (dmp) [thick, inner sep=0pt, transform shape, rotate=-90, minimum width=15pt, minimum height=3pt, draw=none] {};
        \draw [thick] ($(dmp.north east)+(2pt,0)$) -- (dmp.south east) -- (dmp.south west) -- ($(dmp.north west)+(2pt,0)$);
        \draw [thick] ($(dmp.north)+(0,-5pt)$) -- ($(dmp.north)+(0,5pt)$);
    }}, decorate},
    mymassspringdamper/.pic={
        \draw[spring] (0,0) -- (0,-1);
        \node[mass] (m) at (0,-1) {};
        \draw (m.north west) ++(-0.2,0.2) -- ++(0.4,-0.4);
        \draw (m.south west) ++(-0.2,-0.2) -- ++(0.4,0.4);
        \draw[damper] (0,-1) -- (0,-2);
    }
}

\DeclareMathOperator*{\argmin}{arg\,min}
\DeclareMathOperator*{\argmax}{arg\,max}
\newcommand{\pluseq}{\mathrel{+}=}
\newcommand{\mdp}{\textup{MDP}}
\newcommand{\state}{\mathbf{x}}
\newcommand{\action}{\mathbf{u}} 
\newcommand{\statespace}{X} 
\newcommand{\actionspace}{U} 

\newtheorem{definition}{Definition}
\newtheorem{theorem}{Theorem}

\newtheorem{assumption}{Assumption}
\newtheorem*{proof*}{Proof}

\newtheorem{lemma}{Lemma}

\newcommand{\acronym}{SETS}
\newcommand{\name}{Spectral Expansion Tree Search}

\newcommand{\ev}[1]{\mathbbm{E}\left[ #1 \right]}
\newcommand{\prob}[1]{\mathbbm{P}\left[ #1 \right]}
\newcommand{\abs}[1]{\left| #1 \right|}
\newcommand{\ind}[1]{\mathbbm{1}\left[ #1 \right]}
\newcommand{\ceil}[1]{\left \lceil #1 \right \rceil}

\newcommand{\cev}[2]{\mathbbm{E}\left[ #1 \ | \ #2 \right]}
\newcommand{\cprob}[2]{\mathbbm{P}\left[ #1 \ | \ #2 \right]}

\newcommand{\conei}{c^{d}_1}

\newcommand{\ctwoi}{c^{d}_2}

\newcommand{\cthreei}{c^{d}_3}

\newcommand{\cfouri}{c_4^{d}}
\newcommand{\cfourj}{c_4^{d+1}}
\newcommand{\cfivei}{c^{d}_5}
\newcommand{\cfivej}{c^{d+1}_5}
\newcommand{\csixi}{c^{d}_6}
\newcommand{\csixj}{c^{d+1}_6}
\newcommand{\cseveni}{c_7^{d}}
\newcommand{\csevenj}{c_7^{d+1}}
\newcommand{\ceighti}{c_8^{d}}
\newcommand{\ceightj}{c_8^{d+1}}

\newcommand{\coneib}{c^{D-1}_1}

\newcommand{\ctwoib}{c^{D-1}_2}

\newcommand{\cthreeib}{c^{D-1}_3}

\newcommand{\cfourib}{c^{D-1}_4}

\newcommand{\cfiveib}{c^{D-1}_5}

\newcommand{\csixib}{c^{D-1}_6}

\newcommand{\csevenib}{c^{D-1}_7}

\newcommand{\ceightib}{c^{D-1}_8}

\title{\textbf{Monte Carlo Tree Search with Spectral Expansion for Planning with Dynamical Systems}}
\author{
    Benjamin Riviere$^{*1}$, John Lathrop$^{*1}$, Soon-Jo Chung$^{1}$ \\
    $^*$ The first two authors contributed equally to this article. \\ 
    $^1$ Department of Engineering and Applied Science, California Institute of Technology \\
    \textcolor{blue}{This is the accepted version of Science Robotics Vol 9, Issue 97} \\ 
    DOI: 10.1126/scirobotics.ado101, \ \
    \href{https://www.science.org/doi/10.1126/scirobotics.ado1010}{Link to paper}, \ \
    \href{https://www.youtube.com/watch?v=o2ctFPs7OD4}{Link to video}, \ \ 
    \href{https://github.com/aerorobotics/sets}{Link to code}
}
\date{}

\titleformat{\section}{\large\bfseries}{\thesection}{1em}{}
\titleformat{\subsection}{\normalsize\bfseries}{\thesubsection}{1em}{}
\titleformat{\subsubsection}{\normalsize\itshape}{\thesubsubsection}{1em}{}

\begin{document}

\maketitle

\begin{abstract}
\noindent\textbf{Abstract:} 
The ability of a robot to plan complex behaviors with real-time computation, rather than adhering to predesigned or offline-learned routines, alleviates the need for specialized algorithms or training for each problem instance. 
Monte Carlo Tree Search is a powerful planning algorithm that strategically explores simulated future possibilities, but it requires a discrete problem representation that is irreconcilable with the continuous dynamics of the physical world.
We present \name{} (\acronym{}), a real-time, tree-based planner that uses the spectrum of the locally linearized system to construct a low-complexity and approximately equivalent discrete representation of the continuous world. 
We prove \acronym{} converges to a bound of the globally optimal solution for continuous, deterministic and differentiable Markov Decision Processes, a broad class of problems that includes underactuated nonlinear dynamics, non-convex reward functions, and unstructured environments.
We experimentally validate \acronym{} on drone, spacecraft, and ground vehicle robots and one numerical experiment, each of which is not directly solvable with existing methods. We successfully show \acronym{} automatically discovers a diverse set of optimal behaviors and motion trajectories in real time.
\end{abstract}


\section*{Introduction}
\label{sec:introduction}

Endowing robots with high-performing and reliable autonomous decision-making is the ultimate goal of robotics research and will enable applications such as sea, air, and space autonomous exploration, self-driving cars, and urban air mobility.
These autonomous robots need to make decisions encompassing low-level physical movements (i.e., motion planning) and high-level strategy such as selecting goals, sequencing actions, and optimizing other decision variables.

This vision of robotic autonomy remains elusive because solving the decision-making problem exactly in high-dimensional continuous-space systems has ``the curse of dimensionality''~\cite{Bellman1957}.
In light of this complexity, many autonomous robots in deployment avoid a general problem formulation and instead exploit particular problem structure for computational benefits.
For example, motion planning can be solved with sampling-based methods~\cite{lavalle1998rapidly, kavraki1996probabilistic,orthey2023sampling}, trajectory optimization can be solved with convex optimization~\cite{morgan2014model,malyuta2021convex}, and high-level discrete decision-making can be solved with value iteration~\cite{sutton2018reinforcement}.
These methods can also be combined hierarchically for complex behavior, such as autonomous multi-agent inspection of spacecraft~\cite{Nakka_2022b}, robotic manipulation~\cite{kaelbling2011hierarchical,garrett2021integrated}, and self-driving urban vehicles~\cite{paden2016survey, schwarting2018planning}.
Problem-specific solutions are well-understood, can be made computationally efficient, and have proven to be effective.
However, these solutions are limited because they cannot be easily transferred to new problems, adding an expanding burden to the designer for multi-task autonomy.
Additionally, there exist problems that cannot easily be decomposed into computationally tractable sub-problems.

Reinforcement learning is an alternative approach that uses trajectory data to train optimal policies.
Unlike the aforementioned methods, this approach is a generalized procedure that can be applied directly to a broader class of problems.
This property enables important new capabilities such as drone racing~\cite{Song_2023}, helicopter flight~\cite{Abbeel_2006}, grasping~\cite{Lenz_2013} and bipedal locomotion~\cite{castillo2021robust}.
However, these methods usually require an offline training phase and their operation is fundamentally limited in new or changing environments.
In addition, these black-box approaches are unexplainable and provide limited guarantees on optimality, stability, or robustness. 

Some planning and model-based reinforcement learning approaches use tree-based algorithms~\cite{kearns2002sparse, kocsis2006bandit, munos2014bandits} that strategically explore simulated future trajectories from the current state using a family of methods collectively known as Monte Carlo Tree Search (MCTS)~\cite{browne2012survey}. 
In contrast with offline learning methods, MCTS generates high-quality solutions using real time computation: given a system's dynamics and a reward function, the goal is to return the best possible plan with the computational budget available. 
However, whereas the tree's nodes and edges are naturally defined for discrete spaces, the continuous space of robots presents new challenges. 
Uniform spatial and temporal discretization of continuous spaces leads to very large trees and slow convergence rates: a poor discrete representation of the underlying continuous problem.

In this work, we present \name{} (\acronym{}), a real-time and continuous-space planning algorithm that converges to globally optimal solutions.
The new capability of \acronym{} is enabled by efficiently representing continuous space with the system's natural motions, a concept formalized through the spectrum of the locally linearized controllability Gramian.
Without further assumptions on the dynamics or reward, the transformed low-complexity form is solved with Monte Carlo Tree Search (MCTS).
The \acronym{} tree is visualized in Fig~\ref{fig:overview}A. 
Compared to a direct discretization, our method reduces the branching factor and number of decisions in the time horizon, leading to an exponential reduction in tree size.
Our theoretical contribution is twofold: we show our algorithm efficiently discretizes continuous space and we provide a new finite-time convergence analysis of Monte Carlo Tree Search in our setting. 
This analysis proves fast convergence to a bound of the globally optimal solution for planning problems with deterministic and differentiable dynamics, state-dependent and Lipschitz rewards, and continuous state-action spaces.

\begin{figure}
    \centering 
    \includegraphics[width=0.9\linewidth]{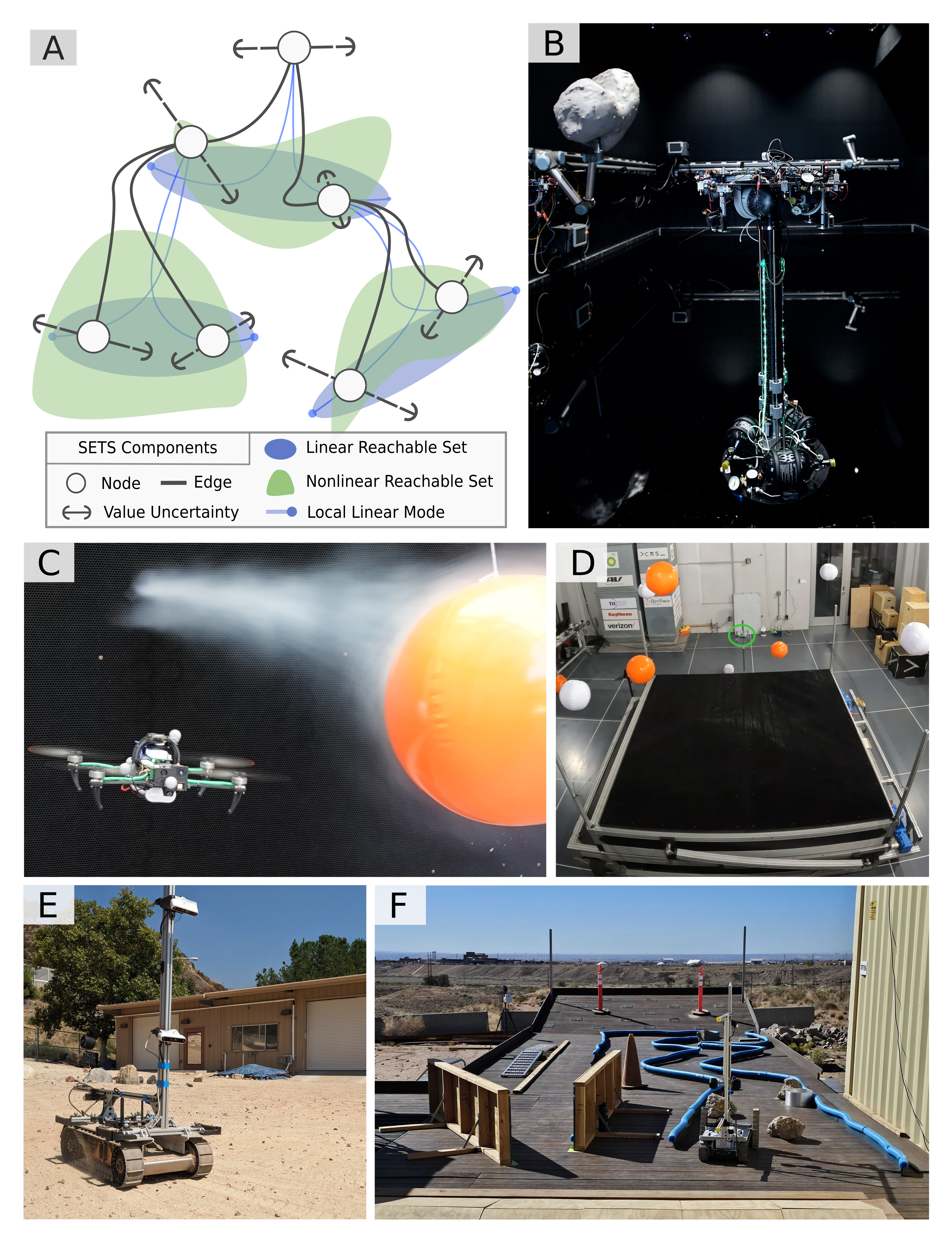}
    \caption{
    (A) Our method, \acronym{}, is a new tree-based planning algorithm for dynamical systems. 
    The tree's edges (shown in gray) are constructed by tracking the spectral modes of the local linearization (shown in blue) with nonlinear feedback control. 
    (B/C/D/E/F) We demonstrate \acronym{} is widely applicable in robotic domains, spanning ground, aerial, and space domains. 
    }
    \label{fig:overview}
\end{figure}

Our hardware and simulation experiments showcase a diverse set of ``discovered-not-designed'' behaviors generated in real time for various robot dynamics and objectives: 
first, a quadrotor solves a dynamically constrained Traveling Salesman Problem to quickly monitor multiple targets in a windy arena, selectively traversing wind gusts and avoiding floating ball obstacles; 
second, a tracked vehicle shares control with a driver through a concourse of ramps, chicane, and sawtooth tracks subject to adversarial degradation;
third, a team of spacecraft use a net to capture and redirect an uncooperative target in a frictionless environment; and
fourth, in simulation, a glider experiencing aerodynamic drag detours into a thermal to extract energy from the environment and survive long enough to achieve its directed task. 
We use this last experiment as a case study to empirically validate the theoretical result, compare our method to state-of-the-art baselines, and tune parameters in a systematic and informed procedure.

\section*{Results}
\label{sec:results}
We present results validating \acronym{} in four experiments, three on different robotic platforms and one in simulation.
The breadth of our experiments highlights the general applicability of \acronym{} on many robotic platforms.
We compare against state-of-the-art baselines in simulation and analyze the empirical results of our theoretical guarantees.

\subsection*{
Quadrotor navigates a dangerous wind field
}
\label{sec:quad}

\begin{figure}
    \centering 
    \includegraphics[width=0.99\linewidth]{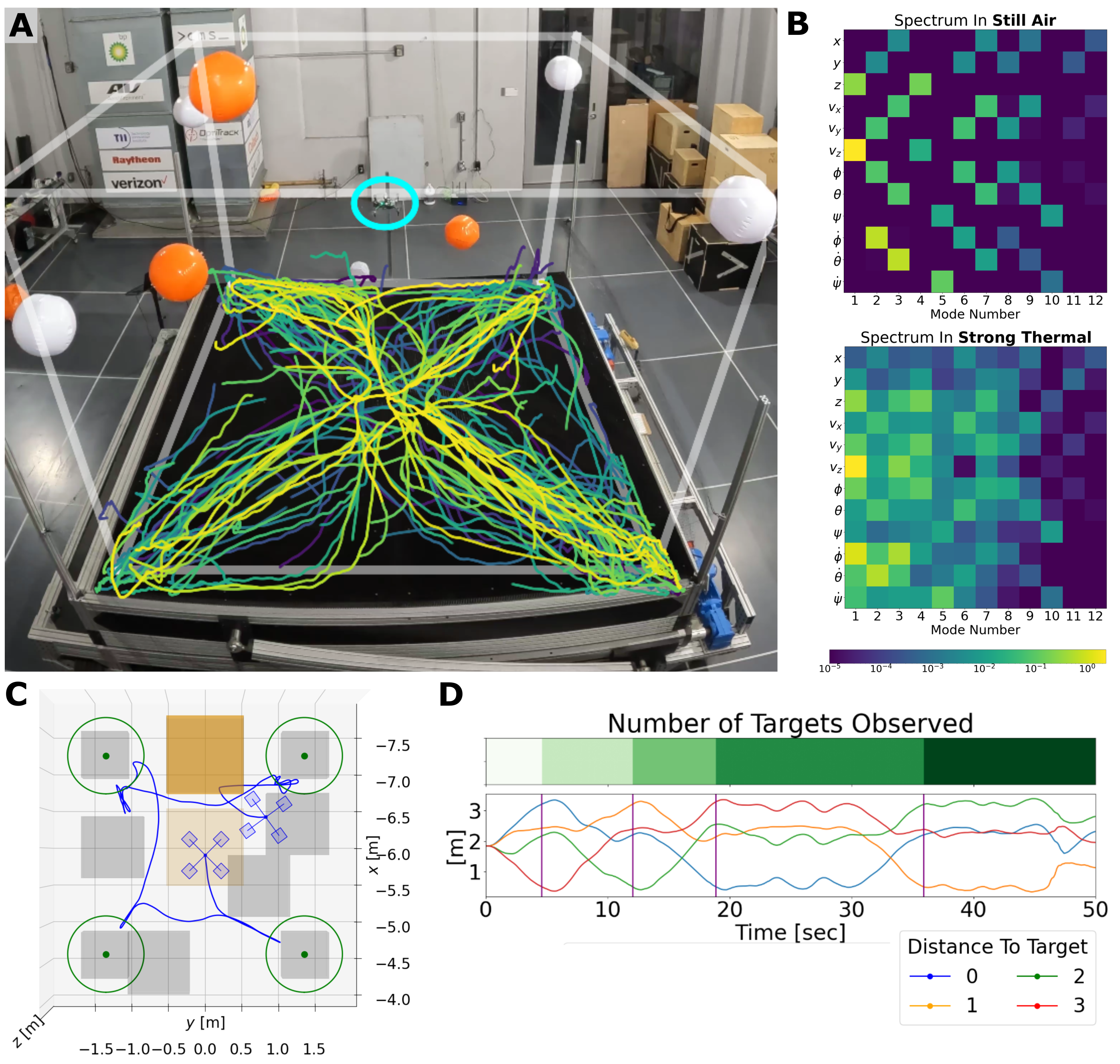}
    \caption{
    (\textbf{A}) 
    \acronym{} enables a drone, circled in blue, to plan trajectories to multiple targets (white) over a fan array and obstacles (orange) in real time. 
    The twelve dimensional search tree is projected onto the two dimensional fan surface. 
    The branches are colored by the order of expansion, with yellow indicating later trajectories. 
    (\textbf{B}) 
    The spectrum of the controllability Gramian is shown for flying in still air and flying through a thermal. 
    Each column corresponds to a natural motion of the system, each row corresponds to a dimension of the state, and each cell is colored by the magnitude of its controllability. 
    (\textbf{C}) A top-down view of the final trajectory, where the targets are shown in green, the thermals in orange, and the obstacles in gray. The thermals are shaded by their relative strength. 
    (\textbf{D}) The distance to each target over time. 
    The mission objective is to visit all four targets and is completed after 37 seconds.
    }
    \label{fig:quad}
\end{figure}

Our first experiment uses \acronym{} to plan trajectories for a quadrotor to monitor multiple targets in a cluttered 3D environment of dangerous wind drafts and moving obstacles and is shown in Fig.~\ref{fig:quad}. 
The experimental arena, shown in Fig.~\ref{fig:quad}A, is a cube with 3 m side lengths where the Real Weather Wind Tunnel in Caltech's Center for Autonomous Systems and Technologies (CAST) generates controlled columns of air to suspend and move spherical obstacles and observation targets. 
In addition to the physical obstacles, there exist dangerous and benign regions of flow of varying speed that affect the planning problem.
A video of the experiment is presented in Movie 1.

This experiment is designed to test the ability of \acronym{} to quickly plan in high-dimensional space subject to the quadrotor dynamics and external forces from aerodynamic interactions.
In particular, the dynamics model used by \acronym{} is a standard quadrotor model augmented with a Deep Neural Network (DNN) to model the learned residual wind force. 
In addition, the algorithm must run in real time to correct model error and to react to new information.
This problem is challenging for existing solutions because the complex interaction between varying wind strengths, dynamics, and obstacles determines path feasibility.
Therefore, the problem is not decomposable into position path planning then tracking control. 
Furthermore, an indicator reward function, a dense obstacle configuration, and multiple goal regions make it challenging to accurately model with conventional motion-planning or optimization-based frameworks. 
Instead, the algorithm must find a global solution by searching through complex dynamics. 

SETS is the planning module of the autonomy stack, generating trajectories of 10 second duration every 5 seconds and running in model-predictive control manner. 
SETS runs in real time by fixing the computation time, planning from the predicted future state, and closing the loop with feedback control. 
In Fig.~\ref{fig:quad}A, the search tree is visualized by projecting the twelve dimensional state trajectories onto the two dimensional surface of the fan array, where the branches are colored by when they were expanded: purple are the first trajectories in the tree and yellow are the last trajectories in the tree. 
Therefore the yellow trajectories serve as an indicator of the plan's progress: we can see that, over computation time, trajectories concentrate at the origin and at each of the target locations.

The \acronym{} tree is constructed with the spectrum of the locally linearized controllability Gramian, shown in Fig.~\ref{fig:quad}B. 
Each column of the spectrum plot is a natural motion (``mode'') of the locally linearized system, and each mode is used to create a branch in the tree search algorithm. 
The modes in still air are interpretable: the first mode of the spectrum corresponds to accelerating the vertical velocity $v_z$, and the next two modes correspond to a pitch and roll maneuver, $\phi$ and $\theta$, respectively. 
The modes in the thermal are more diffuse and are beyond human intuition, yet they are a provably efficient representation of the complex dynamics. 
For example, it becomes much harder for the quadrotor to excite the yaw, $\psi$, degree of freedom independently from the other degrees of freedom.  

The resulting trajectory from this experiment is shown in Fig.~\ref{fig:quad}C, where we see the quadrotor visits all the targets while avoiding obstacles and dangerous winds. 
In order to observe the target, the quadrotor has to be within a sensing radius of 50 cm of the center of the target (visualized in Fig.~\ref{fig:quad}C). 
The physical size of the target creates an obstacle of 30 cm radius, leaving a feasible viewing volume of 0.3 m$^3$ for each target. 
\acronym{}'s ability to quickly find trajectories to these small regions of the position state space while accounting for fast attitude and aggressive wind dynamics demonstrates a sophisticated balance of precision, exploration, and stability. 

The overall mission progress can be described by the the distance to each target over time.
This metric and the cumulative number of targets observed, are shown in Fig.~\ref{fig:quad}D.
The quadrotor visits all the targets in 37 seconds, with the final transition through the narrow corridor of thermals taking the longest time to find a solution. 
The solution's nature is reminiscent of the classical Traveling Salesman Problem~\cite{flood1956traveling}; here \acronym{} discovers the dynamics-dependent cost of traveling between targets and the optimal path to visit them all. 
Discovering this solution automatically, rather than prescribing a sequence of waypoints, has two advantages: 
first, the burden on the designer to enumerate and integrate many behaviors is relieved, extending the operational envelope of the robot and second, \acronym{} may solve problems beyond the intuition of the designer. 
For example, it is not clear to a designer at what fan strength and exposure duration a wind gust becomes too dangerous to cross safely.

\subsection*{
Tracked vehicle with human-in-the-loop
}

\label{sec:linc}

\begin{figure}
    \centering 
    \includegraphics[width=0.75\linewidth]{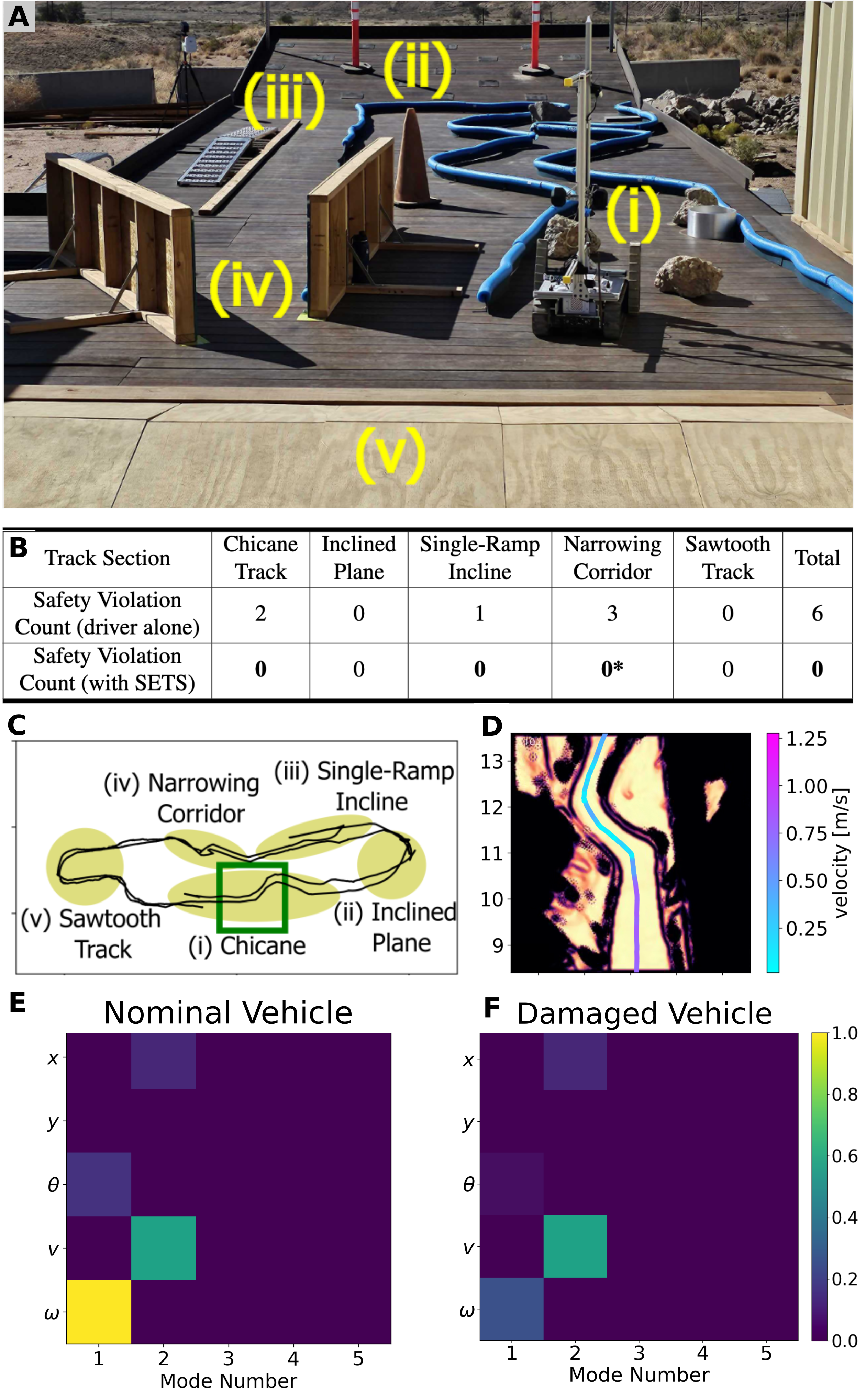}
    \caption{Caption next page.}
    \label{fig:linc_no_caption}
\end{figure}
\addtocounter{figure}{-1}
\begin{figure} [t!]
    \caption{
    (\textbf{A}) The five sections of the test circuit are indicated. 
    (\textbf{B}) We count safety violations of a human driver versus a human driver sharing autonomy with \acronym{} for a selected run through the test course, where a safety violation is a collision with obstacles or a vehicle tip-over. 
    \textbf{*}Due to human error, the autonomy stack was disabled for a few seconds, at which point a collision occurred.
    (\textbf{C}) The trajectory of our vehicle as it completes two loops of the track. Each section of the track is highlighted in yellow. 
    We box in green a segment of the chicane track, shown in (D), where \acronym{} slows the vehicle to prevent a collision with the wall.
    (\textbf{D}) We overlay the path of the vehicle onto a visualization of the hazard map. With no driver input, the intelligent reasoning of \acronym{} slows and turns the vehicle as the course narrows to maintain safety. 
    (\textbf{E}) For the tracked vehicle facing in the positive $x$-direction, we show the spectrum of the local controllability Gramian, where columns correspond to natural motions, rows are states, and each element is colored by its controllability. Under degradation to its actuators (\textbf{F}), the vehicle is limited in its turning ability. The spectrum of the controllability Gramian enables \acronym{} to efficiently interpret actuator degradation, allowing it to make updated plans according to the capabilities of the vehicle.
    }
    \label{fig:linc} 
\end{figure}

Our second experiment, shown in Fig.~\ref{fig:linc}, uses \acronym{} for driver-assist of a tracked vehicle subject to antagonistic terrain effects, tipping constraints, obstacle avoidance, and actuator degradation. 
This human-in-the-loop configuration is also called ``parallel autonomy'' in the self-driving car literature~\cite{schwarting2018planning}.
These experiments were developed during the DARPA Learning Introspective Control (LINC) research program, with various versions of the algorithm tested at the Sandia National Laboratory Robotic Vehicle Range. 
A video of this experiment is presented in Movie 1 and Supplemental Movie 1 and 2.

In the experiment, the tracked vehicle is tasked with assisting a driver to traverse a 20 m $\times$ 10 m test circuit with slippery slopes, a variety of obstacles, and a thin chicane track, shown in Fig.~\ref{fig:linc}A,C. 
In Fig.~\ref{fig:linc}D, we zoom in on one section of the chicane track, in which our algorithm automatically adjusts the driver's command to maintain safety.
Passing through a particularly narrow portion of the track (less than 5 cm clearance), without the need for an updated command from the driver, the robot slows and turns to avoid hitting the walls from a nominal speed of 1 m/s to ~0.25 m/s.
\acronym{} predicts a collision if no diversionary maneuver is taken, then creates and executes an updated plan.

This problem is challenging because of the complex human-robot interaction. 
In our experiments, human pilots prefer an interface that tracks commanded speed, rather than commanded waypoints, meaning a position-space goal region for motion planning does not capture the human-robot interaction well. 
In addition, local optimization-based approaches are inapplicable because exploration is needed to consider nearby candidate trajectories to assist the driver. 
For example, the human driver may command a forward velocity, not understanding there exists an impending collision or risk of tipping over, and instead of coming to a halt, our method explores alternative plans including slowing down, backing up, and turning that ultimately provide the commanded velocity tracking and match the operator intention.
\acronym{} adjusts the pilot's command to maintain safety while maintaining forward progress, displaying intricate maneuvers such as navigating closely around obstacles, decelerating when traversing ridges to avoid tip-over, and executing reverse turning maneuvers in tight corners, all while experiencing adversarial track degradation and time-varying dynamics.

While moving through the track, the vehicle is subject to actuator degradation that scales the control limits of each drive motor by $25\%$ in an alternating sequence of separate degradation and mixed degradation. 
This attack signal is randomly toggled on and off with a period of approximately ten seconds.
\acronym{} is able to efficiently interpret this attack signal through the local spectrum of the tracked vehicle, shown in Fig.~\ref{fig:linc}E,F. 
Here, the locally linearized dynamics have a two-dimensional controllable subspace, one dimension associated with moving forward/backward and one dimension associated with turning left/right.
First, we note the effect of degrading the actuator on the spectrum: the controllability of turning mode is decreased, providing the tree search an expressive interpretation of the current physical situation. 
Second, we note that the mode numbers 3, 4, and 5 have zero magnitude, implying the tracked vehicle is not locally controllable~\cite{Astrom_2008}. 
This illustrates an informal notion of nonlinear global controllability for systems that are not locally controllable:
by stitching together modes from multiple locally linearized systems, \acronym{} creates motions that are not available to a single linearized system. 
For example, a sideways translation can be achieved through the sequential combination of a forward snaking motion followed by a backward snaking motion, reminiscent to motion planning strategies in literature~\cite{murray1991robotic}.

In program demonstrations, the same professional driver drove through the track with and without an autonomy driver assist.
Throughout the experiment, the expert driver was asked to command the vehicle in the same adversarial manner with intentional commands to try to force a safety error. 
We show the number of safety violations in a selected run of the test circuit where adversarial and changing disturbances are present in Fig~\ref{fig:linc}B, where we see \acronym{} outperformed the driver alone and completely prevented safety violations. 
In the chicane section, at 2:36 in Movie 1, SETS causes the robot to turn and avoid a collision with the track, despite a driver command that would cause an impact. 
In the chicane section, at 2:45 in Movie 1, SETS causes the robot to act more conservative, slowing down as it traverses the narrow section.
In addition, during program demonstrations, using \acronym{}, even an inexperienced driver is able to safely navigate the course. 

\acronym{} acts as a planning module which interacts with custom perception, control, and safety algorithms.
Running in model-predictive control fashion, \acronym{} generates trajectories 1.6 seconds into the future every 0.1 seconds.
\acronym{}'s ability to solve problems with general rewards, dynamics, and constraints facilitates its integration with the other autonomy components. 
For example, \acronym{} interacts with a hazard map built with foundational vision models for segmentation to predict traversability.
\acronym{} uses the traversability map to generate a safe plan that takes into account complicated geometric and dynamic information about the robot's environment.
\acronym{} also interacts with an adaptive controller, which compensates for terrain changes and actuator degradation using parameter adaptation to rapidly update the corresponding dynamic function. 
\acronym{} uses the system identification of the adaptive controller to update the dynamics function.
The real-time nature of \acronym{} and its ability to plan over a wide class of dynamics allow it to incorporate nonlinear dynamics updates and automatically benefit from this interaction.

\subsection*{Two collaborative spacecraft redirect debris}

\begin{figure}
    \centering 
    \includegraphics[width=0.9\linewidth]{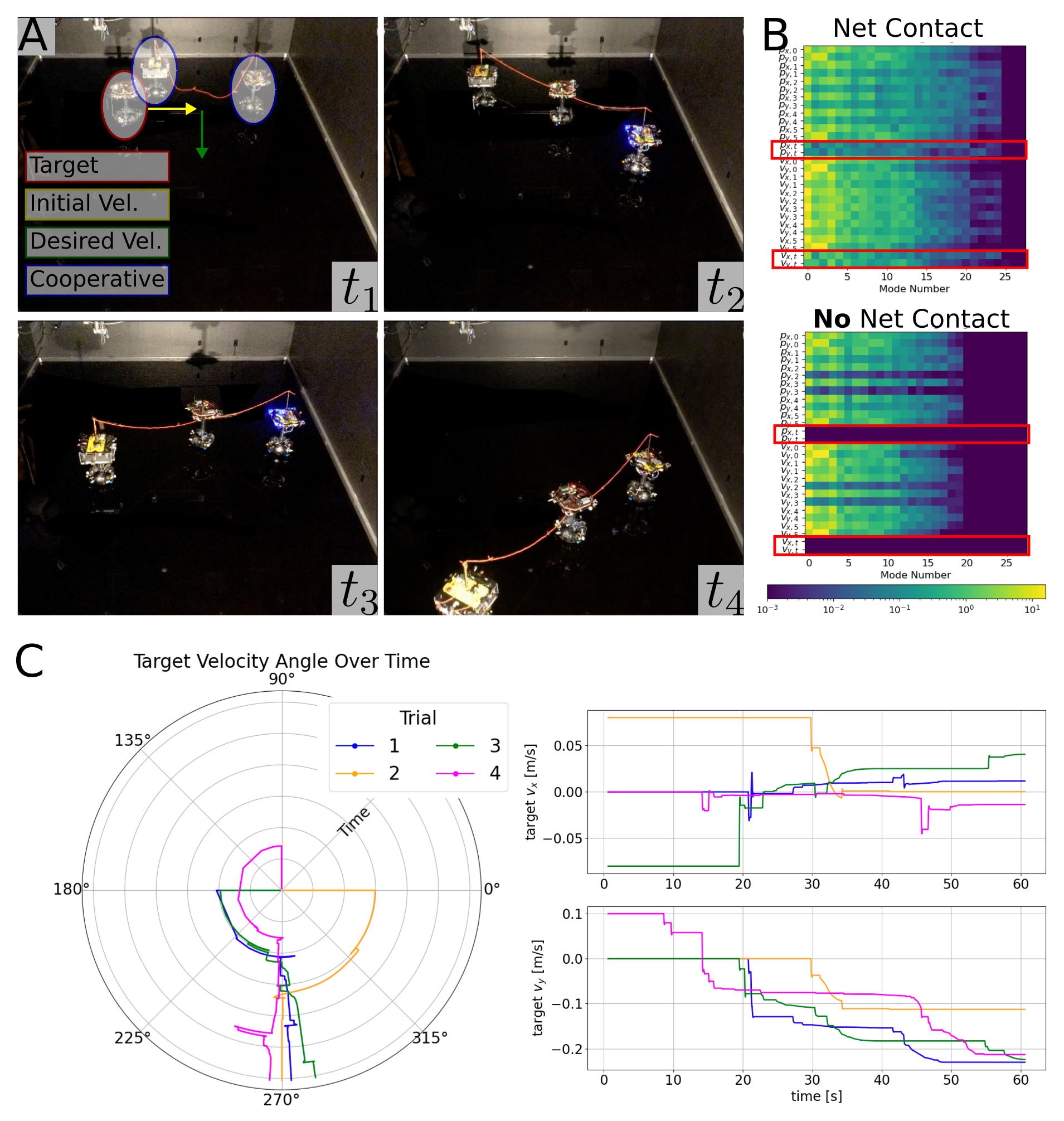}
    \caption{
    (\textbf{A}) Still shots of the spacecraft experiment. The two tethered spacecraft arrest the motion of the target spacecraft.
    Planning through the tether-contact model, the controlled spacecraft shepherd the target spacecraft in the desired direction, toward the camera.
    These frames have been modified to reduce glare off the floor.
    (\textbf{B}) The spectrum of the controllability Gramian captures the high-dimensional dynamics of this problem. At the moment of contact, the states associated with the target (outlined in red) become controllable.
    (\textbf{C}) The direction of the target spacecraft velocity over the course of each experiment.
    In each case, the motion of the target spacecraft is correctly redirected.
    }
    \label{fig:spacecraft}
\end{figure}

In the third and final hardware experiment, we deploy \acronym{} on a team of two spacecraft to capture and redirect a piece of space debris with a tether, reminiscent of a NASA mission concept to capture and redirect a near-Earth asteroid~\cite{nasaredirect}.
This experiment tests the ability of \acronym{} to coordinate multiple robots and plan through high-dimensional nonlinear dynamics induced from contact forces and tether dynamics.

The experiment, depicted in Fig.~\ref{fig:spacecraft}A, takes place in an arena with a smooth and ultra flat epoxy floor, where spacecraft robots hover on air bearings and are actuated with onboard thrusters to simulate a frictionless space environment. 
\acronym{} controls the two cooperative spacecraft tethered together. 
The third spacecraft is not controlled by \acronym{} and models an uncooperative target. 
The cooperative spacecraft are tasked to arrest the motion of the target and shepherd it to exit the arena in a desired direction.
\acronym{} predicts the motion of the three bodies and the tether,  which is modeled with a spring-mass-damper finite-element-model, while only controlling the thruster inputs of the two cooperative spacecraft.
This high-dimensional and underactuated problem tests the ability of \acronym{} to plan for complex dynamics over a horizon.
In particular, the two controlled spacecraft can only affect the dynamics of the target through contact with the center of the net, and the center of the net can only be influenced by the thrusters after propagating through the lattice elements of the tether. 
This chain of dynamics requires deliberate maneuvering of the controlled spacecraft to redirect the target.

Despite these challenges, \acronym{} automatically generates correct behavior due to its efficient representation of the dynamics: in Fig.~\ref{fig:spacecraft}B, the controllable modes of the system capture this networked structure of the finite element tether model.
The most controllable states are those associated with the outermost nodes in the network, which are exactly the controlled spacecraft.
The controllability of the nodes in the net decreases towards the center of the structure.
Importantly, the spectrum reveals that actuating the target is impossible before contact. 

\acronym{} interprets this discrete representation of the dynamics to efficiently find near-optimal solutions.
Four trials are tested by varying the target's initial position and velocity. 
In the first, the target spacecraft is stationary and \acronym{} finds a trajectory that sequentially deploys, captures, and redirects. 
In the second and third configurations, the target is initialized moving parallel to the initial cooperative spacecraft configuration, and quick motion is required of the tethered spacecraft to move into its way for capture. 
In the fourth, the target is initialized moving towards the initial cooperative spacecraft configuration, and they execute a ``trampoline-like'' maneuver, pulling the tether taut to springboard the target back in the desired direction. 

The trajectory data for each case are shown in Fig.~\ref{fig:spacecraft}C. 
The right plots show the absolute velocities over time for each trial.
The polar plot, which is coordinate-aligned with the snapshots in Fig.~\ref{fig:spacecraft}A, shows the relative angle between the $x$-and $y$-components of the target's velocity converge to 270$^{\circ}$, which is the desired mission behavior. 
Successfully controlling this high-dimensional and underactuated system in real time is only possible with \acronym{} and can ultimately enable new mission concepts.  

\subsection*{
Aerodynamic glider performs persistent observation
}
\label{sec:glider}

\begin{figure}
    \centering 
    \includegraphics[width=0.99\linewidth]{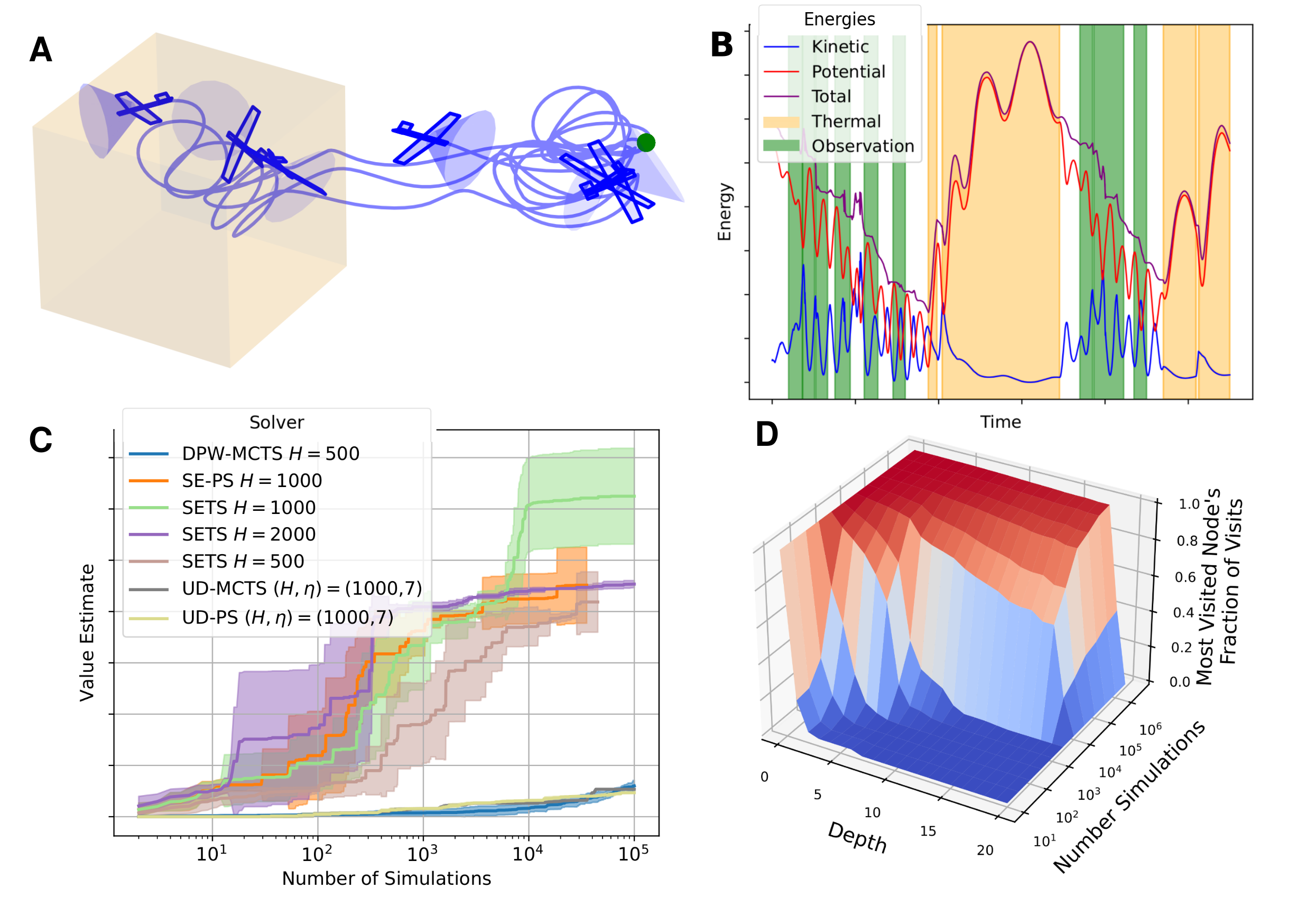}
    \caption{
    (\textbf{A}) The glider (blue) discovers an oscillation strategy between observing the target (green) to maximize its objective, and passing through the thermal (orange) to gain energy and maintain altitude.
    (\textbf{B}) Kinetic and potential energy of the glider over time. In nominal air conditions, the glider loses energy to drag and would eventually crash, but flying through the thermal (orange) provides upwards force and allows the glider to accumulate potential energy.
    (\textbf{C}) Value convergence: as the number of simulations grows, the value estimate increases. As predicted by theory, the branch length parameter $H$ controls the trade-off between convergence and error. Error bars shown are one standard deviation. We also plot the value estimate of each baseline. 
    (\textbf{D}) Policy convergence: as the number of simulations grows, the visits concentrate deeper in the tree and the plan becomes more refined.
    }
    \label{fig:glider}
\end{figure}

In this numerical experiment, we use \acronym{} on a 6 degree-of-freedom glider, which is tasked to observe a target in the presence of a thermal updraft, shown in Fig.~\ref{fig:glider}.
We analyze tree data to make observations about the policy and value convergence. 
We also use this experiment to compare against external baselines and make ablation studies.

The environment is a $2\times2\times1$ km$^3$ volume arena and takes place over 10 minutes in simulation time, shown in Fig.~\ref{fig:glider}A. 
The glider dynamics, parameters, and aerodynamic coefficients are from the Aerosonde UAV~\cite{beard2012small}. 
The interaction of aerodynamics, the observation objective, and the environmental thermal makes this a challenging planning problem: 
In order to generate enough lift to counteract gravity, the glider needs to average approximately 30 meters per second forward velocity. 
At the same time, drag drains the system's kinetic energy and, without exploiting the thermal, would cause the glider to crash into the ground after about 4 minutes. 
However, the glider can save itself by flying into the thermal, extracting energy from the environment, and resuming its observation task. 
\acronym{} discovers this solution in real time, running in model-predictive control fashion and generating 100 seconds of trajectory every 45 seconds. 
The trade-off between kinetic and potential energy in this periodic solution are visualized in Fig.~\ref{fig:glider}B. 
This periodic solution is challenging to generate with existing motion-planning or optimization frameworks because modeling the problem with a single static goal region would either cause the glider to crash or never observe the target. 
Furthermore, detouring to the thermal once at the target requires aggressive exploration and is in contradiction to goal-seeking behavior of local optimization.

The data for the two subfigures, Fig.~\ref{fig:glider}C,D, are generated by running \acronym{} from a given state, and analyzing the search tree. 
In Fig.~\ref{fig:glider}C, we plot \acronym{}'s root node value estimate versus the number of simulations. 
As predicted by our analysis in Sec.~Theoretical Results, our algorithm's branch length hyperparameter, $H$, controls the trade-off between convergence rate and error: short branches have slow convergence to small steady-state error, and long branches have fast convergence to large steady-state error. 
The $H$ $=$ $500$ case induces such a high complexity problem that it is unable to overtake the $H=1000$ estimate within the operating regime of $10^5$ simulations ($10$ minutes of wall-clock time). 
This plot informs our decision to select the optimal branch length $H=1000$.

We use the same value estimate data from the tree in the baseline study. 
This study is designed to isolate the effects of representation and exploration strategy. 
For choices of representation (low-level construction of nodes and edges), we implement
spectral expansion (SE), uniform discretization (UD), and double progressive widening (DPW)~\cite{couetoux2011continuous}. 
The SE method is our approach, the UD approach uniformly discretizes the action space with $\eta$ discrete points per dimension, and the DPW method adaptively samples from the action space, with more samples given to promising nodes. 
For choices of exploration strategy (high-level search on a set of existing nodes and edges), we implement Monte Carlo Tree Search (MCTS) (Algorithm~\ref{algo:spectral_search}) and predictive sampling (PS)~\cite{howell2022predictive}, where PS is a uniform sampling approach that returns the best trajectory found. 
The baselines and our method in Fig.~\ref{fig:glider}C are a permutation of these representation and exploration strategies. 
The baseline data in Fig.~\ref{fig:glider}C is generated by running each solver over $10$ random seeds for branch lengths $H=100,500,1000,2000$ and discretization levels $\eta=3,5,7,11$, where only the best variant of each baseline is shown. 
Other than our approach, the only algorithm that performs well is spectral expansion with predictive sampling, (SE-PS). 
In fact, SE-PS $H=1000$ outperforms SETS $H=500$, suggesting that representation, not exploration strategy, is the most important component of optimal planning for dynamical systems.

In Fig.~\ref{fig:glider}D, we define a measure of the tree's confidence at a particular depth as the most visited node's fraction of the total visit count, and plot it versus the depth and size of the tree. 
The relative visitation frequency is a metric for the tree's confidence in its action selection because of the Upper Confidence Bound rule of Monte Carlo Tree Search~\cite{kocsis2006bandit}: 
actions are selected by their optimistic value estimate so when the visitations are concentrated to a single action the optimistic value estimate of all the other actions is less than that of the highly visited (and well-estimated) selected action.
As the number of simulations increases, concentration occurs deeper in the tree, indicating the plan is being refined further into the future, and the tree has reached sufficient confidence in the plan until that point. 
When running in receding-horizon fashion, the time between replanning steps should be large enough to allow a highly confident plan to develop.

\section*{Discussion}
\label{sec:discussion}

We present a qualitative comparison with existing approaches and discuss how our work can impact the field of robotics.

\subsection*{Optimization and gradient-based planning}
Some planning problems can be solved with convex optimization techniques~\cite{morgan2014model, malyuta2021convex}.
Taking advantage of fast numerical solvers (e.g.~\cite{mosek}), this family of methods has been shown to solve, in real time, a variety of challenging dynamical problems including quadrotor drone racing~\cite{foehn2021time}, in-orbit assembly\cite{foust2020autonomous}, bipedal locomotion~\cite{hereid20163d}, swarm coordination~\cite{morgan2016swarm} and swarm coverage~\cite{schwager2011unifying}.
A variety of techniques, including sequential convex programming~\cite{morgan2014model,morgan2016swarm}, collocation~\cite{hereid20163d}, and single and multiple shooting methods~\cite{mayne1966second, bock2000direct, li2004iterative} enable this method to be applied to non-convex problems,
though theoretical analysis of these methods has shown convergence is limited to a local minima~\cite{dinh2010local, bonalli2019gusto}.
The general planning problem's state, input, dynamics, and reward specifications can each create local minima traps under which a pure exploitation strategy will fail to find the global optimal solution. 
The classical example is the bug trap problem~\cite{sucan2012}, where the obstacle configuration will cause methods that do not explore to converge to a highly sub-optimal point. 
Recent work~\cite{marcucci2023motion} has proposed avoiding local minima in motion planning by partitioning the environment into convex sets and solving a relaxed mixed-integer program.

Compared to these approaches, \acronym{} provides globally optimal solutions by performing a discrete search on a carefully constructed representation of natural motions. 
For example, the quadrotor avoids local minima traps from obstacles and wind gusts and the glider temporarily ignores the high reward of visiting the target and instead navigates to the thermal to maintain altitude and energy constraints.

\subsection*{Sampling-based motion planning}

The standard robotics approach to overcome the nonconvexities inherent in problem data, such as traps from obstacle configurations, is sampling-based search~\cite{orthey2023sampling}.
Foundational methods such as Probabilistic Roadmaps (PRM)~\cite{kavraki1996probabilistic}, Rapidly Exploring Random Trees (RRT)~\cite{lavalle1998rapidly} and their variations (e.g.~RRT*~\cite{Karaman_2011_ijrr}) sample configurations and use a local planner to build trees and graphs on which to perform global optimization. 
Although RRT was originally designed for kinodynamic planning, these planners are most commonly used as geometric path planners.

Applying sampling-based planning for dynamical systems has two challenges: 
first, the higher-order states increase the dimensionality of the search space, resulting in slow convergence and poor performance and
second, these methods rely on the existence of a local planner capable of solving arbitrary two-point boundary value problems. 
Addressing these challenges is an active area of research. 
For example, Stable-Sparse-RRT (SST)~\cite{Li2016} relaxes the knowledge of a local planner by sampling control inputs to generate dynamically feasible paths, then prunes redundant edges to maintain a sparse tree. 
Another work~\cite{poccia2017deterministic} considers quasi-Monte Carlo sampling to speed convergence of tree search for control-affine and driftless control-affine systems.
Discontinuity-bounded A*~\cite{hoenig_2022} is similar to our work in that they plan for high-dimensional dynamical systems over a discrete structure of motion primitives, and is different because their motion primitives are generated offline by sampling input and goal states. 
The generation and integration of motion primitives into discrete planning has a rich history~\cite{frazzoli2002real, frazzoli2005maneuver, saveriano2021dynamic}. 

As in our work, the Differential Fast Marching Tree (DFMT) method~\cite{schmerling2015optimal} uses the controllability Gramian for planning with dynamical constraints. 
Whereas DFMT uses the Gramian to verify reachability of a given two point boundary value problem, with their terminal point uniformly sampled from the state space, our work explicitly uses the spectrum of the Gramian to generate the terminal point in a provably efficient exploration strategy. 
This distinction is significant because uniform sampling poorly covers high dimensional spaces. 
In addition, whereas DFMT considers linear systems and a single global Gramian, we consider nonlinear systems, construct a local Gramian at each node, and bound the linearization error with an additional tracking control procedure. 
Other inspired works that blend control and planning are LQR-trees~\cite{tedrake2009lqr} and funnel library planning~\cite{majumdar2017funnel}, both of which rely on sum-of-square programming to compute regions of attraction upon which the discrete planner searches. 

Most continuous space planning algorithms scale poorly to high-dimensional spaces because they generate nodes and edges by sampling from the state or action space. 
To the best of our knowledge, \acronym{} is the first to use the spectrum of the locally-linearized system to generate nodes and edges, resulting in a branching factor that is linear in the state dimension. 
This conceptual difference manifests itself in new search capabilities: to the best of our knowledge, there do not exist kinodynamic motion planners that \emph{search} in real time through high-dimensional nonlinear dynamics such as a 12-dimensional quadrotor with DNN-modeled wind effects. 

In addition to the differences in search strategies, our work differentiates itself in the generality of the problem assumptions. 
Whereas motion planning problems from the robotics community focus on static, position-space goal regions~\cite{karaman2010optimal}, the planning problem defined by the computer science community~\cite{moerland2023model} relaxes the problem setting.
This is useful in instances where a goal region does not capture the essence of the problem, or when a goal region is difficult or impossible to define.
For example, in chess, a notion of a goal region could be ``the set of all positions in which the opponent's king is checkmated'', but this set is challenging to enumerate or path to with a motion planning method.
With this in mind, the experiments in this paper are selected because they are easily formulated as decision-making problems, but they cannot be cast as motion planning.
For example, in the tracked vehicle experiments, human drivers preferred interfacing with an algorithm that tracks their commanded velocity rather than navigates to a waypoint, implying a conventional motion planning framework with a position goal region does not model the human-robot interaction well.
In the glider experiment, a static goal region does not exist because the optimal behavior is to oscillate between observing the target and the thermal.
Besides immediately growing the set of feasible robot behaviors, the more general problem framework places decision-making as a more natural foundation to extend to stochastic dynamics~\cite{auger2013continuous}, partially observable~\cite{Silver_2010, sunberg2018online, Ragan_2023}, and game-theoretic~\cite{Lisy_2013} settings.

There is a recent body of work on sampling-based model predictive control methods, including cross-entropy motion planning~\cite{kobilarov2012cross} and model predictive path integral control~\cite{williams2016aggressive}.
By sampling many trajectories then performing a weighted average, these methods have shown impressive performance in scenarios including autonomous driving~\cite{williams2016aggressive} and manipulation~\cite{bhardwaj2022storm}.
These methods have also shown straightforward integration with high-fidelity simulators~\cite{pezzato2023sampling} and learned dynamics models~\cite{williams2016aggressive}.
These strategies traditionally perturb inputs with Gaussian noise to promote exploration. 
However, without careful tuning of noise distributions and averaging temperature, this can lead to slow exploration or the generated policies getting stuck in local minima.
Similar to our work, it is possible to consider sampling over abstractions such as splines in state space~\cite{bhardwaj2022storm}, and these are active areas of research. 
Whereas these methods explore by perturbing trajectories with Gaussian noise, our approach uses a theoretically rigorous balance of exploration and exploitation, and we prove this results in convergence to globally optimal solutions.

Sampling-based search is also investigated in the partially observable case.
Algorithms for handling continuous state and action spaces have focused on particle filtering methods, with specializations to overcome the challenges inherent in uncertain observations such as sharing observation data in the tree~\cite{garg2019despot}, progressive widening~\cite{sunberg2018online}, or linear filtering~\cite{Ragan_2023}.
We bring our attention to the fully observed case to focus on the difficulty of decision making, rather than combined decision making and information gathering.
There is recent work that transforms a stochastic planning problem into a deterministic planning problem~\cite{nakka2022trajectory}, moving the difficulty of dealing with stochastic dynamics and observations into same decision-making framework we consider.

\subsection*{Reinforcement learning}
Our problem setting is closest to planning in model-based reinforcement learning: once the dynamics and reward model are learned, the planning component finds the optimal value and policy.
The classical planning methods using dynamics and reward models are value and policy iteration~\cite{Bellman1957, sutton2018reinforcement}. 
The fundamental issue with these methods, which persists in modern approaches, is that representing a high-dimensional continuous space has a high complexity~\cite{SUTTON1991353, moerland2023model}. 
The direct approach is to discretize the state and action space and run value iteration, but this has a storage complexity that is exponential in the state dimension, making these methods computationally infeasible for high-dimensional robotics applications. 
Research has developed in multi-grid~\cite{chow1991optimal} and adaptive-grid~\cite{munos2002variable} representations, but the practical gains in scaling to high-dimensional systems are marginal. 
Recent work~\cite{gorodetsky2018high} has proposed a tensor-based method that exploits a low-rank decomposition of value functions, enabling efficient discretization and improved policy generation on systems of comparable dimensionality to our experiments.

An alternate reinforcement learning approach is that of model-free methods that directly learn correlations between observations and optimal actions without explicitly using a dynamics or reward function~\cite{polydoros2017survey, arulkumaran2017deep}.
Policy gradient methods have offered a powerful technique to handle continuous state and action spaces, with methods such as Proximal Policy Optimization~\cite{schulman2017proximal} becoming a standard tool in the community.
In special cases, such as linear-quadratic regulation~\cite{fazel2018global}, global convergence results exist.
However, the general convergence of these algorithms in continuous space is limited to local stationary points~\cite{zhang2020global}. 
Additional work has gone on to classify these stationary points between problems, drawing conclusions about structural similarities that help policy gradient methods converge globally~\cite{bhandari2024global}.
These methods, while powerful and general, require large datasets and an offline training phase. 
These methods suffer fundamentally from domain shift, the difference between the offline training environments and the environment of the deployed system. 
In contrast to data-driven reinforcement learning methods, our algorithm is able to run on a never-before-seen problem with guaranteed global convergence to an optimal solution, thus avoiding the danger of domain shift.

A natural direction to develop real-time intelligence is via anytime algorithms that can be stopped at an arbitrary point, returning satisfactory solutions that increase in quality as more time is given. 
In the world of decision-making, the prototypical example is Monte Carlo Tree Search~\cite{browne2012survey}, an algorithm that simulates random trajectories while biasing towards actions of high reward. 
Exhaustive~\cite{schmerling2018multimodal} and uniformly random searches~\cite{howell2022predictive} have been shown to be effective in some scenarios, but the improved strategic exploration of MCTS enables convergence in problems where simpler techniques fail to efficiently search the combinatorially large space.
Although MCTS can perform very well in traditional artificial intelligence settings such as games, the complexity of the high-dimensional and continuous world presents a fundamental challenge. 
Random search using simulators~\cite{howell2022predictive} has been investigated, but the investigated search strategies scale poorly. 
Directly sampling the action space~\cite{auger2013continuous}, even with sophisticated strategies~\cite{kim2020monte,williams2017model}, induces trees with large width and depth: sampling high-dimensional continuous action spaces creates a high branching factor and discretizing time creates a large number of decisions over the horizon.

Temporal abstraction, or options~\cite{SUTTON1999}, is a principled framework to make decisions over sequences of actions and reduce the complexity of decision-making. 
Options are analogous to motion primitives in motion planning. 
There has been work in option construction in planning~\cite{Lee-RSS-21} and using MCTS~\cite{Bai_2016,de2016monte}, where options are hand-crafted.
There has also been work in option construction~\cite{jamgochian2023constrained} in the partially observable setting. 
An interpretation of our method is that it automatically generates options in a dynamically informed and provably correct framework.

Data-driven methods~\cite{sutton1999policy}, and combinations with gradient-based techniques~\cite{deisenroth2011pilco,levine2013guided,riviere_2020}, have also been deployed for decision-making.
However, their reliance on large amounts of demonstration data during an offline training phase limits their applicability to systems and scenarios where the complete problem data is known ahead of time.
In contrast, our method can be deployed on a never-before-seen problem, and, for any allowed computational budget, produce a plan of approximately optimal decisions that increases in quality with more time.

\subsection*{Significance}

We expect \acronym{} to positively impact the field of autonomous robots and research in decision-making. 
From a design perspective, \acronym{} provides many important advantages: 
first, \acronym{} solves a broad class of MDPs and therefore interfaces naturally with other autonomy components and can be used for new and diverse tasks and systems. This relieves the burden on the designer and extends the operational envelope of autonomous behavior. 
Second, because the search tree of \acronym{} can be visualized and analyzed, it has a high degree of explainability and can be tuned and verified by the user. 
Third, because \acronym{} is efficient enough to run in real time, it can react to new information on the fly with real-time computation. 
For these reasons, we believe \acronym{} can be a default choice for planners in a wide range of autonomy applications.

From a research perspective, \acronym{} builds an important connection between dynamical systems and machine learning. 
To develop this connection, we interpret the tree policy as an online learning process of a categorical distribution on each node's children. 
Our first theoretical result Theorem~\ref{thm:mdp_spectral_search_value_estimate}, which applies to a broad class of dynamical systems, suggests the spectrum of the local controllability Gramian can be used as provably correct and widely applicable learning features. 
These features enable real-time learning for complex dynamical systems by simplifying the decision-making problem to selecting among a set of natural motions of the system.
Although beyond the scope of this work, we believe these features can be used for offline policy learning, kinodynamic motion planning, and other sampling-based planning, providing reduction in computational complexity and improved convergence rates.

From the algorithmic complexity perspective, \acronym{} reduces complexity in two separate mechanisms as compared to the uniform-discretization approach:
First, as \acronym{} plans over trajectories instead of individual actions, the total tree depth is decreased by a factor of $H$, where $H$ is the duration of a trajectory generated by the spectral nodal expansion. 
Second, the width of the tree (branching factor) is decreased from an exponential dependency on control dimension to a linear dependency on state dimension. 
This is enabled because each child node uniquely tracks one of the basis vectors of the controllable subpsace and, from linear theory, the number of basis vectors is upper bounded by the state dimension. 
In contrast, for a uniform discretization, the number of elements in the $m$-cube has an exponential dependence on $m$. 
The number of combinations in the tree, (i.e. width to the power of depth) is important because it appears in the convergence rate analysis of MCTS. 
This fact is revealed by our second theoretical result, Theorem~\ref{thm:mcts_convergence}, which is, to the best of our knowledge, the first complete MCTS value convergence result that does not assume knowledge of privileged problem information (the gap). 

\section*{Methods}
\label{sec:methods}

In this section, we present the problem formulation, algorithmic overview, and theoretical analysis of convergence to global optimality. 

\subsection*{Problem formulation}
We consider a Markov Decision Process (MDP)~\cite{puterman2014markov}, an abstract problem description written as a tuple of components: $\left<\statespace, \actionspace, F, R, D, \Omega, K, \gamma \right>$. 
Here $\statespace \subseteq \mathbbm{R}^{n}$ is a compact state space, $\actionspace \subseteq \mathbbm{R}^{m}$ is a compact action space, and $F:\statespace \times \actionspace \rightarrow \statespace$ are the discrete-time dynamics.
$R:\statespace \times \actionspace \rightarrow [0,1]$ is the stage reward and $D:\statespace \rightarrow \mathbbm{R}_{\geq 0}$ is the terminal reward.
$\Omega \subseteq \statespace$ is a set of unsafe states, $K \in \mathbbm{N}$ is the time horizon, and $\gamma \in [0,1)$ is the discount factor. 

At an initial state $\state_0$, the planning problem is to select a sequence of actions that maximizes the sum of the stage reward plus the terminal reward, subject to the dynamics and state/action constraints: 
\begin{align}
\label{eq:value_function}
V^*, \state_{[K]}^*, &\action_{[K]}^* = \underset{\state_{[K]}, \ \action_{[K]}}{\operatorname{argmax}}
\sum_{k=1}^{K} \gamma^k R(\state_{k}, \action_{k+1}) + \gamma^K D(\state_K) \\
\text{s.t.}  
\quad \state_k &= F(\state_{k-1}, \action_k), 
\quad \state_k \in \statespace \setminus \Omega, 
\quad \action_k \in \actionspace,
\quad \forall k \in [1,K] \nonumber
\end{align}
where $k$ is the timestep index and a bracket subscript indicates a sequence, e.g. $\state_{[K]} = [\state_1^\top, \state_2^\top, \dots, \state_K^\top]^\top \in \mathbb{R}^{nK}$.

Similar frameworks of decision-making over a horizon are considered by the reinforcement learning and optimal control community. 
The space of problems we consider is focused on smooth, deterministic, and nonlinear dynamics with continuous state and action spaces.
Because the systems we consider operate in continuous time, we discretize time uniformly so each discrete timestep $k$ index has duration of $\Delta t$ seconds.

\subsection*{Monte Carlo Tree Search}
\label{sec:mcts} 

\SetAlFnt{\small}
\begin{algorithm}

\caption{\name{} (\acronym{}) 
} 
\label{algo:spectral_search}

\SetKwProg{Def}{def}{:}{}

\Def{$\textup{SpectralExpansionTreeSearch}(\state_0, \mdp, H)$}{
    \tcc{set root}
    $i_0 = \text{Node}(\state_0)$ \;
    \For{$\ell=1,\hdots,$}{
        \tcc{initialize path at root then rollout}
        $p = [i_0]$ \; 
        \For{$d=1,\hdots,\lceil K/H \rceil$}{
            \tcc{best child (if not fully expanded, choose randomly)}
            \vspace{-2.5pt}
            $i^* = \argmax_{i \in C(p[-1])} 
            V(i,\ell) 
            + c_1 \frac{T(p[-1],\ell)^{c_3}}{T(i,\ell)^{c_2}}$ \label{algo_line:ucb} \;
            \vspace{-2.5pt}
            \uIf{$i^*$ \textup{is not expanded}}{
                \label{lin:spectral_expansion} $\textup{SpectralExpansion}(i^*, \ \state_0, \ \bar{\action}, \ \mdp, \ H)$ \;
            }
            \lIf{$i^* \in \Omega$}{\textup{break}}
            $p\text{.append}(i^*)$ \; 
        }
        \vspace{-3pt} 
        \tcc{backup}
        \For{$d=1,\dots,|p|$}{
            $\tilde{V}(p[d], \ell) \pluseq \sum_{t=d}^{\lceil K/H \rceil} \gamma^{t H} r(p[t])$ \; 
            $T(p[d], \ell) \pluseq 1$ \; 
        }
    }
    $\textbf{yield} \ V(\state_0, \ell)$ \;
}

\Def{$\textup{SpectralExpansion}(i, \ \state_0, \ \bar{\action}, \ \mdp, \ H)$}{
    \tcc{compute local linearization}
    $A_k, B_k, c_k = \textup{Linearization}(F, \ \state_0, \ \bar{\action}_{[H]})\ \quad \forall k \in [0, H-1]$ \; 
    $\textbf{z}_{k+1} = L_k(\textbf{z}_k, \action_{k+1}) = A_k \textbf{z}_k + B_k \action_{k+1} + c_k \quad \forall k \in [0, H-1]$ \; 
    \tcc{compute spectrum of normalized controllability Gramian}
    $S = \text{diag}(\{ \frac{2}{\overline{u}_j - \underline{u}_j} \ | \ \forall j \in [1,m] \})$ 
    \label{line:rescaling} \; 
    $\mathcal{C}= 
    \begin{bmatrix}
        \left(\prod_{k=1}^{H-1} A_k \right) B_0 S, &
        \left(\prod_{k=2}^{H-1} A_k \right) B_1 S, & 
        \dots, & 
        A_{H-1} B_{H-2} S, & 
        B_{H-1} S
    \end{bmatrix}$ \label{line:controllability_matrix}\;
    $[\mathbf{v}_1, \mathbf{v}_2, \dots, \mathbf{v}_n], [\lambda_1, \lambda_2, \dots, \lambda_n] = \textup{eig}(\mathcal{C}  \mathcal{C}^\top)$ \;
    \tcc{compute linear reference trajectory for the $i$th branch}
    $\mathbf{z}_H = (-1)^{i\%2} \sqrt{\lambda}_{i//2} \mathbf{v}_{i//2} + \left(\left( \prod_{k=0}^{H-1} A_k \right) \mathbf{z}_0 + \sum_{k=0}^{H-1} \left( \prod_{j=k}^{H-1} A_j \right) c_k \right) $ \;
    $\action^{\mathrm{ref}}_{[H]} = \text{clip}(\mathcal{C}^{\dagger} \mathbf{z}_H, \actionspace) $ \;
    $\mathbf{z}^{\mathrm{ref}}_{[H]} = L^H(\mathbf{z}_0, \action^{\mathrm{ref}}_{[H]}) $ \;
    \tcc{track reference trajectory with nonlinear system}
    $M_k = \text{DARE}(A_k, B_k, \Gamma_x, \Gamma_u)$ \; 
    $\mathcal{K}_k = (\Gamma_u + B_k^\top M_k B_k)^{-1} B_k^\top M_k A_k$ 
    \label{line:controller}\;
    $\action_{[H]} = \{ \text{clip}(\action^{\mathrm{ref}}_{k} - \mathcal{K}_{k-1} (\state_{k-1} - \mathbf{z}^{\mathrm{ref}}_{k-1}), \actionspace) \ | \ \forall k \in [1,H] \}$ \; 
    $\state_{[H]} = F^H(\state_0, \action_{[H]}) $ \;
    \textbf{return} $(\state_{[H]}, \action_{[H]}, R^H(\state_0, \action_{[H]}))$ \; 
}
\end{algorithm}

The pseudocode for \acronym{} is shown in Algorithm~\ref{algo:spectral_search}. 
\acronym{} performs a Monte Carlo Tree Search (MCTS) with a specialized nodal expansion operator, defined as Spectral Expansion in the pseudocode. 
We briefly summarize the procedure of MCTS, with more in-depth descriptions available in the literature~\cite{browne2012survey}:
while the robot has remaining computational budget, simulate future state trajectories from the current world state forward to the horizon of the MDP. 
At each node, the tree policy selects the best child by balancing exploration of visit counts and exploitation of observed reward. 
If a node is not fully expanded, generate a new child by taking an action and stepping forward in time.
When a simulated trajectory terminates, the accumulated reward and visit count information is backpropagated up the tree to update the total value and the number of visits at each node in the path, and the process iterates. 
We adopt for following notation for the pseudocode: 
$p$ is the ``path'', the list of nodes in one rollout, 
$i$ is a single node, 
$r(i)$ is the reward to node $i$, 
$C(i)$ is the set of node $i$'s children,
$\tilde{V}(i, \ell)$ is the total value of node $i$ after $\ell$ rollouts, 
$T(i, \ell)$ is the number of visits of node $i$ after $\ell$ rollouts, 
$V(i,\ell)$ is the average value of node $i$ after $\ell$ rollouts. 

Whereas the standard MCTS variant~\cite{kocsis2006bandit} uses logarithmic exploration term, \acronym{} uses a polynomial exploration term (see Line~\ref{algo_line:ucb} of Algorithm~\ref{algo:spectral_search}), with constants $c_1=1$, $c_2=0.5$ and $c_3=1$. 
This technique is supported by our theoretical analysis in Theorem~\ref{thm:mcts_convergence}, as well as other MCTS variants, both in theory~\cite{auger2013continuous,shah2020non} and practice~\cite{Silver_2017}.
Although the dynamics, reward, and spectral expansion operator are deterministic, SETS is a stochastic algorithm because, when multiple children of a node have not been visited, the algorithm randomly selects one. 
This stochasticity, which is the same as the original MCTS algorithm, is noted in Line~\ref{algo_line:ucb} of Algorithm~\ref{algo:spectral_search}.

\subsection*{Spectral Expansion Operator}

The Spectral Expansion operator, specified in Algorithm~\ref{algo:spectral_search} Lines 16-29, computes a trajectory of length $H$. 
The first step is to compute the linearization and eigendecomposition of the local controllability Gramian. 
We consider both time-invariant and time-varying linearizations, where the linearized system data is either computed once at the current state and a single nominal control input, or over a time-varying trajectory initialized at the current state, subject to a nominal control sequence. 
In practice, we found that time-varying linearization produces more stable trajectories, especially for highly agile platforms such as a quadrotor, and we select the unforced dynamics as the default nominal control input, ($\bar{\action} \equiv \mathbf{0}$). 

From the linearized system, we construct the controllability matrix $\mathcal{C}$, which is a linear mapping from sequences of control actions to the resulting terminal state.
This matrix and its associated Gramian are well-known in linear control theory~\cite{boyd1994linear}, with two important properties:
first, for linear systems, the set of states reachable with one unit of control energy is an ellipse parameterized by the spectrum of the Gramian 
and second, the pseudoinverse of the controllability matrix applied to a feasible desired terminal state computes the minimum energy control input that drives the system to that state. 

In Lines 19-20 of Algorithm~\ref{algo:spectral_search}, the inputs are scaled by their control limits before computing the Gramian. 
This procedure scales the notion of bounded-energy of the reachable set to the control limits rather than inputs that potentially differ by orders of magnitude.
An alternate preconditioning matrix $S$ could be selected here to create a different trade-off in control energy.
In addition to informing the input rescaling, the elliptical reachable sets property reveals a simple exploration strategy for linear systems: 
the vertices of the ellipse, which are computed via the spectrum of the Gramian, form a bounded covering of the reachable set. 


In Lines 22-24, we compute the linearized reference trajectory. 
First, we select the desired state for this branch using the spectrum of the controllability Gramian. 
We iterate through the modes using the integer floor division and modulus operators, $//$ and $\%$, respectively.
Visiting each mode in the plus and minus direction imply a total branching factor of $2 n$, where $n$ is the state dimension. 
The pseudoinverse of the controllability matrix mapped onto the desired state yields a trajectory for the linearized system that, by the psuedoinverse mapping property, is minimum energy and reaches the target. 
However, the notion of minimum and bounded energy is over the entire trajectory and, to verify each individual control actions are valid, 
we impose the bounded input constraint with a clip operation: given a vector and an interval, the values outside the interval are clipped to the interval edges.


Whereas the reference trajectory is dynamically feasible for the locally linearized system, the actual branch trajectory must satisfy nonlinear dynamics. 
In Lines 25-28, we compute a feedback controller from the Discrete Algebraic Riccati Equation (DARE)~\cite{zhou1998essentials} and rollout a trajectory of the nonlinear system that tracks the linear reference trajectory to the desired mode. 
In Line 28-29, we use the multi-step notation for the dynamics and reward, e.g.: $F^H(\state_0, \action_{[H]})=F(\dots F(F(\state_0, \action_{1}), \action_{2}), \dots, \action_{H})$. 
We include a proof of the controller stability in Lemma 14 in the Supplementary Materials.


\subsection*{Heuristics}

At the cost of an increased branching factor, the user can add manually designed nodal expansion heuristics to guide the search.
In the quadrotor experiments, we use a heuristic to guide the system to the nearest target by including the projection of the nearest target onto the linear set as a branching option.
This is similar to goal biased sampling used in the sampling-based motion planning community~\cite{sucan2012}. 
For the remaining experiments, we did not use heuristics and relied only on the natural exploration of \acronym{}. 
In a manner similar to AlphaZero~\cite{Silver_2017} and our own work of Neural Tree Expansion (NTE)~\cite{riviere2021neural}, it is also possible to incorporate \emph{DNN-based heuristics} to predict the value of the next child state. 

The user can also manually decrease the branching factor by prioritizing certain degrees of freedom.
In the quadrotor experiment, we found that only searching among the velocity and angular velocity modes led to higher performance.
Physically, the system maintains its ability to translate and make attitude adjustments, as the eigenvectors that maximally excite the velocity and angular velocity coordinates also create changes in the position and attitude, respectively. 
We expect this reasoning to be applicable in many second-order systems, and we apply it in all of our experiments. 
The final implementation detail is that we return the maximum valued trajectory (rather than the child of highest average value). 
This practice maintains theoretical guarantees because the maximum valued trajectory always has a higher value than the average value, due to our deterministic setting.

\subsection*{Theoretical Results}
\label{sec:theory_results}

Now we present our two main theoretical results which are combined to show that SETS quickly converges to a bounded-error solution of the planning problem~\eqref{eq:value_function}. 
The proofs and supporting lemmas are in the Supplemental Material.

We make the following assumptions: 
\begin{assumption}
    \label{assumption:mdp_se}
    The dynamics $F$ are twice differentiable, the reward $R$ is Lipschitz and depends only on the state, and the input set $U$ is bounded in a product of intervals.
    \begin{gather*}
        F \in C^2(\statespace \times \actionspace; \statespace) 
        \qquad
        R \in \textup{Lip}_1(\statespace; \mathbb{R})
        \qquad 
        \actionspace \subseteq \{ \action \in \mathbbm{R}^{m} \ | \ \underline{\action}_j \leq \action_j \leq \overline{\action}_j \}
    \end{gather*}
where $\textup{Lip}_1(\statespace; \mathbb{R})$ is the space of Lipschitz functions from $X$ to $\mathbb{R}$.
As $\statespace$ is compact and $R$ is Lipschitz, $R$ is therefore bounded. Without loss of generality, we suppose $R(\state) \in [0,1]$.
\end{assumption}

Our first result, Theorem~\ref{thm:mdp_spectral_search_value_estimate}, states that spectral expansion creates a bounded equivalent discrete representation of continuous MDPs.

\begin{theorem}
\label{thm:mdp_spectral_search_value_estimate}
Consider an MDP $\left<\statespace, \actionspace, F, R, D, K, \gamma \right>$.
For initial state $\state_0$, Spectral Expansion with horizon $H$ creates a discrete representation with a bounded equivalent optimal value function:
\begin{align}
    | V^*(\state_0) - V^*_\textup{SETS}(\state_0) | 
    &\leq
    \frac{\kappa_0 + \gamma^H }{1-\gamma^H}\left(\kappa_1 (1+\kappa_2 \Delta t)^H + \kappa_3\right),
\end{align}
\noindent where $\kappa_{0,1,2,3}$ are problem-specific constants, and $\Delta t$ is the discretization of continuous-time dynamics.
\end{theorem}

To derive this theorem, we first show that the distance between optimal value functions of two discounted MDPs is upper bounded by a constant times the distance between their reachable sets. 
Then, we show that the finite number of trajectories generated by Spectral Expansion efficiently cover the reachable set of the continuous nonlinear system. 
We achieve this by introducing two intermediate sets: the reachable set of the linear system subject to energy constraints and the collection of states constructed from the spectrum of the controllability Gramian. 
This procedure is visualized in Fig.~\ref{fig:method_se}.

\begin{figure}
    \centering
    \includegraphics[width=0.9\linewidth]{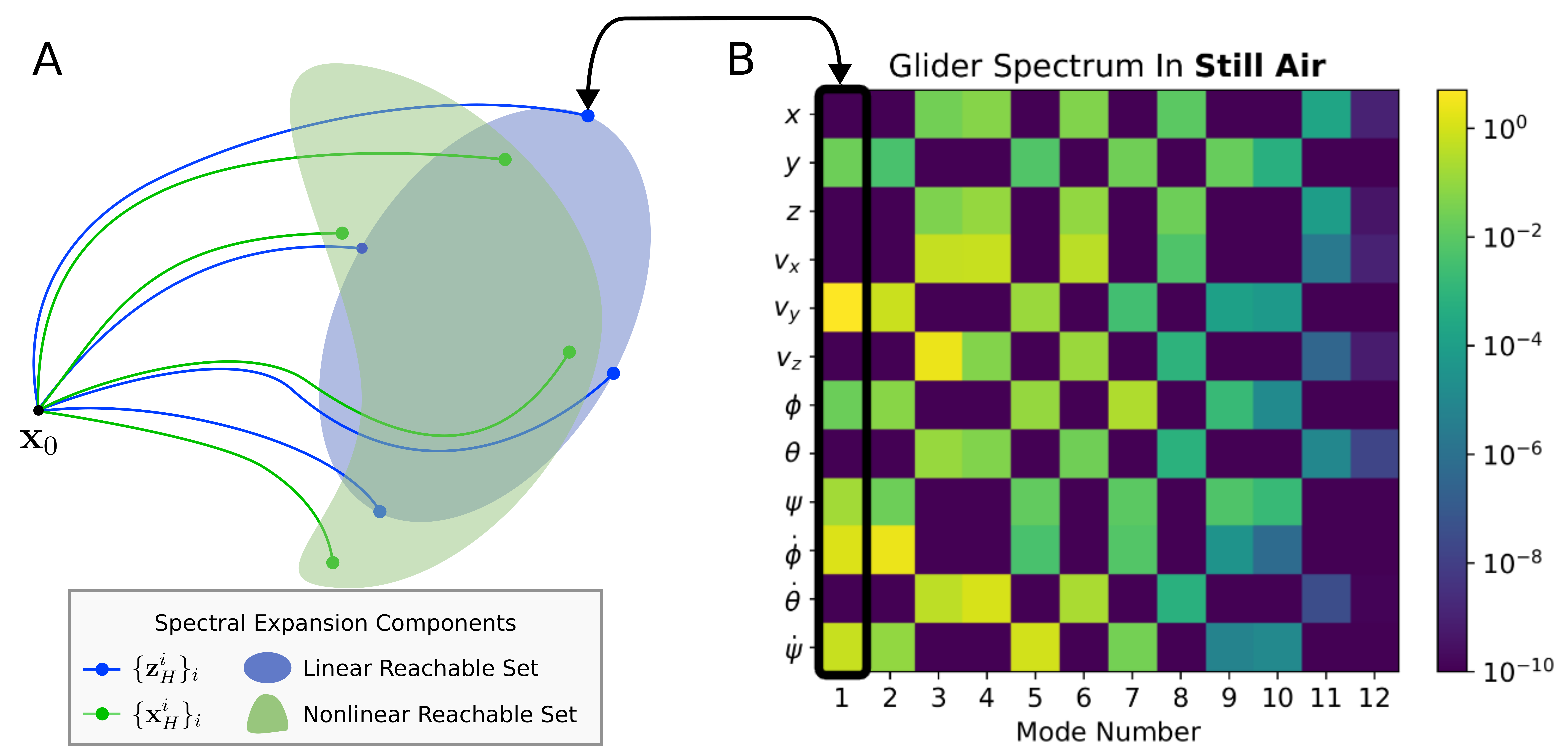}
    \caption{Spectral Expansion: 
    (\textbf{A}) The Spectral Expansion Operator: 
    The nonlinear reachable set is shaded in green and the locally linearized reachable set is shaded in blue. 
    The modes of the controllability spectrum are the axes of the ellipse $\{\mathbf{z}_H^i\}_{i=1}^{2n}$, and these are used to compute SETS trajectories $\{\state_H^i\}_{i=1}^{2n}$ with contraction-based feedback control. 
    (\textbf{B}) We show the glider's spectrum, where each column corresponds to a natural motion, each row corresponds to a state dimension, and each element is colored according to its controllability. We include this figure here to connect the theory to the data of the three experiments: each column of the image in (\textbf{B}) corresponds to a controllable mode $\mathbf{z}^i_H$ in (\textbf{A}). 
    }
    \label{fig:method_se}
\end{figure}

Our second result, Theorem~\ref{thm:mcts_convergence}, states that the value estimate of MCTS converges to the optimal value function for discrete MDPs. 

\begin{theorem}
\label{thm:mcts_convergence}
Consider an MDP $\left<\statespace, \actionspace, F, R, D, K, \gamma \right>$ with a finite action set. Consider a node $i$ at depth $d$ in the tree, where the tree is grown using Algorithm~\ref{algo:spectral_search}. 
Let $T(i,\ell)$ be the number of visits to node $i$ after $\ell$ iterations and let $V(i,\ell)$ be the average value after $\ell$ iterations. 
Then, there exist non-negative constants such that the expected value converges to the optimal value at node $i$: 
\begin{align}
    \abs{ V^*(i) - \cev{V(i,\ell)}{T(i,\ell) = \tau} } \leq \frac{\kappa_4}{\tau^{\kappa_5}}
\end{align}
and the distribution of rollout values sampled in the tree concentrates to the expected value. $\forall z > 0$, $\forall \tau \geq 1$:
\begin{align}
    \cprob{ \abs{ \ev{V(i,\ell)} - V(i,\ell) } \geq \frac{z}{\tau^{\kappa_6}} }{T(i,\ell)=\tau} \leq \frac{\kappa_7}{z^{\kappa_8}}
\end{align}
Our analysis shows $\kappa_5 = \kappa_6 = 0.5$ under our choice of exploration law.
\end{theorem}

Our analysis draws from the recent convergence analysis in the literature~\cite{shah2020non}. 
We are particularly inspired by their proof strategy of backwards induction and non-stationary bandits. 
However, we study a fundamentally different problem setting with deterministic rewards. 
Deterministic rewards, the standard setting for robotics, isolates the source of uncertainty in the tree search to the uncertainty in the future policy, rather than combined uncertainty from both stage rewards and future policy. 
This difference leads to a departure in the low-level analysis, as well as some tangible improvements: 
First, our exploration law requires less information about the problem: whereas the existing work's exploration law~\cite{shah2020non} requires knowledge of the gap (which is typically unknown), we do not require this knowledge.  
Second, our exploration law is simpler: whereas the existing work's exploration law~\cite{shah2020non} is a complex function of depth of the tree, our constants are depth-independent. 
Both our analysis and that of~\cite{shah2020non} share the same final convergence rate of order $\ell^{-1/2}$.

Application of these two results with the triangle inequality (using $V^*_\text{SETS}$ as a cross term) and evaluating Theorem~\ref{thm:mcts_convergence} at the root shows that SETS quickly converges to the optimal solution of the continuous-space planning problem:

\begin{theorem}
\label{cor:converge_ss_error_tradeoff}
Consider an MDP $\left<\statespace, \actionspace, F, R, D, K, \gamma \right>$.
For initial state $\state_0$, \acronym{} (Algorithm~\ref{algo:spectral_search}) with horizon $H$ yields a value estimate as a function of number of iterations $\ell$ satisfying:
\begin{align}
    | V^*(\state_0) - \mathbb{E}[V(\state_0, \ell)] | 
    &\leq
    \underbrace{
        \frac{\kappa_4}{\sqrt{\ell}}}_{\textup{convergence}} + 
    \underbrace{ \frac{\kappa_0 + \gamma^H }{1-\gamma^H}\left(\kappa_1 (1+\kappa_2 \Delta t)^H + \kappa_3\right)}_{\textup{steady-state error}},
\end{align}
\noindent where $\kappa_{0,1,2,3}>0$ are as in Theorem~\ref{thm:mdp_spectral_search_value_estimate} and $\kappa_4$ is as in Theorem~\ref{thm:mcts_convergence}.
\end{theorem}

The first term in Theorem~\ref{cor:converge_ss_error_tradeoff}, labeled convergence, goes to zero as the number of iterations $\ell$ increases. 
In addition, its leading constant $\kappa_4$ scales with a polynomial power of the total number of action sequences.  
The tree formed by SETS has a branching factor upper bounded by $2 n$ (the number of controllable modes) and number of decisions equal to $\ceil{K/H}$. 
Therefore, an increase of the hyperparameter $H$ improves the convergence speed of the value estimate: a longer branch length results in fewer overall decisions, and therefore a faster convergence.

While providing faster convergence, the trade-off of a longer branch length is a larger asymptotic error, shown in the second term, labeled steady-state error. 
We empirically validate this trade-off with an illustrative example in Fig.~\ref{fig:theory}. 
This error term can also be controlled with other parameters. 
For example, the term with the worst growth behavior, $(1+\kappa_2 \Delta t)$, can be made arbitrarily small by decreasing the integration time of the simulator. However, making this change incurs a higher computational cost per trajectory, limiting the number of trajectories $\ell$ finished by the end of the planning budget. 
Similarly, the entire error term can also be controlled by decreasing the discount factor $\gamma$, at the cost of using less information about the future.

\begin{figure}
    \centering 
    \includegraphics[width=0.99\linewidth]{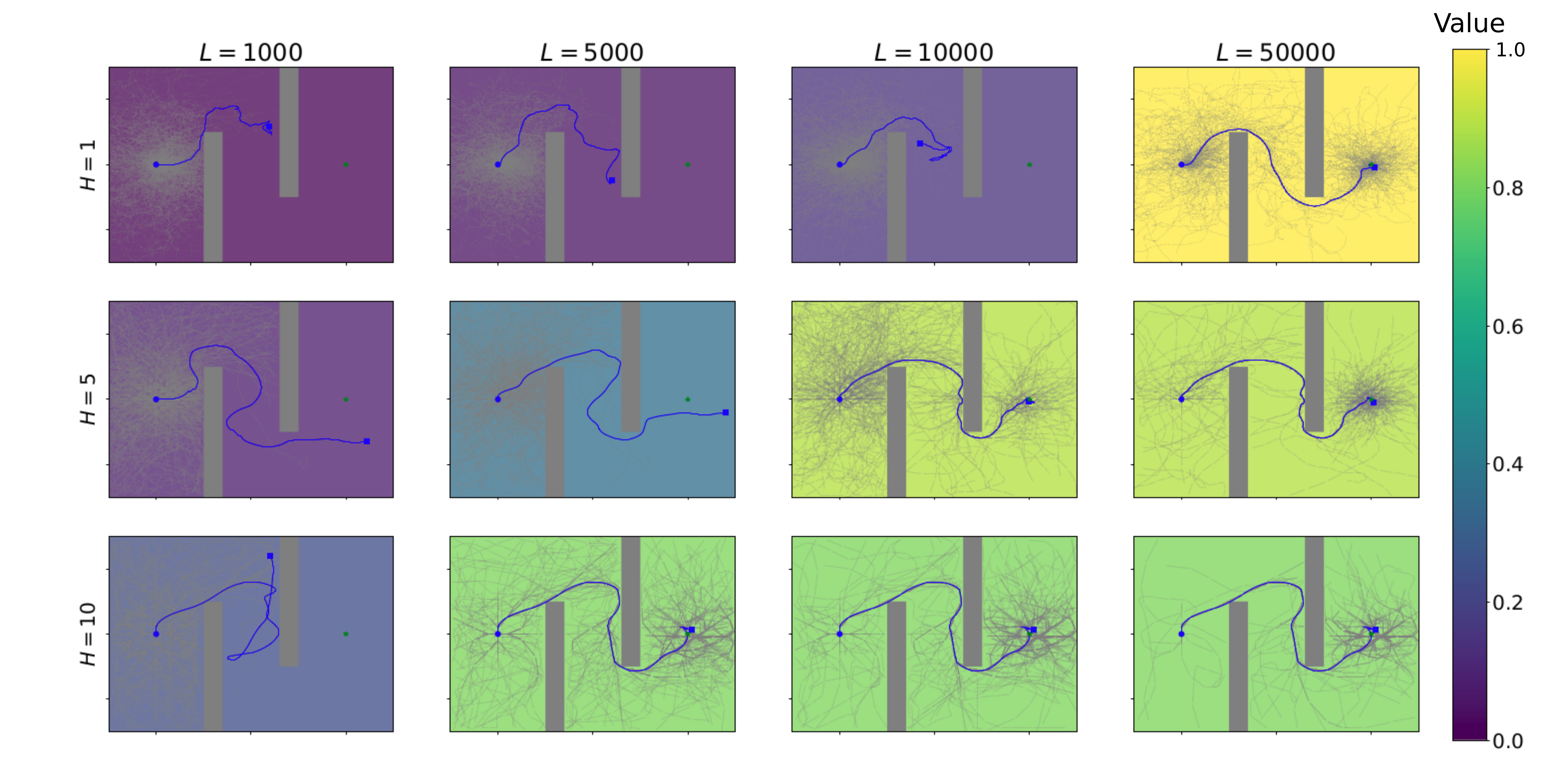}
    \caption{
    We evaluate \acronym{} on a motion planning problem, where a 2D double integrator starts at the blue dot on the left and is tasked to reach the green dot on the right. 
    We vary the branch length parameter $H$ and the number of simulations $L$.
    Each plot is shaded by the value of its trajectory. 
    The empirical trend of the value is predicted by theory: trees with large $H$ converge quickly to sub-optimal solutions, and trees with small $H$ converge slowly to highly optimal solutions. 
    }
    \label{fig:theory}
\end{figure}

The fast optimality convergence of \acronym{} also assists the robot to find a feasible solution.
Under a well-selected terminal reward $D$, the trajectory corresponding to the optimal value function is guaranteed to be feasible over the full horizon. 
One possible choice is $D(x) = H \max_{\mathbf{x} \in X} R(\mathbf{x})$ where the maximum reward is known from Assumption~\ref{assumption:mdp_se}. 
From this terminal reward, the highest-valued infeasible trajectory incurs less reward than the least-valued feasible trajectory, and therefore the optimal trajectory must be feasible if a feasible solution exists. 
A proof of this reasoning is made in recent work~\cite{sfeast}.
Although the optimal trajectory of the transformed problem is guaranteed to be feasible, the trajectory returned by SETS may not be if the number of iterations $\ell$ is insufficient to converge to a feasible trajectory. 
In practice, the algorithm converged to a sufficiently optimal value that the solution was also feasible.

Our theoretical analysis validates the algorithm and provides an explainable interpretation of the planning process. 
In addition, the discussion of the effect of branch length $H$ and, to a lesser extent $\Delta t$ and $\gamma$, makes it clear that our analysis also enables systematic parameter design of various decision-making agents. 
For example, if a robot operates in a dynamic environment and has to re-plan frequently to react to new information, the designer can tune parameters to sacrifice some combination of asymptotic error (larger $H$), dynamics fidelity (larger $\Delta t$), or long-term planning (smaller $K$). 
Alternatively, when a robot operates in a relatively static environment, but has complex long-term behavior to discover, the designer should allocate more budget to each plan. 
The former example corresponds to our quadrotor experiment in Sec.~Quadrotor navigates a dangerous wind field, and the latter example corresponds to the glider experiment in Sec.~Aerodynamic glider performs persistent observation.

Finally, although we combined our main theoretical results into the SETS algorithm, they might be of independent interest as individual components. 
For example, a predictive sampling or model-predictive path integral approach might benefit from a spectral expansion representation, and a traditional game-playing artificial intelligence agent might benefit from our improved MCTS analysis.

\bibliographystyle{ieeetr}
\bibliography{papers}

\section*{Acknowledgements}
We thank the DARPA Learning Introspective Control (LINC) team at Caltech and JPL, who contributed to the full autonomy stack of the results shown in Figure 3. We thank J. Preiss (planning), S. Lupu (control), and M. Anderson (ROS/software) for their support and collaboration on the tracked vehicle experiment. Other LINC team members include J. Burdick, Y. Yue, A. Rahmani, L. Gan (localization), F. Xie (control), J. Becker (segmentation and mapping), T. Touma, and J. Alindogan (experiment).

We would like to thank J. Cho for his support and collaboration on the spacecraft experiment. 
We would also like to thank H. Tsukamoto for discussions on discrete contraction, F. Hadaegh for discussions on space autonomy, and the Sandia National Labs team (T. Blada, E. Lu, D. Wood) for LINC testing support (Figure 3). 
 The drone experiment was conducted at Caltech's Center for Autonomous Systems and Technologies (CAST).
\textbf{Funding}:
We acknowledge the funding support of Supernal (R. Stefanescu, H. Park), The Aerospace Corporation (J. Brader, B. Bycroft), and DARPA LINC (JF Mergen).  
This material is based upon work supported by the National Science Foundation Graduate Research Fellowship Program under Grant No.~2139433. Any opinions, findings, and conclusions or recommendations expressed in this material are those of the authors and do not necessarily reflect the views of the National Science Foundation.
\textbf{Author Contributions}:
The algorithm conceptualization, theory, and software implementation was equally developed by B.R. and J.L. under the guidance and critical reviews of S.-J.C.
The quadrotor and spacecraft experiments were equally designed by B.R. and J.L.
Both B.R. and J.L. contributed to the tracked vehicle experiment and glider experiment, with J.L. leading the tracked vehicle experiment and B.R. leading the glider experiment. 
The manuscript and figures were equally developed by B.R. and J.L. with feedback and multiple iterations with S.-J.C.
All authors reviewed the manuscript. S.-J.C supervised the research.
\textbf{Competing Interests}:
The authors declare no competing interests.
\textbf{Data and materials availability}:
All data and methods needed to evaluate this work are included in the main text and Supplementary Materials.

\section*{Supplementary Materials}

\setcounter{figure}{0}  
\setcounter{equation}{0}  

\renewcommand{\figurename}{Supplemental Figure}
\renewcommand{\thefigure}{\arabic{figure}}

\label{sec:supplemental}

\subsection*{Supplemental Problem Data and Equations of Motion}
\label{sec:supplemental_param}

\subsubsection{Supplemental Quadrotor}
\label{sec:supplemental_param_quad}

\setcounter{MaxMatrixCols}{20}
The physical state is the positions in meters in world frame, velocities in meters per second in body frame, Euler angle attitudes in radians, and attitude rates in radians per second. The control inputs are the thrust in Newtons and body torques in Newton meters. 
\begin{align}
    \mathbf{q} &= \begin{bmatrix}
        p_n & p_e & p_d & u & v & w & \phi & \theta & \psi & \dot{\phi} & \dot{\theta} & \dot{\psi}
    \end{bmatrix}^\top \\ 
    \action &= \begin{bmatrix}
        f_\text{th} & \tau_x & \tau_y & \tau_z 
    \end{bmatrix}^\top
\end{align}
In still air, the quadrotor only experiences external body forces and torques from the rotors. 

The state and action spaces are specified as hypercubes with the following parameters: 
\begin{align}
    \statespace &= \begin{bmatrix}
       -1.3 & -1.3 & -3.2 & -6.0 & -6.0 & -6.0 & -0.8 & -0.8 & -10.0 & -5.0 & -5.0 & -5.0 \\
       1.3 & 1.3 & -2.2 & 6.0 & 6.0 & 6.0 & 0.8 & 0.8 & 10.0 & 5.0 & 5.0 & 5.0
    \end{bmatrix}^\top \\ 
    \actionspace &= \begin{bmatrix}
       0.0 & -0.012 & -0.012 & -0.002  \\ 
       12.0 & 0.012 & 0.012 & 0.002
    \end{bmatrix}^\top
\end{align}

For the dynamical model, we opt for the convention from~\cite{beard2012small} for a 6DOF model of the quadrotor, rather than standard compact notation~\cite{neuralfly}, because it will be reused for the glider model in Sec.~\ref{sec:supplemental_param_glider}. 
The dynamics of the nominal physical system are the following: 
\begin{align*}
    \begin{bmatrix}
        \dot{p}_n \\ \dot{p}_e \\ \dot{p}_d
    \end{bmatrix} &= 
    \begin{bmatrix}
        c_\theta c_\psi & 
        s_\phi s_\theta c_\psi - c_\phi s_\psi & 
        c_\phi s_\theta c_\psi + s_\phi s_\psi \\ 
        c_\theta s_\psi & 
        s_\phi s_\theta s_\psi + c_\phi c_\psi &
        c_\phi s_\theta s_\psi - s_\phi c_\psi \\
        -s_\theta & 
        s_\phi c_\theta &
        c_\phi c_\theta 
    \end{bmatrix}
    \begin{bmatrix}
        u \\ v \\ w
    \end{bmatrix} \nonumber \\ 
    \begin{bmatrix}
        \dot{u} \\ \dot{v} \\ \dot{w}
    \end{bmatrix} &= 
    \begin{bmatrix}
        r v - q w \\
        p w - r u \\ 
        q u - p v 
    \end{bmatrix} + 
    \frac{1}{\mathrm{m}}
    \begin{bmatrix}
        f_x \\ f_y \\ f_z
    \end{bmatrix} \nonumber \\ 
    \begin{bmatrix}
        \dot{\phi} \\ \dot{\theta} \\ \dot{\psi}
    \end{bmatrix} &= 
    \begin{bmatrix}
        1 & s_\phi t_\theta & c_\phi t_\theta \\
        0 & c_\phi & -s_\phi \\ 
        0 & \frac{s_\phi}{c_\theta} & \frac{c_\phi}{c_\theta}
    \end{bmatrix}
    \begin{bmatrix}
        p \\ q \\ r
    \end{bmatrix} \nonumber \\ 
    \begin{bmatrix}
        \dot{p} \\ \dot{q} \\ \dot{r}
    \end{bmatrix} &= 
    \begin{bmatrix}
        \Gamma_1 p q - \Gamma_2 q r \\
        \Gamma_5 p r - \Gamma_6 (p^2 - r^2) \\ 
        \Gamma_7 p q - \Gamma_1 q r 
    \end{bmatrix} + 
    \begin{bmatrix}
        \Gamma_3 l + \Gamma_4 n \\ 
        \frac{1}{J_y} m \\ 
        \Gamma_4 l + \Gamma_8 n 
    \end{bmatrix}
    \label{eq:6dof_dynamics}
\end{align*}
where (i) $s$, $c$, $t$ are shorthand notation for $\sin$, $\cos$, and $\tan$, (ii) $\mathrm{m}$, $\{ \Gamma_i\}_{i=1}^{8}$ are mass and moment of inertia constants, (iii) $f_x$, $f_y$, $f_z$ are the external forces in body frame and $n$, $m$, $l$ are the external torques.

The state also stores the number of timesteps since viewing each target $\xi^i_k$, and these states are updated according to an observation rule:  
\begin{alignat}{2}
    \state_k &= \begin{bmatrix}
        \mathbf{q}_k^\top & \xi^1_k & \hdots & \xi^{n_t}_k
    \end{bmatrix}^\top , \quad  
    \xi^i_{k+1} = \begin{cases}
        0 & \mathbf{q}_k \in \mathcal{G}^i \\
        1 + \xi^i_k & \text{else}
    \end{cases}
\end{alignat}
where $n_t$ is the number of targets. We set the $i$th goal region $G_i$ to be the states within 0.75 meters of the center of the $i$th target. This radius is chosen because the target ball's radius is 0.20 meters, the quadrotor's radius is 0.10 meters, and we leave some tolerance to account for process noise and perception error.  
The obstacles are inflated to account for the quadrotor's 20 cm diameter, so collision checking is done with the quadrotor's center-of-mass.
The complete discrete dynamics $F$ are computed from Euler integration of the continuous physical dynamics subject to the wind's external forcing, and the discrete augmented states transition is computed separately.

The final component of the MDP is the objective: the reward function is an affine function of the number of each targets observed within their associated timescale $T^i$, and the terminal value is uniformly zero: 
\begin{align}
    R(\state) &= r_0 + \frac{1 - r_0}{n_t} \sum_{i=1}^{n_t} \mathbf{1}_{(\xi^i < T^i)}(\state), \quad D(\state) \equiv 0 
\end{align}

Next, we discuss the data generation and model training for the wind model. The wind's effect is modeled as an external body force and torque as a function of the physical state, action and fan speed: 
\begin{align}
    \textup{wind} \left( 
        \mathbf{q}_k, \action, \textup{fan speed}
    \right) &= 
    \begin{bmatrix}
        f_{x,w} & f_{y,w} & f_{z,w} & n_w & m_w & l_w
    \end{bmatrix}^\top 
\end{align}
To train our wind model, we collect a set of trajectories $\{(\state_k, \action_k, \textup{fan speed}_k)\}_{k=1}^{|\mathcal{D}|}$, generate labels of wind effects using the residual error between the predicted and observed state, and train our model using supervised learning of a feedforward neural network: 
\begin{align}
    \mathcal{D} &= \{ (\underbrace{\state_k, \action_{k+1}, \textup{fan speed}_k}_{\text{input}}, \ \underbrace{\state_{k+1} - F_n(\state_k, \action_{k+1})}_{\text{label}}) \} \\ 
    \widehat{\textup{wind}} &= \argmin_{w \in W} \sum_{(\text{input}, \text{label}) \in \mathcal{D}} \| w(\text{input}) - \text{label} \|_2
\end{align}
where $F_n$ is the nominal predicted dynamics without wind. 

The trajectories are generated by randomly sampling terminal points inside the environment arena, fitting a polynomial, and tracking the polynomial trajectory with the built-in controller. 
We collect 2 minutes of data for 5 different fan settings, for a total of 76,237 datapoints. 
Our neural network architecture includes two hidden layers with eight hidden dimensions each, with ReLU activations.
Once the wind model is trained, it is incorporated into the MDP as external forces and torques. 

The autonomy components are as follows: 
we use a commercial off-the-shelf quadrotor and its built-in autopilot for tracking control and a combination of Visual Inertial Odometry (VIO) and motion capture for state estimation, where the obstacle and target locations are specified at runtime. 
\acronym{} is the planning module, generating trajectories of 10 second duration every 5 seconds and running in real-time model-predictive control manner. 
The dynamics model used by \acronym{} is the standard quadrotor model augmented with a DNN to model the learned residual wind force. 
The data collection and model training is a simplified procedure from previous work~\cite{neuralfly}.

\subsubsection{Supplemental Tracked Vehicle}
\label{sec:supplemental_param_linc}

We model the vehicle as a 3DOF car specified by $x$ and $y$ positions in meters, heading angle $\theta$ in radians, and linear and angular velocities $v$, $\omega$ in m/s and rad/s, respectively. The input is a desired linear and angular velocity, which enter the dynamics as a first-order filtration.

The state space and action space are specified with the following parameters: 
\begin{align}
    \statespace = \begin{bmatrix}
        -100 & -100 & -10\pi & -1.8 & -1.5 \\ 
         100 &  100 &  10\pi &  1.8 &  1.5 \\ 
    \end{bmatrix}^\top
    \quad
    \actionspace = \begin{bmatrix}
        -1.0,& -1.0 \\
        1.0,& 1.0
    \end{bmatrix}^\top
\end{align}

The nominal tracked vehicle dynamics are: 
\begin{align}
    F\left(\begin{bmatrix}
        x \\ y \\ \theta \\ v \\ \omega
    \end{bmatrix},
    \begin{bmatrix}
        v_d \\ \omega_d
    \end{bmatrix}\right)
    =
    \begin{bmatrix}
        x \\ y \\ \theta \\ v \\ \omega
    \end{bmatrix} +
    \Delta t
    \begin{bmatrix}
        v \cos{\theta} \\
        v \sin{\theta} \\
        \omega \\
        \frac{1}{\tau_v}(-v + v_d) \\
        \frac{1}{\tau_\omega}(-\omega + \omega_d)
    \end{bmatrix}
\end{align}
for parameters $\Delta t = 0.1$, $\tau_v = 0.2, \tau_\omega = 0.15$.
In the experiment, the dynamics are updated according to the adapted parameters as described in Equation~\eqref{eq:dynamics_update}

The reward is a negative quadratic function of the difference between the driver's intended velocity and the actual velocity, and the terminal value is uniformly zero.
\begin{align}
    R(\state, \action) = \max{(1.0 - 0.8(v-v_d)^2 - 0.6(\omega - \omega_d)^2,0)}, \quad D(\state) \equiv 0
\end{align}

An onboard vision system fuses visible and infrared images, forming an egocentric elevation map.
From the map, a traversability layer is created by evaluating the static stability of the vehicle on every point in the terrain.
Additionally, using the DinoV2 visual transformer~\cite{oquab2023dinov2}, we segment traffic cones, rocks, and other terrain features out of visible-light images to further refine the traversability layer.
During the \acronym{} search, the traversability layer is thresholded to form an unsafe set $\Omega$, with each node's state $\state$ checked to ensure $\state \in \statespace \setminus \Omega$.

An adaptive controller tracks the reference trajectory produced by \acronym{}, using composite adaptation to simultaneously minimize tracking error and converge physical parameter estimates to meaningful values.
At each recomputation, the updated parameter values are passed to the planner, forming a new estimated dynamics function $F$ for the search.
The form of adaptation we use is a disturbance estimator over the space of desired velocity disturbances, where the updated parameters $\{a_1, \dots, a_4\}$ affect the tracked vehicle dynamics as
\begin{align}
    \begin{bmatrix}
        v_d' \\ \omega_d'
    \end{bmatrix}
    =
    \begin{bmatrix}
        1+a_1 & a_2 \\ a_3 & 1+a_4
    \end{bmatrix}
    \begin{bmatrix}
        v_d \\ \omega_d
    \end{bmatrix}
    \label{eq:dynamics_update}
\end{align}
where the parameters $\{a_1, \dots, a_4\}$ are updated with conventional composite adaptation~\cite{slotine1991applied}. 

We consider a decision horizon of 16 timesteps with a discount factor of $\gamma=1$ and a branch length hyperparameter of $H=2$.
The planner recomputes every 0.1 seconds. 
Though our analysis relies on $\gamma < 1$, this demonstration of \acronym{} deployed on an undiscounted MDP demonstrates it can be practically implemented beyond the class of smooth and discounted MDPs.

We augment our algorithm with a Sequential Convex Programming (SCP) post-processing operation to locally optimize the trajectory output of \acronym{}. 
The ``search then optimize'' framework has been used before in robotics community for motion planning problems~\cite{richter2016polynomial, hoenig_2018}, and here we extend it to the more general decision-making problem. 
In serial combination, these two components synergize well: search provides a global notion of optimality that avoids low-valued local minima, and the optimization refines the search's solution with a local notion of optimality.
In practice, this combination led to smoother interactions with the driver while maintaining exploration of nearby candidate trajectories. 
The SCP post-processing iteratively convexifies and solves the convexified problem, allowing a refinement of the search-based solution.
The convexification is achieved by linearizing the dynamics, linearizing the state constraint $\state \in \statespace \setminus \Omega$, and taking a second-order approximation of the reward around the previous iteration's trajectory solution.
The full details of the convexification procedure are similar to the implementation in~\cite{foust2020optimal}.

When deploying our algorithm on the tracked vehicle, we run our algorithm onboard on an NVIDIA Orin (12-core, 2.2 GHz ARM CPU, 32 GB RAM).
We use the OpenVINS package for state estimation and close the loop with a custom learning-based adaptive controller~\cite{MAGIC}. 

\subsubsection{Supplemental Spacecraft}
\label{sec:supplemental_param_spacecraft}

We model each spacecraft as a planar double integrator, with $p_x$ and $p_y$ positions in meters and $v_x$ and $v_y$ velocities in meters per second.
We use a 1-dimensional finite element mesh for the net model, with four nodes.

The state and action spaces are hypercubes, where each $x$ and $y$ position is constrained to the hypercube $(p_x,p_y) \in [-3.0, 5.0]^2$, and velocities to $(v_x,v_y) \in [-0.25, 0.25]^2$.

We model the external forces on each spacecraft and net nodes to model the dynamics.
The two controlled spacecraft are each actuated by forces in the $x$ and $y$ direction.

Between the net nodes and the edge net nodes and the controlled spacecraft, we consider a linear tension-only spring-damper.
The symmetric force, aligned in the direction of the net, for net segment $i$ is
\begin{align}
    F_{\text{net},i} = (l_i > l) (k_n (l_i - l) - c_n \dot{l}_i)
\end{align}
for nominal length $l$, net stiffness $k_n$, and damping coefficient $c_n$.
Here the length $l_i$ is the magnitude of the position difference between two nodes:
$\|[p_{x,i} \ p_{y,i}]^\top - [p_{x,i-1} \ p_{y,i-1}]^\top\|$.

Between each of the net nodes and the target spacecraft, we model the collision as a stiff spring-damper, with a similar form to above.
For node $i$, the contact force between the net node and the target is
\begin{align}
    F_{\text{contact},i} = (d_i < r_t) (k_c (d_i - r_t) - c_c \dot{d}_i)
\end{align}
for target radius $r_t$, collision stiffness $k_c$, and collision damping $c_c$.
Here $d_i$ is the distance between node $i$ and the target: $\|[p_{x,i} \ p_{y,i}]^\top - [p_{x,t} \ p_{y,t}]^\top\|$.
A visualization of the forces involved in this scenario is available in Fig.~\ref{fig:spacecraft_capture}.

\begin{figure}[htbp]
    \centering
    \begin{tikzpicture}[>=Stealth]
        \coordinate (A) at (0,-3);
        \coordinate (B) at (3,0);
        \coordinate (C) at (6,3);
        \coordinate (D) at (9,0);
        \coordinate (E) at (4.5,-3.0);
      
        \foreach \point in {A,D,E} {
            \fill (\point) circle (4pt);
        }

        \foreach \point in {B,C} {
            \fill (\point) circle (2pt);
        }

        \draw[->] (E) -- ++(0,0.75) node[above] {$F_\text{contact,4,y}$};
        \draw[->] (E) -- ++(0.75,0) node[above] {$F_\text{contact,4,x}$};

        \draw[->] (B) -- ++(0,0.75) node[above] {$F_\text{contact,1,y}$};
        \draw[->] (B) -- ++(0.75,0) node[below] {$F_\text{contact,1,x}$};
        \draw[->] (C) -- ++(0,0.75) node[above] {$F_\text{contact,2,y}$};
        \draw[->] (C) -- ++(0.75,0) node[above] {$F_\text{contact,2,x}$};

        \draw[->] (B) -- ++(0.75,0.75) node[right] {$F_\text{net,1,+}$};
        \draw[->] (B) -- ++(-0.75,-0.75) node[left] {$F_\text{net,1,-}$};
        \draw[->] (C) -- ++(0.75,-0.75) node[right] {$F_\text{net,2,+}$};
        \draw[->] (C) -- ++(-0.75,-0.75) node[left] {$F_\text{net,2,-}$};
        \draw[->] (A) -- ++(0.75,0.75) node[right] {$F_\text{net,0,+}$};
        \draw[->] (D) -- ++(-0.75,0.75) node[left] {$F_\text{net,3,-}$};

        \draw[->] (A) -- ++(0,0.75) node[above] {$F_\text{control,0,y}$};
        \draw[->] (A) -- ++(0.75,0) node[below] {$F_\text{control,0,x}$};
        \draw[->] (D) -- ++(0,0.75) node[above] {$F_\text{control,3,y}$};
        \draw[->] (D) -- ++(0.75,0) node[above] {$F_\text{control,3,x}$};

        
        \draw (A) -- (B);
        \draw (B) -- (C);
        \draw (C) -- (D);
        
        \node[below=5pt] at (A) {$0$};
        \node[below=5pt] at (B) {$1$};
        \node[below=5pt] at (C) {$2$};
        \node[below=5pt] at (D) {$3$};
        \node[below=5pt] at (E) {$4$};
    \end{tikzpicture}
    \caption{Forces on Spacecraft Capture Problem for two net nodes.}
    \label{fig:spacecraft_capture}
\end{figure}
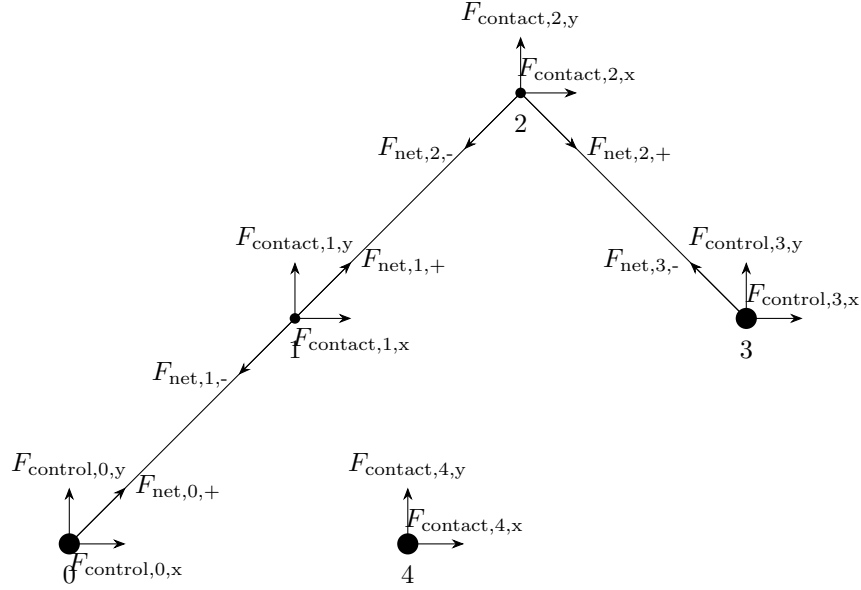

The reward is a composite function of three terms:
\begin{align}
    R(\state) =& c_1 s(\|
        [p_{x,\text{centroid}} \ p_{y,\text{centroid}}]^\top - [p_{x,t} \ p_{y,t}]^\top\|, a_1) \\
        &+ c_2 s(\|
        [v_{x,\text{centroid}} \ v_{y,\text{centroid}}]^\top - [v_{x,d} \ v_{y,d}]^\top\|, a_2) \\
        &+ c_3 s(\|
        [v_{x,t} \ v_{y,t}]^\top - [v_{x,d} \ v_{y,d}]^\top\|, a_3) \\
        D(\state) \equiv 0
\end{align}
where $[p_{x,\text{centroid}} \ p_{y,\text{centroid}}]^\top$ is the centroid of the controlled spacecraft and net structure.
Here $s$ is a normalization function: $s(d, a) = 1 - \frac{2}{\pi} \arctan(\frac{d}{a})$.
The first term serves to guide the centroid of the controlled spacecraft and net toward the target.
The second and third terms encourage the captured target and controller spacecraft to move in the desired direction.

\subsubsection{Supplemental Glider}
\label{sec:supplemental_param_glider}

We use the same physical state and dynamics as the quadrotor example~\eqref{eq:6dof_dynamics}, with the following modifications: (i) the control input affects the control surfaces, (ii) external forces use a 6DOF linear aerodynamic model, and (iii) the external wind model is a uniform upward force when in the thermal area. All the physical and aerodynamic parameters are taken from the Aerosonde UAV available in~\cite{beard2012small}. 

The control vector is the elevator, aileron, and rudder deflections $\action = [\delta_e,  \delta_a, \delta_r]^\top$.

The state and action space are hypercubes with the following upper and lower limits: 
\begin{align}
    \statespace &= \begin{bmatrix}
       -1000 & -1000 & -750 & -600 & -600 & -600 & -2 & -2 & -100 & -50 & -50 & -50 \\
        1000 &  1000 & -0.2 & 600 &  600 &  600 & 2 &  2 & 100 & 50 & 50 & 50 \\
    \end{bmatrix}^\top \\ 
    \actionspace &= \begin{bmatrix}
        -0.5 & -0.5 & -0.5 \\
         0.5 & 0.5 & 0.5
    \end{bmatrix}^\top
\end{align}

The external forces are computed with a first-order Taylor expansion of the following longitudinal and lateral components: 
\begin{alignat}{2}
    f_\text{lift} &= \frac{1}{2} \rho V_a^2 S C_L(\alpha, q, \delta_e) \quad \quad 
    &f_y = \frac{1}{2} \rho V_a^2 S C_Y(\beta, p, r, \delta_a, \delta_r) \\ 
    f_\text{drag} &= \frac{1}{2} \rho V_a^2 S C_D(\alpha, q, \delta_e) 
    &l = \frac{1}{2} \rho V_a^2 S b C_l(\beta, p, r, \delta_a, \delta_r) \\ 
    m &= \frac{1}{2} \rho V_a^2 S c C_m(\alpha, q, \delta_e) 
    &n = \frac{1}{2} \rho V_a^2 S b C_n(\beta, p, r, \delta_a, \delta_r)  
\end{alignat}
where (i) $f_\text{lift}$ and $f_\text{drag}$ are rotated into $f_x$ and $f_y$, (ii) $S$, $c$, $b$ are the planar area of the wing surface, mean chord length, and wind span of the drone, (iii) $\alpha$ and $\beta$ are the angle of attack and slip angle, and (iv) $V_a$ is the relative speed between the glider and the wind. 

The reward function defines the problem objective and has two terms: 
\begin{align}
    R(\state) &= 
    0.1 r_0 
    + 0.9 \left( \mathbf{1}_{\xi < T} + 0.5 (1 - \mathbf{1}_{\xi < T}) s(\| [p_x, p_y, p_z] - [p_{x,g}, p_{y,g}, p_{z,g}] \|, a) \right) \\ 
    D(\state) &\equiv 0
\end{align}
The first term is a nominal stay alive reward, $r_0$. 
The second term is an observation target reward
where $\xi$ is the same time-since-target-observed augmented state as the quadrotor example, except the goal condition is changed to be an inclusion of the target in the observation cone, where the cone's axis is aligned with the body axis of the glider and has a 30 degree angle and 100 meter length. 
If the target has not recently been observed, ($\mathbf{1}_{\xi < T} = 0$), then reward is given for being near the target, where $s$ is the same normalization function in the spacecraft example.

\subsection*{Supplemental Theoretical Analysis}

Our analysis organization is summarized in Fig.~\ref{fig:sets_proof_organization}. 
First, we develop the Spectral Expansion result (Theorem~\ref{thm:mdp_spectral_search_value_estimate}). 

\begin{figure}
    \centering 
    \includegraphics[width=0.99\linewidth]{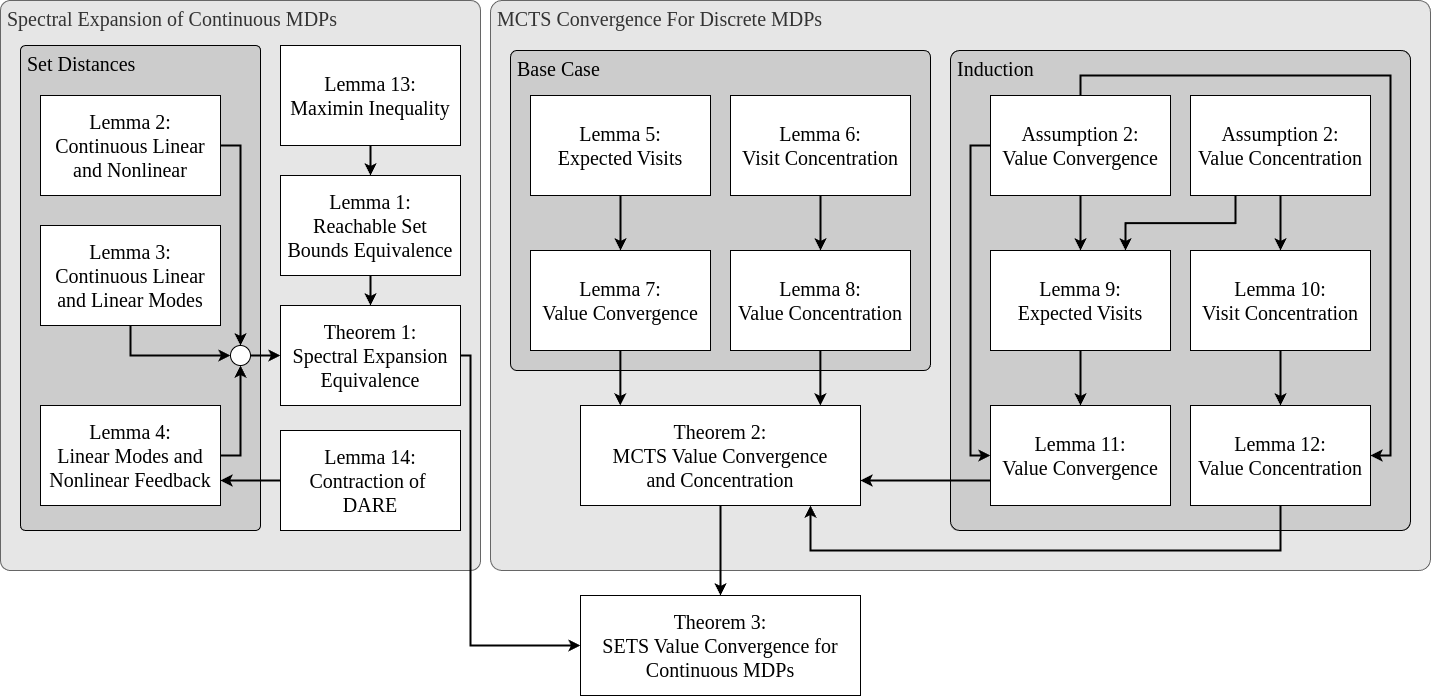}
    \caption{
    Our proof layout shows how Lemmas are connected to yield our main theoretical contributions.}
    \label{fig:sets_proof_organization}
\end{figure}

We require some preliminary definitions of Lipschitz functions, reachable sets, Hausdorff distance, and linearized dynamics: 
\begin{definition}
    \label{def:Lipschitz}
    A function $g:\mathbb{R}^n \rightarrow \mathbb{R}^r$ is Lipschitz continuous if $\exists L_g\geq 0$ s.t. $\forall u, v \in \mathbb{R}^n$, 
    $\| g(u) - g(v) \| \leq L_g \| u - v \|$.
\end{definition}
\begin{definition}
\label{def:reachable_set}
Given an initial state $\state_0$, a dynamical model $F$, and a set of actions $U$, the reachable set is the set of states after rolling out each of the actions:
\begin{align}
    \mathcal{R}_F(\state_0, U) = \{ F(\state_0, \action) \ | \ \forall \action \in U \}
\end{align}
\end{definition}
\begin{definition}
    \label{def:Hausdorff_distance}
    Consider a point $a \in \mathbb{R}^n$ and a set $M \subset \mathbb{R}^n$. 
    The point-set distance between $a$ and $M$ is $d(a, M) = \inf_{m \in M} \| a - m \|$.
    Consider two sets $M, N \subset \mathbb{R}^n$. The Hausdorff distance (standard set distance) between $M$ and $N$ is the symmetric function $d_S(M, N) = \max \left\{ \sup_{m \in M} d(m, N), \ \sup_{n \in N} d(M, n) \right\}$.
\end{definition}
\begin{definition}
    \label{def:linearized_dynamics}
    Consider the nonlinear dynamics $F(\state, \action)$. 
    The linearization about a state and input trajectory $(\bar{\state}_{k}, \bar{\action}_{k+1})$ is defined as:
    \begin{align}
        L_k(\state, \action) = 
        F(\bar{\state}_k, \bar{\action}_{k+1}) 
        + \nabla_\state F\rvert_{(\bar{\state}_k, \bar{\action}_{k+1})} (\state - \bar{\state}_k) 
        + \nabla_\action F\rvert_{(\bar{\state}_k, \bar{\action}_{k+1})} (\action - \bar{\action}_{k+1})
    \end{align}
\end{definition}

We present Lemma~\ref{lemma:mdp_equivalence} on the equivalence of two MDPs and their reachable set.
\begin{lemma}
\label{lemma:mdp_equivalence}
Consider two MDPs satisfying Assumption~\ref{assumption:mdp_se}:
\begin{align}
    \mdp_1 &= \left<\statespace, \actionspace_1, F, R, D, K, \gamma \right> \\
    \mdp_2 &= \left<\statespace, \actionspace_2, F, R, D, K, \gamma \right>,
\end{align}
Let $V_1^*$ and $V_2^*$ denote the respective optimal value functions of the two problems.
The difference in optimal value function between the two problems is uniformly bounded by a constant times the set distance between their reachable sets:
\begin{align}
    \|V_1^* - V_2^*\|_\infty \leq
    \frac{L_R + \gamma L_V}{1-\gamma} \max_{\state \in \statespace} 
        d_S(\mathcal{R}_F(\state, \actionspace_1), \mathcal{R}_F(\state, \actionspace_2))
\end{align}
\end{lemma}

\begin{proof}
First, we define the Bellman value operator for each MDP: 
\begin{align}
    (\mathcal{T}_1 V)(\state) &:= \max_{\action_1 \in \actionspace_1} R(F(\state, \action_1), \action_1) + \gamma V(F(\state, \action_1)) \\
    (\mathcal{T}_2 W)(\state) &:= \max_{\action_2 \in \actionspace_2} R(F(\state, \action_2), \action_2) + \gamma W(F(\state, \action_2))
\end{align}

For discounted MDPs ($\gamma < 1$) with bounded rewards, there exists an optimal value function. 
Additionally, this optimal value function is the unique fixed point of the Bellman value operator~\cite{Bert05}.

Now, we consider the distance between the optimal value function of each problem: 
\begin{align}
    \| V_1^* - V_2^* \|_\infty 
    &= \| \mathcal{T}_1 V_1^* - \mathcal{T}_2 V_2^* \|_\infty   
    \label{eq:proof2_to_combine}
    \leq \| \mathcal{T}_1 V_1^* - \mathcal{T}_1 V_2^* \|_\infty + \| \mathcal{T}_1 V_2^* - \mathcal{T}_2 V_2^* \|_\infty   
\end{align}

The first term of Equation~\eqref{eq:proof2_to_combine} is bounded via contraction of the $\mathcal{T}_1$ operator: $\| \mathcal{T}_1 V_1^* - \mathcal{T}_1 V_2^* \|_\infty \leq \gamma \| V_1^* - V_2^* \|_\infty$. 

Let $Q(F(\state, \action_i)) = R(F(\state, \action_i)) + \gamma V_2^*(F(\state, \action_i))$ for $i=1,2$. 
The Lipschitz constant of $Q$ with respect to $F(\state, \action_i)$ is $L_Q = L_R + \gamma L_{V}$ where $L_V$ is the Lipschitz constant of $V_2^*$. 
\begin{align}
    | &Q(F(\state, \action_1)) - Q(F(\state, \action_2)) | \nonumber \\
    &= | R(F(\state, \action_1)) + \gamma V_2^*(F(\state, \action_1)) - 
    R(F(\state, \action_2)) - \gamma V_2^*(F(\state, \action_2)) | \nonumber \\
    &\leq | R(F(\state, \action_1)) - R(F(\state, \action_2)) | + \gamma | V_2^*(F(\state, \action_1)) - V_2^*(F(\state, \action_2)) | \nonumber \\
    &\leq (L_R + \gamma L_V) \| F(\state, \action_1) - F(\state, \action_2) \|_2 , \quad \forall \state, \action_1, \action_2
    \label{eq:Q_function_Lipschitz}
\end{align}
Note that taking the alternate cross term in Equation~\eqref{eq:proof2_to_combine} results in $L_V$ as the Lipschitz constant of $V_1^*$, so we write $L_V$ with no index. 

Although the exact value of $L_V$ is not known without knowledge of the optimal solution, $V_1^*$ and $V_2^*$ are Lipschitz, and their Lipschitz constants can be upper bounded by a function of the Lipschitz constants of the dynamics and reward functions.
We further note that, as a consequence of the principle of optimality, we only need to use the Lipschitz constant of $R$ restricted to optimal trajectories.
For notational simplicity, we just write $L_R$ in Equation~\eqref{eq:Q_function_Lipschitz}.

To bound the second term of Equation~\eqref{eq:proof2_to_combine}, note that, for all $\state \in \statespace$, 
{
\allowdisplaybreaks
\begin{align}
    &\mathcal{T}_1 V_2^*(\state) - \mathcal{T}_2 V_2^*(\state) \\ 
    & \qquad = \max_{\action_1 \in \actionspace_1} \bigg\{
    R(F(\state, \action_1)) + \gamma V_2^*(F(\state, \action_1)) 
    - \max_{\action_2 \in \actionspace_2} R(F(\state, \action_2)) + V_2^*(F(\state, \action_2))\\
    & \qquad = \max_{\action_1 \in \actionspace_1} \bigg\{
        R(F(\state, \action_1)) + \gamma V_2^*(F(\state, \action_1)) 
        + \min_{\action_2 \in \actionspace_2} \{- R(F(\state, \action_2)) - V_2^*(F(\state, \action_2))\}
    \bigg\} \\
    & \qquad = \max_{\action_1 \in \actionspace_1} \min_{\action_2 \in \actionspace_2} \bigg\{
        R(F(\state, \action_1)) + \gamma V_2^*(F(\state, \action_1)) 
        - R(F(\state, \action_2)) - V_2^*(F(\state, \action_2))
    \bigg\} \\ 
    & \qquad = 
    \max_{\action_1 \in \actionspace_1} \min_{\action_2 \in \actionspace_2} Q(F(\state, \action_1)) - Q(F(\state, \action_2)) \\ 
    & \qquad \leq
    \max_{\action_1 \in \actionspace_1} \min_{\action_2 \in \actionspace_2} \big| Q(F(\state, \action_1)) - Q(F(\state, \action_2)) \big| \quad \text{ by Lemma~\ref{lemma:minmax_helper}} \\ 
    & \qquad \leq 
    \max_{\action_1 \in \actionspace_1} \min_{\action_2 \in \actionspace_2} L_Q \| F(\state, \action_1) - F(\state, \action_2) \| \qquad \text{by Lemma~\ref{lemma:minmax_helper}, as }Q\text{ is Lipschitz} \\ 
    & \qquad = 
    (L_R + \gamma L_{V}) \max_{\action_1 \in \actionspace_1} \min_{\action_2 \in \actionspace_2} \bigg\{
        \|F(\state, \action_1) - F(\state, \action_2)\|
    \bigg\} \\ 
    & \qquad \leq 
    (L_R + \gamma L_{V}) d_S(\mathcal{R}_F(\state, \actionspace_1), \mathcal{R}_F(\state, \actionspace_2)) \\
    & \qquad = \big|(L_R + \gamma L_{V}) d_S(\mathcal{R}_F(\state, \actionspace_1), \mathcal{R}_F(\state, \actionspace_2))\big|
\end{align}
}
and identical analysis applied to the negation shows
{
\allowdisplaybreaks
\begin{align}
    &\mathcal{T}_2 V_2^*(\state) - \mathcal{T}_1 V_2^*(\state) \\ 
    & \qquad = \max_{\action_2 \in \actionspace_2} \bigg\{
    R(F(\state, \action_2)) + \gamma V_2^*(F(\state, \action_1)) 
    - \max_{\action_1 \in \actionspace_1} R(F(\state, \action_1)) + V_2^*(F(\state, \action_1))\\
    & \qquad = \max_{\action_2 \in \actionspace_2} \bigg\{
        R(F(\state, \action_2)) + \gamma V_2^*(F(\state, \action_2)) 
        + \min_{\action_1 \in \actionspace_1} \{- R(F(\state, \action_1)) - V_2^*(F(\state, \action_1))\}
    \bigg\} \\
    & \qquad = \max_{\action_2 \in \actionspace_2} \min_{\action_1 \in \actionspace_1} \bigg\{
        R(F(\state, \action_2)) + \gamma V_2^*(F(\state, \action_2)) 
        - R(F(\state, \action_1)) - V_2^*(F(\state, \action_1))
    \bigg\} \\ 
    & \qquad = 
    \max_{\action_2 \in \actionspace_2} \min_{\action_1 \in \actionspace_1} Q(F(\state, \action_2)) - Q(F(\state, \action_1)) \\ 
    & \qquad \leq
    \max_{\action_2 \in \actionspace_2} \min_{\action_1 \in \actionspace_1} \big| Q(F(\state, \action_2)) - Q(F(\state, \action_1)) \big| \quad \text{ by Lemma~\ref{lemma:minmax_helper}} \\ 
    & \qquad \leq 
    \max_{\action_2 \in \actionspace_2} \min_{\action_1 \in \actionspace_1} L_Q \| F(\state, \action_2) - F(\state, \action_1) \| \qquad \text{by Lemma~\ref{lemma:minmax_helper}, as }Q\text{ is Lipschitz} \\ 
    & \qquad = 
    (L_R + \gamma L_{V}) \max_{\action_2 \in \actionspace_2} \min_{\action_1 \in \actionspace_1} \bigg\{
        \|F(\state, \action_2) - F(\state, \action_1)\|
    \bigg\} \\ 
    & \qquad \leq 
    (L_R + \gamma L_{V}) d_S(\mathcal{R}_F(\state, \actionspace_2), \mathcal{R}_F(\state, \actionspace_1)) \\
    & \qquad = \big|(L_R + \gamma L_{V}) d_S(\mathcal{R}_F(\state, \actionspace_1), \mathcal{R}_F(\state, \actionspace_2))\big|
\end{align}
}
Therefore
\begin{align}
    |\mathcal{T}_1 V_2^*(\state) - \mathcal{T}_2 V_2^*(\state)|
    &\leq \big|(L_R + \gamma L_{V}) d_S(\mathcal{R}_F(\state, \actionspace_1), \mathcal{R}_F(\state, \actionspace_2))\big|
\end{align}
Noting that $\statespace$ is compact, let $\tilde{\state} = \argmax_{\state \in \statespace} | \mathcal{T}_1 V_2^*(\state) - \mathcal{T}_2 V_2^*(\state) |$.
Consequently, the second term of Equation~\eqref{eq:proof2_to_combine} can be bounded as
\begin{align}
    &\| \mathcal{T}_1 V_2^* - \mathcal{T}_2 V_2^* \|_\infty \\
    &= \max_{\state \in \statespace} | \mathcal{T}_1 V_2^*(\state) - \mathcal{T}_2 V_2^*(\state) | \\
    &= | \mathcal{T}_1 V_2^*(\tilde{\state}) - \mathcal{T}_2 V_2^*(\tilde{\state}) | \\
    &\leq | (L_R + \gamma L_{V}) d_S(\mathcal{R}_F(\tilde{\state}, \actionspace_1), \mathcal{R}_F(\tilde{\state}, \actionspace_2)) | \\
    &= (L_R + \gamma L_{V}) d_S(\mathcal{R}_F(\tilde{\state}, \actionspace_1), \mathcal{R}_F(\tilde{\state}, \actionspace_2)) \qquad \text{as }d_S\text{ is a metric and }L_R, L_V, \gamma \geq 0\\
    &\leq \max_{\state \in \statespace} \left((L_R + \gamma L_{V}) d_S(\mathcal{R}_F(\state, \actionspace_1), \mathcal{R}_F(\state, \actionspace_2))\right) \\
    &= (L_R + \gamma L_V) \max_{\state \in \statespace} d_S(\mathcal{R}_F(\state, \actionspace_1), \mathcal{R}_F(\state, \actionspace_2)) \quad \text{as }L_R, L_V, \gamma \geq 0
\end{align}

We arrive at the desired result by combining the two bounds in~\eqref{eq:proof2_to_combine} and manipulating: 
\begin{align}
    \|V_1^* - V_2^*\|_\infty \leq
    \frac{(L_R + \gamma L_V)}{1-\gamma} \max_{\state \in \statespace} \left(
        d_S(\mathcal{R}_F(\state, \actionspace_1), \mathcal{R}_F(\state, \actionspace_2))
    \right).
\end{align}
\end{proof}

In order to apply Lemma~\ref{lemma:mdp_equivalence} to our algorithm analysis, we need to bound the set distance between the reachable set of the continuous nonlinear system and the finite set of trajectories generated by SETS. 
As $d_S$ is a metric, we can apply the triangle inequality along with cross terms:
\begin{align}
    \label{eq:reachable_set_triangle}
    d_S(\mathcal{R}_{F^H}(\state_0, \actionspace^H), \{\state_i\}_{i=1}^{2n}) &\leq
    d_S(\mathcal{R}_{F^H}(\state_0, \actionspace^H), \mathcal{R}_{L^H}(\state_0, \actionspace^H))
    + d_S(\mathcal{R}_{L^H}(\state_0, \actionspace^H), \{\mathbf{z}^i_H\}_{i=1}^{2n}) 
    \nonumber \\ & \qquad 
    + d_S(\{\mathbf{z}^i_H\}_{i=1}^{2n}, \{\state^i_H\}_{i=1}^{2n})
\end{align}
where $\{\mathbf{z}^i_H\}_{i=1}$ are the modes of the locally linearized system and $\{\state^i_H\}_{i=1}$ are the trajectories generated by SETS. 

Lemmas~\ref{lemma:continuous_nonlinear_linear_set_distance},~\ref{lemma:continuous_linear_finite_linear_set_distance}, and~\ref{lemma:set_distance_linear_modes_sets_trajs}
provide upper bounds for each sum on the RHS of~\eqref{eq:reachable_set_triangle}. 
This analysis is a formalization of the discussion related to Fig.~\ref{fig:method_se}. 

\begin{lemma}
\label{lemma:continuous_nonlinear_linear_set_distance}
The set distance between the continuous nonlinear and continuous linear reachable sets is bounded: 
\begin{align}
    d_S(\mathcal{R}_{F^H}(\state_0, \actionspace^H), \mathcal{R}_{L^H}(\state_0, \actionspace^H)) \leq 
    \frac{(L_F)^H - 1}{L_F - 1} \left( \frac{1}{2}L_{\nabla F} \varepsilon_H^2 \right)
\end{align}
\end{lemma}

\begin{proof}
Consider two trajectories of length $H$, each driven by the same inputs.
One trajectory satisfies the nonlinear dynamics: $\xi_{k+1} = F(\xi_k, \action_{k+1})$ and one satisfies the linearized dynamics $\eta_{k+1} = L_k(\eta_k, \action_{k+1})$, where $\state_0 = \eta_0 = \xi_0$.
If $F$ has Lipschitz gradients, then the distance between the endpoints of the trajectories is bounded.

Note that, $\forall k\in[0,H-1]$:
\begin{align}
    \|\xi_{k+1} - \eta_{k+1}\| &=
    \|F(\xi_{k}, \action_{k+1}) - L_k(\eta_k, \action_{k+1})\| \\
    &=\|F(\xi_k, \action_{k+1}) - F(\eta_k, \action_{k+1}) + F(\eta_k, \action_{k+1}) - L_k(\eta_k, \action_{k+1})\| \\
    &\textup{by the triangle inequality,} \nonumber \\
    &\leq \|F(\xi_k, \action_{k+1}) - F(\eta_k, \action_{k+1})\| + \|F(\eta_k, \action_{k+1}) - L_k(\eta_k, \action_{k+1})\| \\
    &\textup{as } F \textup{ is Lipschitz, and applying Taylor's Theorem } \nonumber \\
    &\leq L_F \|\xi_k - \eta_k\| + \frac{1}{2}L_{\nabla F} 
    \| (\eta_k, \action_{k+1}) - (\bar{\state}_k, \bar{\action}_{k+1}) \|^2
\end{align}
Applying the discrete Grönwall Lemma~\cite{stuart1998dynamical} and using that $\xi_0 = \eta_0$, the trajectories satisfy
\begin{align}
    \label{eq:def_varepsilon}
    \| \xi_k - \eta_k \| \leq \frac{(L_F)^k - 1}{L_F - 1} \left( \frac{1}{2}L_{\nabla F} \varepsilon_k^2 \right)
    \quad \forall k\in[1,H]
\end{align}
where we let $\varepsilon_k = \max_{\action_{[1:k]}}\| (\eta_k, \action_{k+1}) - (\bar{\state}_k, \bar{\action}_{k+1}) \|$ for notational simplicity.

\end{proof}

\begin{lemma}
\label{lemma:continuous_linear_finite_linear_set_distance}
The set distance between the continuous linear reachable set and the finite set of linear modes is bounded: 
\begin{align}
    d_S (\mathcal{R}_{L^H}(\state_0, \actionspace^H), \{\mathbf{z}^i_H\}_{i=1}^{2n}) 
    &\leq 2 \sigma_\mathrm{max}(\mathcal{C})
\end{align}
\end{lemma}

\begin{proof}
Consider the linear reachable set $\mathcal{R}_{L^H}(\state_0, \actionspace^H)$: 
\begin{align}
    \mathcal{R}_{L^H}(\state_0, \actionspace^H) =
    \{
        \eta_H \ | \ \eta_H = \bar{\eta}_H + \mathcal{C} \mathbf{v}_{[H]} \text{ s.t. } \|\mathbf{v}_{[H]}\|_\infty \leq 1
    \}
\end{align}
where $\bar{\eta}_H = L^H(\state_0, \bar{\action}_{[H]})$ is the endpoint of the free linearized trajectory from $\state_0$ and 
\begin{align}
    \mathcal{C} = 
    \begin{bmatrix}
        \left(\prod_{k=1}^{H-1} A_k \right) B_0 S, &
        \left(\prod_{k=2}^{H-1} A_k \right) B_1 S, & 
        \dots, & 
        A_{H-1} B_{H-2} S, & 
        B_{H-1} S
    \end{bmatrix} 
\end{align} is the input-normalized controllability matrix of the linearized system.
Here $S \mathbf{v}_k = \action_k$ is a linear transformation of the input such that 
$\action_{[H]} \in \actionspace \implies \|\mathbf{v}_{[H]}\|_\infty \leq 1$.

Our algorithm uses a singular value decomposition of $\mathcal{C}$ to extract the controllability information of our linear system. 
The left-singular vectors and singular values are the controllable directions and a measure of their controllability.
Let the set $\{\mathbf{z}^i_H\}_{i=1}^{2n}$ be each left-singular vector of $\mathcal{C}$ times plus/minus its corresponding singular value, which we call the controllable modes of the system.
We now quantify the distance between $\mathcal{R}_{L^H}(\state_0, \actionspace^H)$ and $\{\mathbf{z}^i_H\}_{i=1}^{2n}$.
$\forall \eta_H\in\mathcal{R}_{L^H}(\state_0, \actionspace^H)$, by Hölder's inequality,
\begin{align}
    \| \eta_H - \bar{\eta}_H \|_2 &\leq
    \| \mathcal{C} \mathbf{v}_{[H]} \|_2 
    \leq \| \mathcal{C} \|_2 \| \mathbf{v}_{[H]} \|_\infty 
    \leq \| \mathcal{C} \|_2
\end{align}
where $\| \mathcal{C} \|_2 = \sigma_\mathrm{max}(\mathcal{C})$ is the largest singular value of $\mathcal{C}$.

And now, noting that $\{\mathbf{z}^i_H\}_{i=1}^{2n} \subseteq \mathcal{R}_{L^H}(\state_0, \actionspace^H)$,
\begin{align}
    d_S (\mathcal{R}_{L^H}(\state_0, \actionspace^H), \{\mathbf{z}^i_H\}_{i=1}^{2n}) 
    &= \max_{\eta_H \in \mathcal{R}_{L^H}(\state_0, \actionspace^H)} \min_{i \in [2n]} \| \mathbf{z}^i_H - \eta_H \| \\
    &= \max_{\eta_H \in \mathcal{R}_{L^H}(\state_0, \actionspace^H)} \min_{i \in [2n]} \| \mathbf{z}^i_H - \bar{\eta}_H + \bar{\eta}_H - \eta_H \| \\
    &\leq \max_{\eta_H \in\mathcal{R}_{L^H}(\state_0, \actionspace^H)} \min_{i \in [2n]} \| \mathbf{z}^i_H - \bar{\eta}_H\| + \|\eta_H - \bar{\eta}_H \| \\
    &\leq 2 \sigma_\mathrm{max}(\mathcal{C})
\end{align}
\end{proof}

\begin{lemma}
\label{lemma:set_distance_linear_modes_sets_trajs}
The set distance between the linear modes and the trajectories generated by the spectral expansion procedure is bounded: 
\begin{align}
    d_S(\{\mathbf{z}^i_H\}_{i=1}^{2n}, \{\state^i_H\}_{i=1}^{2n}) \leq  \sqrt{\frac{\overline{m}}{\underline{m}}} \frac{(L_{\nabla F_{cl}}) \varepsilon_H^2}{2} \frac{1 - \alpha^H}{1 - \alpha}
\end{align}
\end{lemma}

\begin{proof}
For each controllable mode, we consider a desired trajectory to be the linear trajectory that terminates at $\mathbf{z}^i_H$.
We fix a controller as a discrete algebraic Riccati controller tracking the desired trajectory $(\mathbf{z}^\mathrm{ref}_{[H]}, \action^\mathrm{ref}_{[H]})$ to form autonomous nonlinear and linearized systems:
\begin{align}
    F_{cl}(\state_k) &= F(\state_k, \action^\mathrm{ref}_{k+1} - \mathcal{K}_k(\state_k - \mathbf{z}^\mathrm{ref}_k)) \\
    L_{cl,k}(\mathbf{z}_k) &= L(\mathbf{z}_k, \action^\mathrm{ref}_{k+1} - \mathcal{K}_k(\mathbf{z}_k - \mathbf{z}^\mathrm{ref}_k))
\end{align}
Applying Lemma~\ref{lemma:dare_linear_contraction}, this controller creates a contracting linear system with contraction rate $\alpha$.
A similar manipulation as in Lemma~\ref{lemma:continuous_nonlinear_linear_set_distance} allows us to bound the tracking error, but now for a contracting system.

We treat the nonlinear system as a perturbed linear system, where the disturbance is the difference between the linear and nonlinear dynamics, and apply a robust contraction result~\cite{LohmillerS98} to bound the difference between the linear and nonlinear trajectory evolution. 
Let $\Theta_k^\top \Theta_k = M_k$ for the contraction metric $M_k$ induced by the controller at time $k$.
Then, for all $k = 0,...,H-1$
\begin{align}
    \state_{k+1} = F_{cl}(\state_k) &= L_{cl,k}(\state_k) + F_{cl}(\state_k) - L_{cl,k}(\state_k) \\
    \state_{k+1} - \mathbf{z}_{k+1} &= L_{cl,k}(\state_k) - L_{cl,k}(\mathbf{z}_k) + F_{cl}(\state_k) - L_{cl,k}(\state_k) \\
    \Theta_k (\state_{k+1} - \mathbf{z}_{k+1}) &= \Theta_k (L_{cl,k}(\state_k) - L_{cl,k}(\mathbf{z}_k)) + \Theta_k (F_{cl}(\state_k) - L_{cl,k}(\state_k)) \\
    \|\Theta_k (\state_{k+1} - \mathbf{z}_{k+1})\| &\leq \alpha \|\Theta_k (\state_k - \mathbf{z}_{k})\| + \| \Theta(F_{cl}(\state_k) - L_{cl,k}(\state_k)) \| \\
    \|\Theta_k (\state_{k+1} - \mathbf{z}_{k+1})\| &\leq \alpha \|\Theta_k(\state_k - \mathbf{z}_{k})\| + \sqrt{\overline{m}} \| F_{cl}(\state_k) - L_{cl,k}(\state_k) \| \\
    \|\Theta_k (\state_{k+1} - \mathbf{z}_{k+1})\| &\leq \alpha \|\Theta_k(\state_k - \mathbf{z}_{k})\| + \frac{\sqrt{\overline{m}}}{2} (L_{\nabla F_{cl}}) \| (\state_k, \action^\mathrm{cl}_{k+1}) - (\bar{\state}_k, \bar{\action}_{k+1}) \|^2
\end{align}
where $\action^\mathrm{cl}_{k+1} = \action^\mathrm{ref}_{k+1} - \mathcal{K}_k(\state_k - \mathbf{z}^\mathrm{ref}_k)$.
Therefore for all $k=1,...,H$:
\begin{align}
    \|\state_k - \mathbf{z}_k\| &\leq \alpha^k \sqrt{\frac{\overline{m}}{\underline{m}}} \|\state_0 - \mathbf{z}_0^\mathrm{ref}\| + \frac{1 - \alpha^k}{1 - \alpha} \sqrt{\frac{\overline{m}}{\underline{m}}} \left( \frac{1}{2} (L_{\nabla F_{cl}}) \varepsilon_k^2 \right)
\end{align}
where $\varepsilon_k$ is defined in Equation~\ref{eq:def_varepsilon}.
We note that $F_{cl}$ inherits Lipschitz gradients from $F$, and furthermore $L_{\nabla F_{cl}} \leq L_{\nabla F}$.
As the trajectories start from the same initial condition, we are left with only the second term. 
Hence the nonlinear rollout error associated with tracking $(\mathbf{z}^\mathrm{ref}_{[H]}, \action^\mathrm{ref}_{[H]})$ is bounded as
\begin{align}
    \|\state_H - \mathbf{z}_H\| &\leq \frac{1 - \alpha^H}{1 - \alpha} \sqrt{\frac{\overline{m}}{\underline{m}}} \left( \frac{1}{2} (L_{\nabla F_{cl}}) \varepsilon_H^2 \right)
\end{align}
where we note that $\alpha < 1$.
As this bound holds for each controllable mode endpoint $\mathbf{z}^i_H$,
\begin{align}
    d_S(\{\mathbf{z}^i_H\}_{i=1}^{2n}, \{\state^i_H\}_{i=1}^{2n}) \leq  \sqrt{\frac{\overline{m}}{\underline{m}}} \frac{(L_{\nabla F_{cl}}) \varepsilon_H^2}{2} \frac{1 - \alpha^H}{1 - \alpha}
\end{align}
\end{proof}

Using the previous results, it is straightforward to derive the main spectral expansion Theorem~\ref{thm:mdp_spectral_search_value_estimate} (restated here for convenience): 
\setcounter{theorem}{0}
\begin{theorem}
Consider an MDP $\left<\statespace, \actionspace, F, R, D, K, \gamma \right>$.
For initial state $\state_0$, Spectral Expansion with horizon $H$ creates a discrete representation with a bounded equivalent optimal value function:
\begin{align}
    | V^*(\state_0) - V^*_\textup{SETS}(\state_0) | 
    &\leq
    \frac{\kappa_0 + \gamma^H }{1-\gamma^H}\left(\kappa_1 (1+\kappa_2 \Delta t)^H + \kappa_3\right),
\end{align}
\noindent where $\kappa_{0,1,2,3}$ are problem-specific constants, and $\Delta t$ is the discretization of continuous-time dynamics.
\end{theorem}

\begin{proof}
Let $\mdp_\mathrm{\acronym}$ denote the new MDP created by Spectral Expansion: $$\mdp_\mathrm{\acronym} = \left<\statespace, \actionspace_\mathrm{\acronym}, F^H, R^H, V, K/H, \gamma^H\right>$$
As $\gamma^H < 1$, there exists a unique optimal value function $V^*_\mathrm{\acronym}$ associated with $\mdp_\mathrm{\acronym}$.
By Equation~\ref{eq:reachable_set_triangle} and Lemmas~\ref{lemma:mdp_equivalence},~\ref{lemma:continuous_nonlinear_linear_set_distance},~\ref{lemma:continuous_linear_finite_linear_set_distance}, and~\ref{lemma:set_distance_linear_modes_sets_trajs}: 
\begin{align} \label{eq:v_H_minus_V_SETS}
    \| (V^H)^* - V^*_\mathrm{\acronym} \|_\infty &\leq
    \frac{L_R + \gamma^H L_V}{1-\gamma^H}
    \bigg(
    \frac{1}{2} (L_{\nabla F}) \varepsilon_H^2 
    \frac{(L_F)^H-1}{L_F-1} \nonumber \\
    &+ 2 \sigma_\mathrm{max}(\mathcal{C}) +
    \sqrt{\frac{\overline{m}}{\underline{m}}} \frac{(L_{\nabla F_{cl}}) \varepsilon_H^2}{2} \frac{1-\alpha^H}{1-\alpha}
    \bigg)
\end{align}
where $(V^H)^*$ is the optimal value of original MDP transcribed as a horizon $K/H$ decision-making problem, where, for each new timestep, $H$ actions are decided simultaneously.
We note that, for deterministic dynamics, $(V^H)^*(\state) = V^*(\state)$ for all $\state \in \statespace$.

We bound the feedback control term as:
\begin{align}
    \| (V^H)^* - V^*_\mathrm{\acronym} \|_\infty &\leq
    \frac{L_R + \gamma^H L_V}{1-\gamma^H}
    \bigg(
    \frac{1}{2} (L_{\nabla F}) \varepsilon_H^2 
    \frac{(L_F)^H-1}{L_F-1} \nonumber \\
    &+ 2 \sigma_\mathrm{max}(\mathcal{C}) +
    \sqrt{\frac{\overline{m}}{\underline{m}}} \frac{(L_{\nabla F_{cl}}) \varepsilon_H^2}{2} \frac{1}{1-\alpha}
    \bigg)
\end{align}
Now rearrangement with
\begin{align}
    \kappa_0 = \frac{L_R}{L_V}, \
    \kappa_1 &= \frac{L_V L_{\nabla F} \varepsilon_H^2}{2(L_F-1)}, \ 
    \kappa_2 = L_f, \\
    \kappa_3 &= 2 L_V \sigma_\mathrm{max}(\mathcal{C}) 
    + \frac{\sqrt{\overline{m}} L_V L_{\nabla F_{cl}} \varepsilon_H^2}{2\sqrt{\underline{m}}(1-\alpha)} -
    \frac{L_V L_{\nabla F} \varepsilon_H^2}{2(L_F-1)}
\end{align}
completes the proof.
\end{proof}

As shown in shown in Fig.~\ref{fig:sets_proof_organization}, the preceding result completes the first analysis group. 
Now, we move to the next (and final) analysis group: MCTS value convergence. 


\begin{theorem}
Consider a node $i$ at depth $d$ in the tree, where the tree is grown using Algorithm~\ref{algo:spectral_search}. 
There exist non-negative constants such that the value's expectation converges to the optimal value: 
\begin{align}
    \abs{ V^*(i) - \cev{V(i,\ell)}{T(i,\ell) = \tau} } \leq \frac{\cfouri}{\tau^{\cfivei}}
\end{align}
and the distribution concentrates to the expected value: $\forall z$, $\forall \tau \geq 1$:
\begin{align}
    \cprob{ \abs{ \ev{V(i,\ell)} - V(i,\ell) } \geq \frac{z}{\tau^{\csixi}} }{T(i,\ell)=\tau} \leq \frac{\cseveni}{z^{\ceighti}}
\end{align}
\end{theorem}

\begin{proof}
We will prove this by induction.
We begin with the base case (Lemmas \ref{lemma:base_convergence} and \ref{lemma:base_concentration}), then proceed backwards:
we will prove that the result holding at a particular depth (Assumption \ref{assumption:induction_hypo}) implies it holds one level up as well (Lemmas \ref{lemma:inductive_convergence} and \ref{lemma:inductive_concentration}).
\end{proof}

\subsection*{Supplemental Base Case: Expected Visits}

Before starting the base case, we provide a useful lemma to bound the number of visits to suboptimal children: 
\begin{lemma}
\label{lemma:base_counts}
For all nodes $i$ at depth $D-1$, the expected number of visits to each suboptimal child $j$, $j\neq j^*$ is upper bounded:
\begin{align}
    \cev{T(j,\ell)}{T(i,\ell)=\tau} \leq 
        \left( \frac{\coneib \tau^{\cthreeib}}{\Delta(j)} \right)^{\frac{1}{\ctwoib}} 
        +1
\end{align}
where 
$\Delta(j) = Q^*(j^*) - Q^*(j)$.
\end{lemma}

\begin{proof}
Fix $\tau$ and let $u$ be arbitrary. Then,
\begin{align}
    &\cev{T(j,\ell)}{T(i,\ell) = \tau} 
    = \cev{\sum_{t \in L(i,\ell)} \ind{\mu(i,t) = j}}{T(i,\ell) = \tau} \\
    & \quad = \sum_{t \in L(i,\ell)} \cprob{\mu(i,t) = j}{T(i,\ell) = \tau} \\
    & \quad = \sum_{t \in L(i,\ell)} \cprob{\mu(i,t) = j, T(j,t) \leq u}{T(i,\ell) = \tau} \\
    &\quad \quad \quad+ \cprob{\mu(i,t) = j, T(j,t) > u}{T(i,\ell) = \tau} \\
    \label{eq:base_counts_tmp}
    & \quad \leq u + \sum_{t \in L(i,\ell)} \cprob{\mu(i,t) = j, T(j,t) > u}{T(i,\ell) = \tau}
\end{align}

At each $t$, a sufficient condition to select child $j$ is that the UCB score of child $j$ is greater than that of the optimal child $j^*$. Hence,
\begin{align}
    &\cprob{\mu(i,t) = j, T(j,t) > u}{T(i,\ell) = \tau} \\
    &\leq \cprob{Q(j^*,t) + \frac{\coneib T(i,t)^{\cthreeib}}{T(j^*,t)^{\ctwoib}} \leq Q(j,t) + \frac{\coneib T(i,t)^{\cthreeib}}{T(j,t)^{\ctwoib}}, T(j,t) > u}{T(i,\ell) = \tau} \\
    &\leq \cprob{Q(j^*,t) \leq Q(j,t) + \frac{\coneib T(i,t)^{\cthreeib}}{T(j,t)^{\ctwoib}}, T(j,t) > u}{T(i,\ell) = \tau}
\end{align}

Select $u = \ceil{\left( \frac{\coneib \tau^{\cthreeib}}{\Delta(j)} \right)^{1/\ctwoib}}$, and note that for all $t \in L(i,\ell)$
\begin{align}
    T(j,t) > u 
    &\implies \Delta(j) > \frac{\coneib T(i,t)^{\cthreeib}}{T(j,t)^{\ctwoib}} 
    \implies Q^*(j^*) > Q^*(j) + \frac{\coneib T(i,t)^{\cthreeib}}{T(j,t)^{\ctwoib}} \\
    &\implies Q(j^*,t) > Q(j,t) + \frac{\coneib T(i,t)^{\cthreeib}}{T(j,t)^{\ctwoib}}
\end{align}
where the first implication is by our choice of $u$ and that $T(i,t) \leq \tau$, the second implication is by the definition of $\Delta(j)$, and the third is by application of leaf node and deterministic rewards. 
Therefore, 
\begin{align}
    \label{eq:base_case_visit_stop}
    \cprob{Q(j^*,t) \leq Q(j,t) + \frac{\coneib T(i,t)^{\cthreeib}}{T(j,t)^{\ctwoib}}, T(j,t) > u}{T(i,\ell) = \tau} = 0
\end{align}
and consequently, suboptimal arm $j$ can never be selected when $T(j,\ell)>u$. 
Application of~\eqref{eq:base_case_visit_stop} to Equation~\eqref{eq:base_counts_tmp} yields the desired result:
\begin{align}
    \cev{T(j,\ell)}{T(i,\ell) = \tau} \leq \left( \frac{\coneib \tau^{\cthreeib}}{\Delta(j)} \right)^{\frac{1}{\ctwoib}} + 1
\end{align}

\end{proof}

\subsection*{Supplemental Base Case: Visits Tail}
\begin{lemma}
\label{lemma:base_visits_tail}
Consider a node $i$ at depth $D-1$ in the tree, where the tree is grown using Algorithm~\ref{algo:spectral_search}. 
Then, exists $\ceightib>0$ such that the probability of selecting suboptimal child $j$ more than $w$ times is upper bounded: 
\begin{align}
    \cprob{T(j,\ell) > w}{T(i,\ell) = \tau} \leq \xi_{D-1}(j) \tau^{\cthreeib \ceightib / \ctwoib} w^{-\ceightib} 
\end{align}
where $\xi_{D-1}(j) = \left( \left( \frac{\coneib}{\Delta(j)} \right)^{1/\ctwoib} + 1 \right)^{\ceightib}$. 
\end{lemma}

\begin{proof}
Apply the law of total probability. Letting $u = \ceil{\left( \frac{\coneib \tau^{\cthreeib}}{\Delta(j)} \right)^{1/\ctwoib}}$ as before,
\begin{align}
    &\cprob{T(j,\ell) > w}{T(i,\ell) = \tau} \\
    &= \underbrace{\cprob{T(j,\ell) > w, w \geq u}{T(i,\ell) = \tau}}_\text{I} + \underbrace{\cprob{T(j,\ell) > w, w < u}{T(i,\ell) = \tau}}_\text{II}
\end{align}

Following our previous analysis (Equation~\eqref{eq:base_case_visit_stop}),
\begin{align}
    \text{I} \leq \cprob{T(j,\ell) > u}{T(i,\ell) = \tau} = 0
\end{align}
We bound term II as
\begin{align}
    \text{II}
    &\leq \prob{w < u}
    \leq 
    \ind{w < u} 
    \leq u^{\ceightib} w^{-\ceightib} 
    = 
    \left( \ceil{\left( \frac{\coneib \tau^{\cthreeib}}{\Delta(j)} \right)^{1/\ctwoib}} \right)^{\ceightib} w^{-\ceightib} 
    \\ &\leq \left( \left( \frac{\coneib}{\Delta(j)} \right)^{1/\ctwoib} + 1 \right)^{\ceightib} \tau^{\cthreeib \ceightib / \ctwoib} w^{-\ceightib} 
\end{align}
where $\ceightib > 0$ is arbitrary.
\end{proof}

\subsection*{Supplemental Base Case: Value Convergence}
\begin{lemma}
\label{lemma:base_convergence}
Consider a node $i$ at depth $D-1$ in the tree, where the tree is grown using Algorithm~\ref{algo:spectral_search}. 
Then, the expectation of the estimated value converges to the optimal value as:
\begin{align}
    \abs{ V^*(i) - \cev{V(i,\ell)}{T(i,\ell)=\tau} } \leq \frac{\cfourib}{\tau^{\cfiveib}}
\end{align}
where $\cfourib = (b-1) \left( \left( \frac{\coneib}{\Delta_\text{min}(i)} \right)^{1/\ctwoib} + 1 \right)$ and $\cfiveib = 1 - \frac{\cthreeib}{\ctwoib}$.
\end{lemma}

\begin{proof}.
We decompose the regret and apply Lemma~\ref{lemma:base_counts} to yield the desired result:
{
\allowdisplaybreaks
\begin{align}
    & V^*(i) - \cev{ V(i,\ell) }{ T(i,\ell)=\tau } \\
    &= \cev{ \frac{1}{\tau} \sum_{t \in L(i,\ell)} \sum_{j \in C(i)} \ind{\mu(i,t)=j} (Q^*(j^*) - q(j,t)) }{T(i,\ell) = \tau} \\
    \label{eq:by_deterministic_leaf_rewards}
    &= \frac{1}{\tau} \sum_{j\in C(i)} \cev{ T(j,\ell) }{T(i,\ell) = \tau} \Delta(j) 
    \\
    &= \frac{1}{\tau} \sum_{j\in C(i), j\neq j^*} \cev{ T(j,\ell) }{T(i,\ell) = \tau} \Delta(j) 
    \\ 
    \label{eq:by_delta_j_leq_1}
    &\leq \frac{1}{\tau} \sum_{j\in C(i), j\neq j^*} \cev{ T(j,\ell) }{T(i,\ell) = \tau} 
    \\
    &\leq \frac{1}{\tau} \sum_{j\in C(i), j\neq j^*} \left( \left( \frac{\coneib \tau^{\cthreeib}}{\Delta(j)} \right)^{1/\ctwoib} + 1 \right) 
    \\
    &\leq \frac{1}{\tau} (b-1) \left( \left( \frac{\coneib \tau^{\cthreeib}}{\Delta_\text{min}(i)} \right)^{1/\ctwoib} + 1 \right) \\
    &\leq \frac{1}{\tau} (b-1) \left( \left( \frac{\coneib}{\Delta_\text{min}(i)} \right)^{1/\ctwoib} + 1 \right) \tau^{\frac{\cthreeib}{\ctwoib}} 
    \\
    \label{eq:constantrules_baseconvergence_0}
    &= (b-1) \left( \left( \frac{\coneib}{\Delta_\text{min}(i)} \right)^{1/\ctwoib} + 1 \right) \tau^{\frac{\cthreeib}{\ctwoib}-1} 
\end{align}
}
where $\Delta_\text{min}(i) = \min_{j \in C(i), j \neq j^*} \Delta(j)$. Equation~\eqref{eq:by_deterministic_leaf_rewards} follows by deterministic leaf rewards and Equation~\eqref{eq:by_delta_j_leq_1} follows by $\Delta(j)\leq 1$. 
\end{proof}

\subsection*{Supplemental Base Case: Value Concentration}
\begin{lemma}
\label{lemma:base_concentration}
Consider a node $i$ at depth $D-1$ in the tree, where the tree is grown using Algorithm~\ref{algo:spectral_search}. 
Then, there exists depth-dependent constants $\csixib, \csevenib>0$, such that, $\forall z > 0$, the estimated value's distribution concentrates to its expected value:
\begin{align}
    \cprob{ \abs{ \ev{V(i,\ell)} - V(i,\ell) } \geq \frac{z}{\tau^{\csixib}} }{T(i,\ell)=\tau} \leq \frac{\csevenib}{z^{\ceightib}} 
\end{align}
\end{lemma}

\begin{proof}
Consider a similar decomposition to the convergence proof in Lemma~\ref{lemma:base_convergence}:

\begin{align}
    &\tau | \ev{V(i,\ell)} - V(i,\ell) |
    \leq \underbrace{| \sum_{t \in L(i,\ell)} \sum_{j \in C(i), j\neq j^*} \ind{\mu(i,t) = j} (\ev{V(i,\ell)} - q(j,t)) |}_\text{I} \\
    &\quad+ \underbrace{|\sum_{t \in L(i,\ell)} \ind{\mu(i,t)=j^*} (\ev{V(i,\ell)} - q(j^*,t)) |}_\text{II}
\end{align}
Note that term I is bounded as
\begin{align}
    \text{I} 
    &\leq \sum_{t \in L(i,\ell)} \sum_{j \in C(i), j\neq j^*} \ind{\mu(i,t) = j} |\ev{V(i,\ell)} - q(j,t)| \\
    &\leq \sum_{t \in L(i,\ell)} \sum_{j \in C(i), j\neq j^*} \ind{\mu(i,t) = j} 
    = \sum_{j \in C(i), j\neq j^*} T(j,\ell)
\end{align}
and term II is bounded with the triangle inequality
\begin{align}
    \text{II}
    &\leq \sum_{t \in L(i,\ell)} \ind{\mu(i,t)=j^*} \bigg| \ev{V(i,\ell)} - q(j^*,t) \bigg| 
    = \sum_{t \in L(j^*,\ell)} |\ev{V(i,\ell)} - q(j^*,t)|
\end{align}
Noting that, as the rewards are deterministic, $q(j^*,t) = V^*(i)$, term II is further manipulated:
\begin{align}
    \text{II}
    &\leq T(j^*,\ell) |\ev{V(i,\ell)} - V^*(i)| 
    \leq T(i,\ell) |\ev{V(i,\ell)} - V^*(i)| 
    \leq \cfourib \tau^{1 - \cfiveib}
\end{align}

{
\allowdisplaybreaks
We can therefore manipulate the desired event probability as, for arbitrary $\csixib$:
\begin{align}
    &\cprob{ \abs{ \ev{V(i,\ell)} - V(i,\ell) } \geq \frac{z}{\tau^{\csixib}} }{T(i,\ell)=\tau} \\
    &= \cprob{ \tau \abs{ \ev{V(i,\ell)} - V(i,\ell) } \geq \frac{\tau z}{\tau^{\csixib}} }{T(i,\ell)=\tau} \\
    &\leq \cprob{\text{I} + \text{II} \geq \frac{\tau z}{\tau^{\csixib}}}{T(i,\ell) = \tau} \\
    &\leq \cprob{\text{I} \geq \frac{\tau z}{2\tau^{\csixib}}}{T(i,\ell) = \tau} + \cprob{\text{II} \geq \frac{\tau z}{2\tau^{\csixib}}}{T(i,\ell) = \tau} \\
    &\leq \cprob{\sum_{\substack{j \in C(i) \\ j\neq j^*}} T(j,\ell) \geq \frac{\tau z}{2\tau^{\csixib}}}{T(i,\ell) = \tau} + \cprob{\cfourib \tau^{1 - \cfiveib} \geq \frac{\tau z}{2\tau^{\csixib}}}{T(i,\ell) = \tau} \\
    \label{eq:base_concentration_tmp1}
    &\leq \sum_{\substack{j \in C(i) \\ j\neq j^*}} \cprob{T(j,\ell) \geq \frac{\tau z}{2(b-1)\tau^{\csixib}}}{T(i,\ell) = \tau} + \cprob{\cfourib \tau^{-\cfiveib} \geq \frac{z}{2\tau^{\csixib}}}{T(i,\ell) = \tau}
\end{align}
}
We bound the first term in~\eqref{eq:base_concentration_tmp1} with Lemma~\ref{lemma:base_visits_tail}. For each suboptimal arm $j$,
\begin{align}
    &\cprob{T(j,\ell) \geq \frac{\tau z}{2(b-1)\tau^{\csixib}}}{T(i,\ell) = \tau} 
    \leq \xi_{D-1}(j) \tau^{\cthreeib \ceightib / \ctwoib} \left( 
        \frac{z}{2(b-1)\tau^{\csixib-1}}
    \right)^{-\ceightib} \\ 
    &= \xi_{D-1}(j) (2(b-1))^{\ceightib} 
    \frac{\tau^{\csixib \ceightib - \ceightib + \cthreeib \ceightib / \ctwoib}}{z^{\ceightib}} 
    \leq \frac{\xi_{D-1}(j) (2(b-1))^{\ceightib}}{z^{\ceightib}}
\end{align}
where the last inequality follows from constraining $\csixib \leq 1 - \cthreeib / \ctwoib$. 

Further constraining $\csixib$ as $\cfiveib = \csixib$ allows us to bound the remaining term in~\eqref{eq:base_concentration_tmp1} as
\begin{align}
    \label{eq:constantrules_baseconcentration_1}
    \cprob{\cfourib \tau^{-\cfiveib} \geq \frac{z}{2\tau^{\csixib}}}{T(i,\ell) = \tau}
    &\leq \cprob{\cfourib \geq \frac{z}{2}}{T(i,\ell) = \tau} 
    \leq \frac{(2\cfourib)^{\ceightib}}{z^{\ceightib}}
\end{align}
The last inequality follows as the event $\{\cfourib \geq \frac{z}{2}\}$ is not random, and can be bounded by a polynomial of arbitrary degree.

Plugging back into Equation~\eqref{eq:base_concentration_tmp1} yields:
\begin{align}
    \eqref{eq:base_concentration_tmp1}
    &\leq \sum_{\substack{j \in C(i) \\ j \neq j^*}} \left( \frac{\xi_{D-1}(j) (2(b-1))^{\ceightib}}{z^{\ceightib}} \right) + \frac{(2\cfourib)^{\ceightib}}{z^{\ceightib}} \\
    &\leq \left( (b-1) \left(\ \underset{j \neq j^*}{\max}\ \xi_{D-1}(j) + (2(b-1))^{\ceightib}\right) + (2\cfourib)^{\ceightib}\right) \frac{1}{z^{\ceightib}}
\end{align}

Combining the expressions, grouping terms, and letting 
$$\csevenib = \left( (b-1) \left(\ \underset{j \neq j^*}{\max}\ \xi_{D-1}(j) + (2(b-1))^{\ceightib}\right) + (2\cfourib)^{\ceightib}\right)$$ allows us to conclude that
\begin{align}
    \label{eq:constantrules_baseconcentration_2}
    \cprob{ \abs{ \ev{V(i,\ell)} - V(i,\ell) } \geq \frac{z}{\tau^{\csixib}} }{T(i,\ell)=\tau} \leq \frac{\csevenib}{z^{\ceightib}} 
\end{align}

\end{proof}

\subsection*{Supplemental Inductive Hypothesis}
\begin{assumption}
\label{assumption:induction_hypo}
Consider a node $i$ at depth $d$ in the tree, where the tree is grown using Algorithm~\ref{algo:spectral_search}. 
There exists constants $\cfourj$, $\cfivej$, $\csixj$, $\csevenj$ such that child node $j \in C(i)$, satisfies the following convergence and concentration properties: 
\begin{align}
    \label{eq:induction_hypo_convergence}
    V^*(j) - \cev{V(j,\ell)}{T(j,\ell)=\tau} \leq \frac{\cfourj}{\tau^{\cfivej}}
\end{align}
and the estimated value's distribution concentrates to its expected value: $\forall z$, $\forall T(j,\ell)$:
\begin{align}
    \label{eq:induction_hypo_concentration}
    \cprob{ \abs{ \ev{V(j,\ell)} - V(j,\ell) } \geq \frac{z}{\tau^{\csixj}} }{T(j,\ell)=\tau} \leq \frac{\csevenj}{z^{\ceightj}}
\end{align}
\end{assumption}

\subsection*{Supplemental Inductive Step: Expected Visits}

\begin{lemma}
\label{lemma:expected_count_upperbound}
Assume Assumption~\ref{assumption:induction_hypo}. Then, for $\tau \geq 1$, the expected number of visits to suboptimal arm $j\in C(i)$, $j\neq j^*$ is upper bounded: 
\begin{align}
    \cev{T(j,\ell)}{T(i,\ell)=\tau} 
    \leq \left( \left( \frac{2 \conei}{ \Delta(j) } \right)^{\frac{1}{\ctwoi}} 
+ \left( \frac{2 \gamma \cfourj}{\conei} \right)^\frac{1}{\cthreei + \csixj - \ctwoi} + 1 + \frac{\csevenj (2\gamma)^{\ceightj} \pi^2}{6(\conei)^{\ceightj}}\right) \tau^{\frac{\cthreei}{\ctwoi}}
\end{align}
\end{lemma}

\begin{proof}

Let
$u = \ceil{ \left( \frac{2 \conei \tau^{\cthreei}}{ \Delta(j) } \right)^{\frac{1}{\ctwoi}} 
    + \left( \frac{2 \gamma \cfourj}{\conei} \right)^\frac{1}{\cthreei + \csixj - \ctwoi} }
$,
noting that $u \geq \left(\frac{2\conei T(i,\ell)^{\cthreei}}{\Delta(j)}\right)^{\frac{1}{\ctwoi}}$. Then
\begin{align}
    T(j,\ell) &= \sum_{t \in L(i, \ell)} \ind{\mu(i,t) = j} \\
    &= \sum_{t \in L(i, \ell)} \ind{\mu(i,t) = j, T(j,t) < u} + \sum_{t \in L(i, \ell)} \ind{\mu(i,t) = j, T(j,t) \geq u} \\
    &\leq u + \sum_{t \in L(i, \ell)} \ind{\mu(i,t) = j, T(j,t) \geq u} \label{eq:Tjl_rv_bound_intermediate_term_0}
\end{align}
Note that, for all $t$, 
\begin{align}
    \{\mu(i,t) = j\} \implies \{ Q(j,t) + \frac{\conei T(i,t)^{\cthreei}}{T(j,t)^{\ctwoi}} > Q^*(j^*) \} \lor \{ Q(j^*,t) + \frac{\conei T(i,t)^{\cthreei}}{T(j^*,t)^{\ctwoi}} \leq Q^*(j^*) \}
\end{align}
and hence
\begin{align}
    \sum_{t \in L(i, \ell)} \ind{\mu(i,t) = j, T(j,t) \geq u} 
    &\leq \sum_{t \in L(i, \ell)} \ind{Q(j,t) + \frac{\conei T(i,t)^{\cthreei}}{T(j,t)^{\ctwoi}} > Q^*(j^*), T(j,t) \geq u} \label{eq:Tjl_rv_bound_intermediate_term_1} \\
    &\quad+ \sum_{t \in L(i, \ell)} \ind{Q(j^*,t) + \frac{\conei T(i,t)^{\cthreei}}{T(j^*,t)^{\ctwoi}} \leq Q^*(j^*), T(j,t) \geq u}
\end{align}
Manipulating the first term in the right-hand side of Equation~\eqref{eq:Tjl_rv_bound_intermediate_term_1}, for each $t \in L(i, \ell)$,
\begin{align}
    &\sum_{t \in L(i, \ell)} \ind{Q(j,t) + \frac{\conei T(i,t)^{\cthreei}}{T(j,t)^{\ctwoi}} > Q^*(j^*), T(j,t) \geq u} \\
    &= \ind{Q(j,t) - Q^*(j) + \frac{\conei T(i,t)^{\cthreei}}{T(j,t)^{\ctwoi}} > Q^*(j^*) - Q^*(j), T(j,t) \geq u} \nonumber \\
    &= \ind{Q(j,t) - Q^*(j) > \Delta(j) - \frac{\conei T(i,t)^{\cthreei}}{T(j,t)^{\ctwoi}}, T(j,t) \geq u} \label{eq:Tjl_rv_bound_intermediate_term_3}
\end{align}
Note that $T(j,t) \geq u$ and $T(i,t) \leq T(i,\ell)$, so
$u \geq \left(\frac{2\conei T(i,\ell)^{\cthreei}}{\Delta(j)}\right)^{\frac{1}{\ctwoi}} \implies \frac{\conei T(i,t)^{\cthreei}}{T(j,t)^{\ctwoi}} \leq \frac{\Delta(j)}{2}$. 
Therefore
\begin{align}
    \eqref{eq:Tjl_rv_bound_intermediate_term_3}
    &\leq \ind{Q(j,t) - Q^*(j) > \frac{\Delta(j)}{2}, T(j,t) \geq u}
\end{align}
As the rewards are deterministic, $Q(j,t) \leq Q^*(j)$ for all $t$.
Noting that $j$ is a suboptimal arm, $\Delta(j)/2 > 0$, so $\ind{Q(j,t) - Q^*(j) > \frac{\Delta(j)}{2}, T(j,t) \geq u} = 0$ with probability one.

Plugging into Equation~\eqref{eq:Tjl_rv_bound_intermediate_term_0} yields
\begin{align}
    T(j,\ell) \leq u + \sum_{t \in L(i, \ell)} \ind{Q(j^*,t) + \frac{\conei T(i,t)^{\cthreei}}{T(j^*,t)^{\ctwoi}} \leq Q^*(j^*), T(j,t) \geq u}
\end{align}
By linearity of conditional expectation,
\begin{align}
    &\cev{T(j,\ell)}{T(i,\ell) = \tau}\\
    &\leq u + \sum_{t \in L(i,\ell)} \underbrace{\cprob{Q(j^*,t) + \frac{\conei T(i,t)^{\cthreei}}{T(j^*,t)^{\ctwoi}} \leq Q^*(j^*), T(j,t) \geq u}{T(i,n) = \tau}}_\text{I} \label{eq:etjl_intermediate_0}
\end{align}

Manipulating term I, for each $t \in L(i,\ell)$,
\begin{align}
    \text{I}&= \mathbb{P} \bigg[Q(j^*,t) - \ev{Q(j^*,t)} + \frac{\conei T(i,t)^{\cthreei}}{T(j^*,t)^{\ctwoi}} \leq Q^*(j^*) - \ev{Q(j^*,t)}, \nonumber \\
    &\qquad \qquad T(j,t) \geq u \mid T(i,\ell) = \tau \bigg] \\
    &= \mathbb{P} \bigg[\ev{Q(j^*,t)} - Q(j^*,t) \geq \frac{\conei T(i,t)^{\cthreei}}{T(j^*,t)^{\ctwoi}} - Q^*(j^*) + \ev{Q(j^*,t)}, \nonumber \\
    &\qquad \qquad T(j,t) \geq u \mid T(i,\ell) = \tau \bigg] \\
    &= \mathbb{P} \bigg[\ev{V(j^*,t)} - V(j^*,t) \geq \frac{\conei T(i,t)^{\cthreei}}{\gamma T(j^*,t)^{\ctwoi}} - V^*(j^*) + \ev{V(j^*,t)}, \nonumber \\
    &\qquad \qquad T(j,t) \geq u \mid T(i,\ell) = \tau \bigg] \\
    &\leq \cprob{\ev{V(j^*,t)} - V(j^*,t) \geq \frac{\conei T(i,t)^{\cthreei}}{\gamma T(j^*,t)^{\ctwoi}} - \frac{\cfourj}{T(j^*,t)^{\cfivej}}, T(j,t) \geq u}{T(i,\ell) = \tau} \label{eq:etjl_intermediate_1}
\end{align}
When we constrain $\cfivej = \csixj$,

\begin{align}
    \eqref{eq:etjl_intermediate_1}
    &=
    \mathbb{P}\bigg[\ev{V(j^*,t)} - V(j^*,t) \geq \frac{\frac{1}{\gamma}\conei T(i,t)^{\cthreei}T(j^*,t)^{\csixj - \ctwoi} - \cfourj}{T(j^*,t)^{\csixj}}, \\
    &\qquad \qquad 
    T(j,t) \geq u \ \mid \ T(i,\ell) = \tau \bigg] \\
    &\leq \frac{\csevenj}{(\frac{1}{\gamma}\conei T(i,t)^{\cthreei}T(j^*,t)^{\csixj - \ctwoi} - \cfourj)^{\ceightj}} \label{eq:etjl_intermediate_2}
\end{align}
by the inductive hypothesis, Assumption~\ref{assumption:induction_hypo}.
Further constraining $\ctwoi \geq \csixj$, note that $T(j^*,t)^{-(\ctwoi - \csixj)} \leq T(i,t)^{-(\ctwoi - \csixj)}$.
Additionally, using our choice of $u$, note $\frac{1}{2\gamma}\conei u^{(\cthreei + \csixj - \ctwoi)} - \cfourj \geq 0$. 
Consequently $T(j,t) \geq u \implies T(i,t) \geq u \implies \frac{1}{2\gamma}\conei T(i,t)^{(\cthreei + \csixj - \ctwoi)} - \cfourj \geq 0$ when $\cthreei + \csixj - \ctwoi > 0$. Therefore,
\begin{align}
    \eqref{eq:etjl_intermediate_2} 
    &\leq \frac{\csevenj}{(\frac{1}{\gamma}\conei T(i,t)^{\cthreei + \csixj - \ctwoi} - \cfourj)^{\ceightj}} 
    \leq \frac{\csevenj}{(\frac{1}{2\gamma}\conei T(i,t)^{\cthreei + \csixj - \ctwoi})^{\ceightj}}
\end{align}

Returning to Equation~\eqref{eq:etjl_intermediate_0},
\begin{align}
    \eqref{eq:etjl_intermediate_0}
    &\leq u + \sum_{t \in L(i,\ell)} \frac{\csevenj}{(\frac{1}{2\gamma}\conei T(i,t)^{\cthreei + \csixj - \ctwoi})^{\ceightj}} \\
    &= u + \frac{\csevenj (2\gamma)^{\ceightj}}{(\conei)^{\ceightj}} \sum_{t \in L(i,\ell)} T(i,t)^{-\ceightj(\cthreei + \csixj - \ctwoi)} \label{eq:etjl_intermediate_3}
\end{align}
For visit times $t_1, \dots, t_\tau \in L(i,\ell)$, we can reindex the sum in Equation~\eqref{eq:etjl_intermediate_3} as $t_s \rightarrow s$:
\begin{align}
    \eqref{eq:etjl_intermediate_3}
    &= u + \frac{\csevenj (2\gamma)^{\ceightj}}{(\conei)^{\ceightj}}
    \sum_{s = 1}^\tau
    s^{-\ceightj(\cthreei + \csixj - \ctwoi)} \label{eq:etjl_intermediate_4}
\end{align}
When we select $\ceightj$ such that $\ceightj(\cthreei + \csixj - \ctwoi) \geq 2$,
\begin{align}
    \eqref{eq:etjl_intermediate_4}
    &\leq u + \frac{\csevenj (2\gamma)^{\ceightj}}{(\conei)^{\ceightj}} \sum_{s = 1}^\infty s^{-2}
    = u + \frac{\csevenj (2\gamma)^{\ceightj} \pi^2}{6(\conei)^{\ceightj}} \label{eq:etjl_intermediate_5}
\end{align}

Plug the definition of $u$ into Equation~\eqref{eq:etjl_intermediate_5} and note that
\begin{align}
    \eqref{eq:etjl_intermediate_5}
    &=\ceil{ \left( \frac{2 \conei \tau^{\cthreei}}{ \Delta(j) } \right)^{\frac{1}{\ctwoi}} 
    + \left( \frac{2 \gamma \cfourj}{\conei} \right)^\frac{1}{\cthreei + \csixj - \ctwoi} } + \frac{\csevenj (2\gamma)^{\ceightj} \pi^2}{6(\conei)^{\ceightj}} \\
    &\leq \left( \frac{2 \conei \tau^{\cthreei}}{ \Delta(j) } \right)^{\frac{1}{\ctwoi}} 
    + \left( \frac{2 \gamma \cfourj}{\conei} \right)^\frac{1}{\cthreei + \csixj - \ctwoi} + 1 + \frac{\csevenj (2\gamma)^{\ceightj} \pi^2}{6(\conei)^{\ceightj}} \\
    &\leq \left( \left( \frac{2 \conei}{ \Delta(j) } \right)^{\frac{1}{\ctwoi}} 
    + \left( \frac{2 \gamma \cfourj}{\conei} \right)^\frac{1}{\cthreei + \csixj - \ctwoi} + 1 + \frac{\csevenj (2\gamma)^{\ceightj} \pi^2}{6(\conei)^{\ceightj}}\right) \tau^{\frac{\cthreei}{\ctwoi}}
\end{align}
as $\tau \geq 1$ and $\frac{\cthreei}{\ctwoi} > 0$.

\end{proof}

\subsection*{Supplemental Inductive Step: Visit Concentration}

\begin{lemma}
\label{lemma:induction_visits_tail}
For all $w > 0$ and $\tau \geq 1$,
\begin{align}
    \cprob{T(j,\ell) > w}{T(i,\ell) = \tau} \leq \xi_d(j) \tau^{\cthreei \ceighti / \ctwoi} w^{-\ceighti} 
\end{align}
where $\ceighti \leq \ceightj (\cthreei + \csixj - \ctwoi) - 1$ and $\xi_d(j) = \left( \left( \left( \frac{2 \conei}{ \Delta(j) } \right)^{\frac{1}{\ctwoi}} 
    + \left( \frac{2 \gamma \cfourj}{\conei} \right)^\frac{1}{\cthreei + \csixj - \ctwoi} + 1 \right)^{\ceighti} \right.$ $\left.+\frac{\csevenj (2\gamma)^{\ceightj}}{(\conei)^{\ceightj}(\ceightj(\cthreei + \csixj - \ctwoi) - 1)}\right)$.
\end{lemma}

\begin{proof}

Let $u = \ceil{ \left( \frac{2 \conei \tau^{\cthreei}}{ \Delta(j) } \right)^{\frac{1}{\ctwoi}} 
+ \left( \frac{2 \gamma \cfourj}{\conei} \right)^\frac{1}{\cthreei + \csixj - \ctwoi} }$ as in Lemma~\ref{lemma:expected_count_upperbound}.

By the law of total probability: 
\begin{align}
    \cprob{T(j,\ell) > w}{T(i,\ell) = \tau} 
    &= \underbrace{\cprob{T(j,\ell) > w, w < u}{T(i,\ell) = \tau}}_\text{I} \nonumber \\
    &+ \underbrace{\cprob{T(j,\ell) > w, w \geq u}{T(i,\ell) = \tau}}_\text{II} \label{eq:tjl_concentration_int_0}
\end{align}

Term I
is upper bounded as:
\begin{align}
    \label{eq:inductive_visits_tail_pt1}
    \cprob{T(j,\ell) > w, w < u}{T(i,\ell) = \tau}
    &\leq \cprob{w < u}{T(i,\ell)=\tau} \\
    &= \ind{w < u} \leq \frac{u^{\ceightj(\cthreei + \csixj - \ctwoi)-1}}{w^{\ceightj(\cthreei + \csixj - \ctwoi)-1}}
\end{align}

To bound term II, we adapt the following logic from~\cite{audibert2009exploration}.
Consider the following events:
\begin{align}
    A &=  \{ \forall t \in L(i,\ell), T(j,t) > w \implies Q(j,t) + \frac{\conei T(i,t)^{\cthreei}}{T(j,t)^{\ctwoi}} < Q^*(j^*) \} \\
    B &= \{ \forall t \in L(i,\ell), T(j^*,t) < T(i,t) - w \implies Q(j^*,t) + \frac{\conei T(i,t)^{\cthreei}}{T(j^*,t)^{\ctwoi}} \geq Q^*(j^*) \}
\end{align}
$A$ is the event that: if more than $w$ visits have been to arm $j$, then its UCB score is below the optimal value of the optimal arm.
$B$ is the event that: if at least $w$ visits have been made to non-optimal arms, then the UCB score of the optimal arm is greater than or equal to the value of the optimal arm.

We note that $\{A \land B\} \implies \{T(j,\ell) \leq w\}$, as if the UCB score of arm $j$ is less than or equal to $Q^*(j^*)$ and the UCB score of arm $j^*$ is greater than $Q^*(j^*)$, then arm $j$ cannot be selected.
Considering the contrapositive, $\{ T(j,\ell) > w \} \implies \{ \neg A \text{ or } \neg B \}$, where
\begin{align}
    \neg A &= \{ \exists t \in L(i,\ell) \text{ s.t. } T(j,t) > w \text{ and } Q(j,t) + \frac{\conei T(i,t)^{\cthreei}}{T(j,t)^{\ctwoi}} \geq Q^*(j^*) \} \\
    \neg B &= \{ \exists t \in L(i,\ell) \text{ s.t. } T(j^*,t) < T(i,t) - w \text{ and } Q(j^*,t) + \frac{\conei T(i,t)^{\cthreei}}{T(j^*,t)^{\ctwoi}} < Q^*(j^*) \}
\end{align}

Consequently,
\begin{align}
    &\cprob{T(j,\ell) > w, w \geq u}{T(i,\ell)=\tau} 
    \leq \cprob{(\neg A \text{ or } \neg B) \text{ and } w \geq u}{T(i,\ell)=\tau} \\
    \label{eq:inductive_visits_tail_pt2pt0}
    & \qquad \leq \underbrace{\cprob{\neg A, w \geq u }{T(i,\ell)=\tau}}_\text{III} + 
    \underbrace{\cprob{\neg B, w \geq u }{T(i,\ell)=\tau}}_\text{IV}
\end{align}

First considering term III,
identical reasoning to Lemma~\ref{lemma:expected_count_upperbound} shows this term equals zero:
\begin{align}
    \text{III} 
    &= \cprob{\exists t \in L(i,\ell) \text{ s.t. } T(j,t) > w, Q(j,t) + \frac{\conei T(i,t)^{\cthreei}}{T(j,t)^{\ctwoi}} \geq Q^*(j^*), w \geq u}{T(i,\ell) = \tau} \\
    &\leq \sum_{t \in L(i,\ell)} \cprob{Q(j,t) + \frac{\conei T(i,t)^{\cthreei}}{T(j,t)^{\ctwoi}} \geq Q^*(j^*), T(j,t) > w, w \geq u}{T(i,\ell) = \tau} \\
    &\leq \sum_{t \in L(i,\ell)} \cprob{Q(j,t) + \frac{\conei T(i,t)^{\cthreei}}{T(j,t)^{\ctwoi}} \geq Q^*(j^*), T(j,t) > u}{T(i,\ell) = \tau} = 0
\end{align}

{
\allowdisplaybreaks
Next, we consider the second term of~\eqref{eq:inductive_visits_tail_pt2pt0}. 
Note that
\begin{align}
&\prob{\neg B} = \prob{\exists t \in L(i,\ell) \text{ s.t. } T(j^*,t) < T(i,t) - w \text{ and } Q(j^*,t) + \frac{\conei T(i,t)^{\cthreei}}{T(j^*,t)^{\ctwoi}} < Q^*(j^*)} \\
&\leq \mathbb{P}\bigg[ \exists s \in [1, T(i,\ell) - w] \text{ where } T(j^*,t) = s \\
&\qquad \qquad \text{ s.t. } T(i,t) \geq w + s, \text{ and } Q(j^*,t) + \frac{\conei T(i,t)^{\cthreei}}{T(j^*,t)^{\ctwoi}} < Q^*(j^*) \bigg] \\
&\leq \mathbb{P}\bigg[\exists s \in [1, T(i,\ell) - w] \text{ where } T(j^*,t) = s \\ 
&\qquad \qquad \text{ s.t. } T(i,t) \geq w + s, Q(j^*,t) + \frac{\conei (w+s)^{\cthreei}}{s^{\ctwoi}} < Q^*(j^*) \bigg] \\
&\leq 
\sum_{s=1}^{T(i,\ell) - w}
\prob{Q(j^*,t) + \frac{\conei (w+s)^{\cthreei}}{s^{\ctwoi}} < Q^*(j^*), T(i,t) \geq w + s \text{ where } T(j^*,t) = s}
\end{align}
and consequently, where, for each $s$, let $t$ be such that $T(j^*,t) = s$:
\begin{align}
    \text{IV} &\leq
    \sum_{s = 1}^{\tau - w} \cprob{Q(j^*,t) + \frac{\conei (w+s)^{\cthreei}}{s^{\ctwoi}} < Q^*(j^*), T(i,t) \geq w + s, w \geq u}{T(i,\ell) = \tau} \\
    &\leq \sum_{s = 1}^{\tau - w} \cprob{Q(j^*,t) + \frac{\conei (w+s)^{\cthreei}}{s^{\ctwoi}} < Q^*(j^*), w \geq u}{T(i,\ell) = \tau} \\
    &\leq \sum_{s = 1}^{\tau - w} \mathbb{P}\bigg[Q(j^*,t) - \ev{Q(j^*,t)} + \frac{\conei (w+s)^{\cthreei}}{s^{\ctwoi}} < Q^*(j^*) - \ev{Q(j^*,t)}, \\
    &\qquad \qquad w \geq u \ \mid \ T(i,\ell) = \tau \bigg] \\
    &= \sum_{s = 1}^{\tau - w} \mathbb{P}\bigg[V(j^*,t) - \ev{V(j^*,t)} + \frac{\conei (w+s)^{\cthreei}}{\gamma s^{\ctwoi}} < V^*(j^*) - \ev{V(j^*,t)}, \\
    &\qquad \qquad w \geq u \ \mid \ T(i,\ell) = \tau \bigg] \\
    &= \sum_{s = 1}^{\tau - w} \mathbb{P}\bigg[\ev{V(j^*,t)} - V(j^*,t) > \frac{\conei (w+s)^{\cthreei}}{\gamma s^{\ctwoi}} + \ev{V(j^*,t)} - V^*(j^*), \\
    &\qquad \qquad w \geq u \ \mid \ T(i,\ell) = \tau \bigg] \\
    &\leq \sum_{s = 1}^{\tau - w} \mathbb{P}\bigg[\ev{V(j^*,t)} - V(j^*,t) > \frac{\conei (w+s)^{\cthreei}}{\gamma s^{\ctwoi}} - \frac{\cfourj}{s^{\cfivej}}, w \geq u \ \mid \ T(i,\ell) = \tau \bigg] \label{eq:tjl_concentration_int_1}
\end{align}
where we've applied the inductive convergence hypothesis.
Constraining $\cfivej = \csixj$,
\begin{align}
    \eqref{eq:tjl_concentration_int_1}
    &= \sum_{s = 1}^{\tau - w} \mathbb{P}\bigg[\ev{V(j^*,t)} - V(j^*,t) > \frac{\frac{\conei}{\gamma} (w+s)^{\cthreei} s^{\csixj - \ctwoi} - \cfourj}{s^{\csixj}}, w \geq u \mid T(i,\ell) = \tau \bigg] \\
    &\leq \sum_{s = 1}^{\tau - w} \frac{\csevenj}{(\frac{\conei}{\gamma}(w+s)^{\cthreei} s^{\csixj - \ctwoi} - \cfourj)^{\ceightj}} \label{eq:tjl_concentration_int_2}
\end{align}
by applying the inductive concentration hypothesis. Further constraining $\ctwoi \geq \csixj$, note $\frac{1}{s^{\ceightj(\csixj - \ctwoi)}} \leq \frac{1}{(w+s)^{\ceightj(\csixj - \ctwoi)}}$. Consequently,
\begin{align}
    \eqref{eq:tjl_concentration_int_2}
    &\leq \sum_{s = 1}^{\tau - w} \frac{\csevenj}{(\frac{\conei}{\gamma}(w+s)^{(\cthreei + \csixj - \ctwoi)} - \cfourj)^{\ceightj}} \label{eq:tjl_concentration_int_3}
\end{align}
Noting that $w \geq u$ and $u \geq \left( \frac{2 \gamma \cfourj}{\conei} \right)^\frac{1}{\cthreei + \csixj - \ctwoi}$, we have that $\frac{\conei}{\gamma} (w+s)^{(\cthreei + \csixj - \ctwoi)} - \cfourj \geq \frac{\conei}{2\gamma} (w+s)^{(\cthreei + \csixj - \ctwoi)}$.
As such,
\begin{align}
    \eqref{eq:tjl_concentration_int_3}
    &\leq \sum_{s = 1}^{\tau - w} \frac{\csevenj}{(\frac{\conei}{2\gamma}(w+s)^{(\cthreei + \csixj - \ctwoi)})^{\ceightj}} \\
    &= \frac{\csevenj (2\gamma)^{\ceightj}}{(\conei)^{\ceightj}} \sum_{s = 1}^{\tau - w} \frac{1}{(w+s)^{\ceightj(\cthreei + \csixj - \ctwoi)}} \\
    &\leq \frac{\csevenj (2\gamma)^{\ceightj}}{(\conei)^{\ceightj}} \int_{1}^{\tau - w} \frac{ds}{(w+s)^{\ceightj(\cthreei + \csixj - \ctwoi)}} \\
    &= \frac{\csevenj (2\gamma)^{\ceightj}}{(\conei)^{\ceightj}} \left[ \frac{(w+s)^{1 - \ceightj(\cthreei + \csixj - \ctwoi)}}{1 - \ceightj(\cthreei + \csixj - \ctwoi)} \right]_{s=1}^{\tau-w} \\
    &= \frac{\csevenj (2\gamma)^{\ceightj}}{(\conei)^{\ceightj}} \left( \frac{(w+1)^{1 - \ceightj(\cthreei + \csixj - \ctwoi)} - \tau^{1 - \ceightj(\cthreei + \csixj - \ctwoi)}}{\ceightj(\cthreei + \csixj - \ctwoi) - 1} \right) \\
    &\leq \frac{\csevenj (2\gamma)^{\ceightj}}{(\conei)^{\ceightj}} \left( \frac{(w+1)^{1 - \ceightj(\cthreei + \csixj - \ctwoi)}}{\ceightj(\cthreei + \csixj - \ctwoi) - 1} \right) \label{eq:tjl_concentration_int_4}
\end{align}
where we constrain $\ceightj(\cthreei + \csixj - \ctwoi) - 1 > 0$.
As such,
\begin{align}
    \eqref{eq:tjl_concentration_int_4}
    &\leq \frac{\csevenj (2\gamma)^{\ceightj}}{(\conei)^{\ceightj}(\ceightj(\cthreei + \csixj - \ctwoi) - 1)} \frac{1}{w^{(\ceightj(\cthreei + \csixj - \ctwoi) - 1)}}
\end{align}

Substituting the bounds of terms I and II into Equation~\eqref{eq:tjl_concentration_int_0} and letting $\ceighti \leq (\ceightj(\cthreei + \csixj - \ctwoi) - 1)$ yields:
\begin{align}
    &\cprob{T(j,\ell) > w}{T(i,\ell) = \tau} \\
    &\leq \frac{u^{\ceighti}}{w^{\ceighti}}
    + \frac{\csevenj (2\gamma)^{\ceightj}}{(\conei)^{\ceightj}(\ceightj(\cthreei + \csixj - \ctwoi) - 1)} \frac{1}{w^{\ceighti}} \\
    &\leq \frac{1}{w^{\ceighti}}
    \left( \left( \left( \frac{2 \conei \tau^{\cthreei}}{ \Delta(j) } \right)^{\frac{1}{\ctwoi}} 
    + \left( \frac{2 \gamma \cfourj}{\conei} \right)^\frac{1}{\cthreei + \csixj - \ctwoi} + 1 \right)^{\ceighti} \hspace{-10pt} + \frac{\csevenj (2\gamma)^{\ceightj}}{(\conei)^{\ceightj}(\ceightj(\cthreei + \csixj - \ctwoi) - 1)}\right) \\
    &\leq \frac{\tau^{\frac{\cthreei \ceighti}{\ctwoi}}}{w^{\ceighti}} \left( \left( \left( \frac{2 \conei}{ \Delta(j) } \right)^{\frac{1}{\ctwoi}} 
    + \left( \frac{2 \gamma \cfourj}{\conei} \right)^\frac{1}{\cthreei + \csixj - \ctwoi} + 1 \right)^{\ceighti} \hspace{-10pt} + \frac{\csevenj (2\gamma)^{\ceightj}}{(\conei)^{\ceightj}(\ceightj(\cthreei + \csixj - \ctwoi) - 1)}\right)
\end{align}
Letting 
\begin{align}
    \xi_d(j) = \left( \left( \left( \frac{2 \conei}{ \Delta(j) } \right)^{\frac{1}{\ctwoi}} 
    + \left( \frac{2 \gamma \cfourj}{\conei} \right)^\frac{1}{\cthreei + \csixj - \ctwoi} + 1 \right)^{\ceighti} \hspace{-10pt} + \frac{\csevenj (2\gamma)^{\ceightj}}{(\conei)^{\ceightj}(\ceightj(\cthreei + \csixj - \ctwoi) - 1)}\right)
\end{align}
completes the proof.
}
\end{proof}

\subsection*{Supplemental Inductive Step: Value Convergence}

\begin{lemma}\label{lemma:inductive_convergence}
Assume Assumption 1. Then exists constants $\cfouri$, $\cfivei$ such that for all $\tau$:
\begin{align}
    V^*(i) - \cev{V(i,\ell)}{T(i,\ell) = \tau} \leq \frac{\cfouri}{\tau^{\cfivei}}
\end{align}
\end{lemma}
 
\begin{proof}
We use the definition of $V(i,\ell)$, and split the contribution from the optimal arm $j^*$ and the suboptimal arms $j \neq j^*$:
\begin{align}
    &V^*(i) - \cev{V(i,\ell)}{T(i,\ell) = \tau} = \\
    & \qquad =
    V^*(i) - \cev{\frac{1}{\tau} \sum_{t\in L(i,\ell)}
        \sum_{j \in C(i)}\ind{\mu(i,t) = j} q(j,t)}{T(i,\ell) = \tau} \nonumber \\
    & \qquad =
    \underbrace{\left(
        \frac{1}{\tau}
        \cev{
        \sum_{t\in L(i,\ell)}
        \sum_{j \in C(i), j \neq j^*} \ind{\mu(i,t) = j} (Q^*(j^*) - q(j,t)) }{T(i,\ell) = \tau}
    \right)}_\text{I} \nonumber \\ 
    & \qquad +
    \underbrace{\left( 
        \frac{1}{\tau}
        \cev{
        \sum_{t\in L(i,\ell)}
        \ind{\mu(i,t) = j^*} (Q^*(j^*) - q(j^*,t)) }{T(i,\ell) = \tau} 
    \right)}_\text{II}
\end{align}
The contribution of the suboptimal arms (term I) is bounded by bounding the number of visits to each suboptimal arm $j\neq j^*$ with Lemma~\ref{lemma:expected_count_upperbound}.
$Q^*(j^*) - q(j,t) \leq \frac{1 - \gamma^D}{1 - \gamma^{d}}$.
First, we note that, using bounded stage-rewards assumption, the differences of all action-value functions is bounded: $\forall Q_1, Q_2 \in\mathcal{Q}$, $\forall p, q$, 
\begin{align}
    \label{eq:bounded_qvalues}
    Q_1(p) - Q_2(q) \leq \frac{1 - \gamma^D}{1 - \gamma^{d}}. 
\end{align}
Then: 
\begin{align} 
    \text{I}&\leq \frac{1}{\tau} \left(\frac{1 - \gamma^D}{1 - \gamma^{d}} \right) \cev{\sum_{t\in L(i,\ell)}
    \sum_{j \in C(i), j \neq j^*} \ind{ \mu(i,t) = j} }{T(i,\ell) = \tau} \\
    &= \frac{1}{\tau} \left(\frac{1 - \gamma^D}{1 - \gamma^{d}} \right)  \sum_{j \in C(i), j \neq j^*} \cev{T(j,\ell)}{T(i,\ell) = \tau} \\ 
    &\leq \frac{1}{\tau} \left(\frac{1 - \gamma^D}{1 - \gamma^{d}} \right)  \sum_{j \in C(i), j \neq j^*} \left(\left( \frac{2 \conei}{ \Delta(j) } \right)^{\frac{1}{\ctwoi}} 
    + \left( \frac{2 \gamma \cfourj}{\conei} \right)^\frac{1}{\cthreei + \csixj - \ctwoi} + 1 + \frac{\csevenj (2\gamma)^{\ceightj} \pi^2}{6(\conei)^{\ceightj}}\right) \tau^{\frac{\cthreei}{\ctwoi}} \\
    &\leq (b-1) \left(\frac{1 - \gamma^D}{1 - \gamma^{d}} \right) \left(\left( \frac{2 \conei}{ \Delta_\text{min}(i) } \right)^{\frac{1}{\ctwoi}} 
    + \left( \frac{2 \gamma \cfourj}{\conei} \right)^\frac{1}{\cthreei + \csixj - \ctwoi} + 1 + \frac{\csevenj (2\gamma)^{\ceightj} \pi^2}{6(\conei)^{\ceightj}}\right) \tau^{\frac{\cthreei}{\ctwoi}-1}
\end{align} 

The contribution from the optimal arm (term \text{II}) is bounded by application of the inductive convergence hypothesis. 
We apply the following inequality, which we will prove in Equations~\eqref{eq:marginalization_trick_start}~--~\eqref{eq:marginalization_trick_stop}:
\begin{align}
    \label{eq:temp_assumption}
    \frac{1}{\tau} \cev{ \sum_{t\in L(j^*,\ell)} Q^*(j^*) - q(j^*,t) }{T(i,\ell) = \tau} \leq Q^*(j^*) - \cev{ Q(j^*,\ell) }{T(j^*,\ell) = \tau}
\end{align}
Then, the desired result follows easily: 
\begin{align}
    \text{II} 
    &= \frac{1}{\tau} \cev{ \sum_{t\in L(j^*,\ell)} Q^*(j^*) - q(j^*,t) }{T(i,\ell) = \tau}  
    \leq Q^*(j^*) - \cev{ Q(j^*,\ell) }{T(j^*,\ell) = \tau} \nonumber \\ 
    &= \gamma(V^*(j^*) - \cev{V(j^*,\ell)}{T(j^*,\ell) = \tau})
    \leq \frac{\gamma \cfourj}{\tau^{\cfivej}}
\end{align}

Combining the upper bounds for terms \text{I} and $\text{II}$: 
\begin{align}
    \label{eq:constantrules_inductiveconvergence_0}
    &V^*(i) - \cev{V(i,\ell)}{T(i,\ell) = \tau} 
    \leq \frac{\gamma \cfourj}{\tau^{\cfivej}} \\
    & + (b-1) \left(\frac{1 - \gamma^D}{1 - \gamma^{d}} \right) \left(\left( \frac{2 \conei}{ \Delta_\text{min}(i) } \right)^{\frac{1}{\ctwoi}} 
    + \left( \frac{2 \gamma \cfourj}{\conei} \right)^\frac{1}{\cthreei + \csixj - \ctwoi} + 1 + \frac{\csevenj (2\gamma)^{\ceightj} \pi^2}{6(\conei)^{\ceightj}}\right) \tau^{\frac{\cthreei}{\ctwoi}-1}
\end{align}
where letting $\cfouri = (b-1) \left(\frac{1 - \gamma^D}{1 - \gamma^{d}} \right) \left(\left( \frac{2 \conei}{ \Delta_\text{min}(i) } \right)^{\frac{1}{\ctwoi}} 
+ \left( \frac{2 \gamma \cfourj}{\conei} \right)^\frac{1}{\cthreei + \csixj - \ctwoi} + 1 + \frac{\csevenj (2\gamma)^{\ceightj} \pi^2}{6(\conei)^{\ceightj}}\right) + \gamma \cfourj$, $\cfivei \leq 1 - \frac{\cthreei}{\ctwoi}$, and $\cfivei \leq \cfivej$ completes the proof. 

It remains to show~\eqref{eq:temp_assumption}. Although it is an intuitive statement, it requires a technical manipulation of the random variables. In particular, we use the law of total probability while marginalizing over: (i) all sets of visit times $L(i,\ell)$ to node $i$, (ii) visit counts $T(j^*,\ell)$ and (iii) times $L(j^*,\ell)$ to node $j^*$:

\begin{align}
    \label{eq:marginalization_trick_start}
    &\frac{1}{\tau} \cev{ \sum_{t\in L(j^*,\ell)} Q^*(j^*) - q(j^*,t) }{T(i,\ell) = \tau} = \frac{1}{\tau} \sum_{|\mathcal{L}| = \tau} \sum_{s=1}^{\tau} \sum_{|\mathcal{F}| = s} \hdots  \nonumber \\
    &\Bigg\{  \mathbb{E} \Bigg[ \sum_{t\in \mathcal{F}} Q^*(j^*) - q(j^*,t) \ | \ T(i,\ell) = \tau, L(i,\ell) = \mathcal{L}, T(j^*,\ell) = s, L(j^*,\ell) = \mathcal{F} \Bigg] \Bigg\} \nonumber \\ 
    &\cprob{L(j^*,\ell) = \mathcal{F}}{T(i,\ell) = \tau, L(i,\ell) = \mathcal{L}, T(j^*,\ell) = s} \nonumber \\
    & \cprob{T(j^*,\ell) = s}{T(i,\ell) = \tau, L(i,\ell) = \mathcal{L}}
    \cprob{L(i,\ell) = \mathcal{L}}{T(i,\ell) = \tau}
\end{align}

\noindent We will upper bound~\eqref{eq:marginalization_trick_start} by summing across an upper bound for the expectation term in braces. 
\begin{align}
    &\mathbb{E} \Bigg[ \sum_{t\in \mathcal{F}} Q^*(j^*) - q(j^*,t) \ | \ T(i,\ell) = \tau, L(i,\ell) = \mathcal{L}, T(j^*,\ell) = s, L(j^*,\ell) = \mathcal{F} \Bigg] \\
    &\leq \mathbb{E} \Bigg[ \sum_{t\in \mathcal{L}} Q^*(j^*) - q(j^*,t) \ | \ T(i,\ell) = \tau, L(i,\ell) = \mathcal{L}, \underline{T(j^*,\ell) = \tau}, \underline{L(j^*,\ell) = \mathcal{L}} \Bigg]
\end{align}
where we use the fact that for all $t \in \mathcal{L} \backslash \mathcal{F}$, $Q^*(j^*) - q(j^*,t) \geq 0$, and where the changed terms have been underlined for clarity.

The sum can now be simplified to the desired result~\eqref{eq:temp_assumption} with two observations: (i) there are no dependencies on $s$ or $\mathcal{F}$ so we pull the expectation out of the relevant summations, and the corresponding probability measures sum to one. (ii) Then, we write the total return as the number of visits times the average to reveal that, given $T(j^*,\ell)$, the expectation is independent of $T(i,\ell)$, $L(i,\ell)$, and $L(j^*,\ell)$, enabling us to sum the probabilities to one. 
\begin{align}
    &\frac{1}{\tau}\sum_{|\mathcal{L}| = \tau} 
    \mathbb{E} \Bigg[ \sum_{t\in \mathcal{L}} Q^*(j^*) - q(j^*,t) \ | \ T(i,\ell) = \tau, L(i,\ell) = \mathcal{L}, T(j^*,\ell) = \tau, L(j^*,\ell) = \mathcal{L} \Bigg] 
    \nonumber \\
    &\qquad \qquad \cprob{L(i,\ell) = \mathcal{L}}{T(i,\ell) = \tau} \\
    &= \frac{1}{\tau} \sum_{|\mathcal{L}| = \tau} 
    \mathbb{E} \Bigg[ \tau Q^*(j^*) - \tau Q(j^*,t) \ | \ T(j^*,\ell) = \tau \Bigg] 
    \cprob{L(i,\ell) = \mathcal{L}}{T(i,\ell) = \tau} \\
    \label{eq:marginalization_trick_stop}
    &= \cev{Q^*(j^*) - Q(j^*,\ell)}{T(j^*,\ell) = \tau}
\end{align}
\end{proof}

\subsection*{Supplemental Inductive Step: Value Concentration}

\begin{lemma}
\label{lemma:inductive_concentration}
Assume Assumption 1. Then exists constants $\csixi$, $\cseveni$, $\ceighti$ such that, for all $z$,
\begin{align}
    \cprob{|\ev{V(i,\ell)} - V(i,\ell)| \geq \frac{z}{\tau^{\csixi}}}{T(i,\ell)=\tau} \leq \frac{\cseveni}{z^{\ceighti}}
\end{align}
where $\cseveni = (6(\cfouri + \gamma \cfourj + \csevenj))^{\ceightj} + \frac{\underset{j \neq j^*}{\max} \xi_d(j) (b-1) (2(b-1)(1-\gamma^D))^{\ceighti}}{(1-\gamma^{d})^{\ceighti}}$, with $\xi_d(j)$ as in Lemma~\ref{lemma:induction_visits_tail}. 
\end{lemma}

\begin{proof}

We recall the definition of $V(i,\ell)$ and split the terms into the optimal and suboptimal arm contributions. Given $T(i,\ell)=\tau$: 
\begin{align}
    &\hspace{-1cm} 
    \tau \abs{\ev{V(i,\ell)} - V(i,\ell)} = 
    \abs{\ev{V(i,\ell)} - \frac{1}{\tau} \sum_{t\in L(i,\ell)}
        \sum_{j \in C(i)} \ind{\mu(i,t) = j} q(j,t) }
    \nonumber \\
    &\leq \underbrace{ \abs{ 
        \sum_{t\in L(i,\ell)}
        \sum_{j \in C(i), j \neq j^*} \ind{\mu(i,t) = j} (\ev{V(i,\ell)} - q(j,t)) }
    }_\text{I} 
    \nonumber \\ 
    &+ \underbrace{ \abs{ 
        \sum_{t\in L(i,\ell)}
        \ind{\mu(i,t) = j^*} (\ev{V(i,\ell)} - q(j^*,t)) 
    }}_\text{II} 
\end{align}
where we apply the triangle inequality. 

This implies, for arbitrary $\csixi$, that
\begin{align}
    &\cprob{| \ev{V(i,\ell)} - V(i,\ell) | \geq \frac{z}{\tau^{\csixi}}}{T(i,\ell) = \tau} \\
    &= \cprob{| \tau \ev{V(i,\ell)} - \tau V(i,\ell) | \geq \frac{\tau z}{\tau^{\csixi}}}{T(i,\ell) = \tau} \\
    &\leq \cprob{\text{I} + \text{II} \geq \frac{\tau z}{\tau^{\csixi}}}{T(i,\ell) = \tau}
    \label{eq:inductive_concentration_helper} \\
    &\leq \cprob{\text{I} \geq \frac{\tau z}{2 \tau^{\csixi}}}{T(i,\ell) = \tau} + 
    \cprob{\text{II} \geq \frac{\tau z}{2 \tau^{\csixi}}}{T(i,\ell) = \tau}
\end{align}

Term \text{I} is upper bounded using the Q value upper bound~\eqref{eq:bounded_qvalues}: 
\begin{align}
    \text{I} 
    &\leq \Bigg| \sum_{t \in L(i,\ell)} \sum_{j \in C(i), j \neq j^*} \ind{\mu(i,t) = j} \frac{1-\gamma^D}{1-\gamma^{d}} \Bigg| 
    = \Bigg| \sum_{j \in C(i), j \neq j^*} T(j,\ell) \frac{1-\gamma^D}{1-\gamma^{d}} \Bigg| \\
    &\leq \frac{1-\gamma^D}{1-\gamma^{d}} \Bigg| \sum_{\substack{j \in C(i) \\ j \neq j^*}} T(j,\ell) \Bigg|
    \leq \frac{1-\gamma^D}{1-\gamma^{d}} \sum_{\substack{j \in C(i) \\ j \neq j^*}} | T(j,\ell) |
    = \frac{1-\gamma^D}{1-\gamma^{d}} \sum_{\substack{j \in C(i) \\ j \neq j^*}} T(j,\ell)
\end{align}

{
\allowdisplaybreaks
and its concentration is controlled with Lemma~\ref{lemma:induction_visits_tail}:
\begin{align}
    &\cprob{\text{I} \geq \frac{\tau z}{2 \tau^{\csixi}}}{T(i,\ell) = \tau} \leq
    \cprob{\frac{1-\gamma^D}{1-\gamma^{d}} \sum_{j \in C(i), j \neq j^*}T(j,\ell) \geq \frac{\tau z}{2 \tau^{\csixi}}}{T(i,\ell) = \tau} \\
    &\quad \leq 
    \sum_{j\in C(i), j\neq j^*} 
    \cprob{T(j,\ell) \geq 
    \frac{\tau z (1-\gamma^{d})}{2(b-1)(1-\gamma^D)\tau^{\csixi}}
    }{T(i,\ell) = \tau} \\
    &\quad \leq 
    \sum_{j\in C(i), j\neq j^*} 
    \xi_d(j) 
    \tau^{\cthreei \ceighti / \ctwoi}
    \left( 
        \frac{z (1-\gamma^{d})}{2(b-1)(1-\gamma^D) \tau^{\csixi-1}}
    \right)^{-\ceighti} \\
    &\quad \leq 
    \frac{(b-1) (2(b-1)(1-\gamma^D))^{\ceighti}}{(1-\gamma^{d})^{\ceighti}} 
    \underset{j \neq j^*}{\max} \ \xi_d(j)
    \tau^{\cthreei \ceighti / \ctwoi + \csixi \ceighti - \ceighti}
    z^{-\ceighti} \\ 
    \label{eq:induction_concentration_3}
    &\quad \leq 
    \frac{(b-1) (2(b-1)(1-\gamma^D))^{\ceighti}}{(1-\gamma^{d})^{\ceighti}} 
    \underset{j \neq j^*}{\max} \ \xi_d(j) 
    z^{-\ceighti} 
\end{align}
where we use $\csixi \leq 1 - \cthreei/\ctwoi$. 
}

Term \text{II} is upper bounded by bounding cross terms with the triangle inequality: 
\begin{align}
    &\hspace{-1cm} \text{II} = \Bigg| \sum_{t \in L(j^*,\ell)} (\ev{V(i,\ell)} - q(j^*,t)) \Bigg| 
    \leq \underbrace{\sum_{t \in L(j^*,\ell)} \big|\ev{V(i,\ell)} - V^*(i)\big|}_\text{III} 
    \nonumber \\ 
    & + \underbrace{\sum_{t \in L(j^*,\ell)} \big|V^*(i) - \ev{q(j^*,t)}\big|}_\text{IV} 
    + \underbrace{\sum_{t \in L(j^*,\ell)} \big|\ev{q(j^*,t)} - q(j^*,t)\big|}_\text{V}
\end{align}

We upper bound terms $\text{III}$ and $\text{IV}$ using inductive convergence, and we use the previously derived marginalization procedure~\eqref{eq:marginalization_trick_start}~--~\eqref{eq:marginalization_trick_stop} to upper bound $\text{IV}$ and $\text{V}$
\begin{align}
    \text{III} &\leq \cfouri \tau^{1 - \cfivei} \\ 
    \text{IV} &\leq \tau (V^*(i) - \cev{Q(j^*,\ell)}{T(j^*,\ell) = \tau}) \\
    &= \tau \gamma (V^*(j^*) - \cev{V(j^*,\ell)}{T(j^*,\ell) = \tau}) 
    \leq \gamma \cfourj \tau^{1 - \cfivej} \\
    \text{V} &\leq \tau \big|\cev{Q(j^*,\ell)}{T(j^*,\ell) = \tau} - Q(j^*,\ell)\big| \\
    &=\tau \gamma \big| \cev{V(j^*,\ell)}{T(j^*,\ell) = \tau} - V(j^*,\ell) \big|
\end{align}

Now term $\text{II}$'s concentration can be controlled with the new terms: 
\begin{align}
    &\cprob{\text{II} \geq \frac{\tau z}{2 \tau^{\csixi}}}{T(i,\ell) = \tau} 
    \leq \cprob{\text{III} + \text{IV} + \text{V} \geq \frac{\tau z}{2 \tau^{\csixi}}}{T(i,\ell) = \tau} \\ 
    & \quad \leq \cprob{\text{III} \geq \frac{\tau z}{6 \tau^{\csixi}}}{T(i,\ell) = \tau} + \cprob{\text{IV} \geq \frac{\tau z}{6 \tau^{\csixi}}}{T(i,\ell) = \tau} + \cprob{\text{V} \geq \frac{\tau z}{6 \tau^{\csixi}}}{T(i,\ell) = \tau} \\
    & \quad \leq \cprob{\cfouri \tau^{-\cfivei} \geq \frac{z}{6 \tau^{\csixi}}}{T(i,\ell) = \tau} + \cprob{\gamma \cfourj \tau^{-\cfivej} \geq \frac{z}{6 \tau^{\csixi}}}{T(i,\ell) = \tau} \\
    & \qquad + \cprob{\big| \cev{V(j^*,\ell)}{T(j^*,\ell) = \tau} - V(j^*,\ell) \big| \geq \frac{z}{6 \tau^{\csixi}}}{T(i,\ell) = \tau} 
\end{align}
Now constraining $\csixi \geq \cfivei$ and $\csixi \leq \csixj$,
\begin{align}
    \label{eq:constantrules_inductiveconcentration_0}
    &\cprob{\text{II} \geq \frac{\tau z}{2 \tau^{\csixi}}}{T(i,\ell) = \tau} 
    \leq \prob{6\cfouri \geq z} + \prob{6\gamma \cfourj \geq z} + \frac{6^{\ceightj} \csevenj}{z^{\ceightj}}
\end{align}
Note that the events $\{6\cfouri \geq z\}$ and $\{6\gamma \cfourj \geq z\}$ are not random, therefore $\prob{6\cfouri \geq z} \leq \frac{(6\cfouri)^{\ceighti}}{z^{\ceighti}}$ and $\prob{6\gamma \cfourj \geq z} \leq \frac{(6\gamma \cfourj)^{\ceighti}}{z^{\ceighti}}$, then, constraining $\ceighti \leq \ceightj$,
\begin{align}
    &\cprob{\text{II} \geq \frac{\tau z}{2 \tau^{\csixi}}}{T(i,\ell) = \tau} \leq 
    \frac{(6\cfouri)^{\ceighti} + (6\gamma \cfourj)^{\ceighti} + 6^{\ceightj}\csevenj}{z^{\ceighti}} \label{eq:inductive_concentration_termii_bound}
\end{align}

The desired result is found by combining~\eqref{eq:inductive_concentration_termii_bound}~\eqref{eq:induction_concentration_3} and~\eqref{eq:inductive_concentration_helper}: 
\begin{align}
    &\cprob{| \ev{V(i,\ell)} - V(i,\ell) | \geq \frac{z}{\tau^{\csixi}}}{T(i,\ell) = \tau} \leq 
    \Bigg( (6\cfouri)^{\ceighti} + (6\gamma \cfourj)^{\ceighti} + 6^{\ceightj}\csevenj 
    \\ & \qquad + 
    \frac{\underset{j \neq j^*}{\max} \ \xi_d(j) (b-1) (2(b-1)(1-\gamma^D))^{\ceighti}}{(1-\gamma^{d})^{\ceighti}} 
    \Bigg) 
    \frac{1}{z^{\ceighti}}
    = \frac{\cseveni}{z^{\ceighti}}
\end{align}
where $\cseveni = (6\cfouri)^{\ceighti} + (6\gamma \cfourj)^{\ceighti} + 6^{\ceightj}\csevenj + \frac{\underset{j \neq j^*}{\max} \ \xi_d(j) (b-1) (2(b-1)(1-\gamma^D))^{\ceighti}}{(1-\gamma^{d})^{\ceighti}}$.

\end{proof}

\subsection*{Supplemental Constants Discussion}

Collecting all constraints on the constants, we arrive at a consistent set of rules the constants must satisfy,
and then propose the constants for our exploration law: $\conei$, $\ctwoi$, and $\cthreei$. 
Recall that all constants are non-negative: $c^d_{k} \geq 0$, $\forall k \in [1,8]$, $\forall d \in [1,D]$.

Base Case:
\begin{align}
    \cfourib &= (b-1) \left( \left( \frac{\coneib}{\Delta_\text{min}(i)} \right)^{1/\ctwoib} + 1 \right) \\
    \cfiveib &= \csixib = 1 - \frac{\cthreeib}{\ctwoib} \\
    \csevenib &= \left( (b-1) \left(\ \underset{j \neq j^*}{\max}\ \xi_{D-1}(j) + (2(b-1))^{\ceightib}\right) + (2\cfourib)^{\ceightib}\right) \\
    &\text{where }\xi_{D-1}(j) = \left( \left( \frac{\coneib}{\Delta(j)} \right)^{1/\ctwoib} + 1 \right)^{\ceightib} \\
    \ceightib &\text{ is arbitrary}
\end{align}

Inductive Case:
{
\allowdisplaybreaks
\begin{align}
    \cfouri &= (b-1) \left(\frac{1 - \gamma^D}{1 - \gamma^{d}} \right) \left(\left( \frac{2 \conei}{ \Delta_\text{min}(i) } \right)^{\frac{1}{\ctwoi}} 
    + \left( \frac{2 \gamma \cfourj}{\conei} \right)^\frac{1}{\cthreei + \csixj - \ctwoi} + 1 + \frac{\csevenj (2\gamma)^{\ceightj} \pi^2}{6(\conei)^{\ceightj}}\right) + \gamma \cfourj \\
    \cfivei &= \csixi = 1 - \frac{\cthreei}{\ctwoi} \\
    \cseveni &= 6\cfouri)^{\ceighti} + (6\gamma \cfourj)^{\ceighti} + 6^{\ceightj}\csevenj + \frac{\underset{j \neq j^*}{\max} \ \xi_d(j) (b-1) (2(b-1)(1-\gamma^D))^{\ceighti}}{(1-\gamma^{d})^{\ceighti}} \\
    &\text{where } \xi_d(j) = \left( \left( \left( \frac{2 \conei}{ \Delta(j) } \right)^{\frac{1}{\ctwoi}} 
    + \left( \frac{2 \gamma \cfourj}{\conei} \right)^\frac{1}{\cthreei + \csixj - \ctwoi} + 1 \right)^{\ceighti} +\frac{\csevenj (2\gamma)^{\ceightj}}{(\conei)^{\ceightj}(\ceightj(\cthreei + \csixj - \ctwoi) - 1)}\right) \\
    \ceighti &= \ceightj(\cthreei + \csixj - \ctwoi) - 1 \label{eq:ceight_induction_rule_1}\\
    &\ceightj(\cthreei + \csixj - \ctwoi) \geq 2 \label{eq:ceight_induction_rule_2} \\
    \ctwoi &\geq \csixj
\end{align}
}

Selecting our exploration law as
$c_1 = 1$, 
$c_3 \in [\frac{1}{4}, \frac{1}{2})$, $c_2 = 2c_3$
yields $\cfivei = \csixi = \frac{1}{2}$ for all $d \in [1,D]$.
As $\ceightib$ is arbitrary, we pick $\ceightib = \frac{2}{(0.5 - c_3)^D} + \sum_{d=1}^{D-1} \frac{1}{(0.5 - c_3)^d}$ so that Equations~\eqref{eq:ceight_induction_rule_1} and~\eqref{eq:ceight_induction_rule_2} are satisfied for all $d \in [1,D]$.
In practice, we found values of $c_3$ closer to $\frac{1}{2}$ (and consequently $c_2$ near $1$) converged faster, and we used exactly those values in our experiments.

\subsection*{Supplemental Analysis: Utilities}

We have the following supporting Lemmas. 

\begin{lemma}
\label{lemma:minmax_helper}
Let $X$, $Y$ be compact sets and let $f: X \times Y \rightarrow \mathbbm{R}$ and $g: X \times Y \rightarrow \mathbbm{R}$ be functions such that $f(x,y) \leq g(x,y)$, $\forall (x,y) \in (X,Y)$. Then: 
\begin{align}
    \max_{x \in X}( \min_{y \in Y} f(x,y)) 
    \leq \max_{x \in X} (\min_{y \in Y} g(x,y)) 
\end{align}
\end{lemma}
\begin{proof}
Define $f_s(x) = \min_{y \in Y} f(x,y)$, $g_s(x) = \min_{y \in Y} g(x,y)$.

Consider, for fixed $x$, the two sets
$\{ f(x,y) \ | \ y \in Y \}$ and $\{ g(x,y) \ | \ y \in Y \}$.
Note that each element in the first set has a corresponding element in the second set that is at least as large as it since $f(x,y)\leq g(x,y) $ for all $(x,y) \in (X,Y)$. As such, the minimum element of the first set is less than or equal to that of the second set: $\min\{ f(x,y) \ | \ y \in Y \} \leq \min\{ g(x,y) \ | \ y \in Y \}$, and consequently $f_s(x) \leq g_s(x)$ for all $x \in X$.

Now, consider the sets $\{ f_s(x) \ | \ x \in X \}$ and $\{ g_s(x) \ | \ x \in X \}$.
Each element in the first set has a matching element in the second that is at least as large as it, again since $f(x,y) \leq g(x,y) $ for all such pairs $(x,y) \in (X,Y)$.
As such, the maximum element in the first set is less than or equal to that of the second set:
$\max\{ f_s(x) \ | \ x \in X \} \leq 
\max\{ g_s(x) \ | \ x \in X \}$, and consequently $\max_{x \in X} f_s(x) \leq \max_{x \in X} g_s(x)$. 

Plugging in the definition of $f_s$ and $g_s$ completes the proof.
\end{proof}

\begin{lemma}
\label{lemma:dare_linear_contraction}
Consider a dynamically feasible reference trajectory expressed as a linear system 
 of $\mathbf{z}^{\mathrm{ref}}_{[H]}$, $\action^{\mathrm{ref}}_{[H]}$, and a feedback controller: 
\begin{align}
    \mathbf{z}_{k+1} &= A \mathbf{z}_{k} + B \action_{k+1} + c \\ 
    \action_{k+1} &= \action^\mathrm{ref}_{k+1} - \mathcal{K} (\mathbf{z}_k - \mathbf{z}^\mathrm{ref}_k)
\end{align}
If the gain matrix $\mathcal{K}$ is selected as $(\Gamma_u + B^\top M B)^{-1} (B^\top M A)$, where $M$ solves the discrete algebraic Riccati equation (DARE) and $\Gamma_u \succ 0$, $\Gamma_x \succ 0$ then the system is contracting~\cite{LohmillerS98} at some rate $\alpha\in[0,1)$. 
\end{lemma}
Here we note that if $(A, B)$ is stabilizable over the entire compact state space, $M$ is bounded as $\underline{m}I \preceq M \preceq \overline{m}I$ for $0 < \underline{m}, \overline{m} < \infty$.

\begin{proof}

Combine the system and controller to write the closed loop system and its differential dynamics: 
\begin{align}
    \mathbf{z}_{k+1} &= A \mathbf{z}_{k} + B \action^\mathrm{ref}_{k+1} - B \mathcal{K} (\mathbf{z}_k - \mathbf{z}^\mathrm{ref}_k) + c \\ 
    \delta \mathbf{z}_{k+1} &= \underbrace{(A - B \mathcal{K})}_{A_\text{cl}} \delta \mathbf{z}_k
\end{align}

Recall the definition of discrete-time contraction from~\cite{TsukamotoCS21}: A discrete-time system is contracting with time-invariant metric $M$ iff 
\begin{align}
    \label{eq:discrete_time_contraction}
    A_\text{cl}^\top M A_\text{cl} - \alpha^2 M \preceq 0
\end{align}
where $A_\text{cl}$ is the closed loop dynamics of the differential system and $\alpha \in [0,1)$ is the contraction rate. 

Compute the left-hand-side of the contraction condition, and we will seek to show it is negative definite to prove contraction~\eqref{eq:discrete_time_contraction}: 
\begin{align}
    \textup{LHS} &= (A - B \mathcal{K})^\top M (A - B \mathcal{K}) - \alpha^2 M \\ 
    &= A^\top M A - A^\top M B \mathcal{K} - \mathcal{K}^\top B^\top M A + \mathcal{K}^\top B^\top M B \mathcal{K} - \alpha^2 M 
\end{align}

Manipulate this term, $\mathcal{K}^\top B^\top M B \mathcal{K}$, by plugging in the definition of $\mathcal{K}$, add/subtracting $\mathcal{K}^\top \Gamma_u \mathcal{K}$ and grouping terms:
\begin{align}
    \mathcal{K}^\top B^\top M B \mathcal{K} &= \mathcal{K}^\top B^\top M B (\Gamma_u + B^\top M B)^{-1} (B^\top M A) \\ 
    &= \mathcal{K}^\top B^\top M B (\Gamma_u + B^\top M B)^{-1} (B^\top M A) + \mathcal{K}^\top \Gamma_u \mathcal{K} - \mathcal{K}^\top \Gamma_u \mathcal{K} \\ 
    &= \mathcal{K}^\top (\Gamma_u + B^\top M B) (\Gamma_u + B^\top M B)^{-1} (B^\top M A) - \mathcal{K}^\top \Gamma_u \mathcal{K} \\ 
    &= \mathcal{K}^\top B^\top M A - \mathcal{K}^\top \Gamma_u \mathcal{K} 
\end{align}

Plug back into LHS: 
\begin{align}
    \textup{LHS} &= A^\top M A - A^\top M B \mathcal{K} - \mathcal{K}^\top \Gamma_u \mathcal{K} - \alpha^2 M 
\end{align}

Recall the definition of the discrete algebraic Riccati equation, $\text{DARE}(A, B, \Gamma_x, \Gamma_u)$~\cite{kirk2012optimal}:
\begin{align}
    M &= A^\top M A - A^\top M B (\Gamma_u + B^\top M B)^{-1} (B^\top M A) + \Gamma_x 
\end{align}

Let $M$ solve $\text{DARE}(A, B, \Gamma_x, \Gamma_u)$. Plug into LHS: 
\begin{align}
    \textup{LHS} &= - \Gamma_x - \mathcal{K}^\top \Gamma_u \mathcal{K} + (1 - \alpha^2) M 
\end{align}

Because $\mathcal{K}^\top \Gamma_u \mathcal{K} \succeq 0$, if we select $\Gamma_x$ to be sufficiently large, the system is contracting with rate $\alpha$. 

\end{proof}

\subsection*{Supplemental Table of Mathematical Symbols}

We provide a table of mathematical symbols to help presentation. 
Some symbols are repeated due to the large number of examples, although their meaning should be clear from context. 
\begin{longtable}{|l|l|}
\caption{Mathematical Symbols and Their Meanings}
\label{tab:symbols}
\\ \hline
\textbf{Symbol} 
& \textbf{Meaning} 
\\ \hline
$\state$ 
& State of the system 
\\ \hline 
$\action$ 
& Action of the system 
\\ \hline 
$\statespace$ 
& State space of the system 
\\ \hline 
$\actionspace$ 
& Action space of the system 
\\ \hline
$F$ 
& Discrete-time dynamics of the system 
\\ \hline
$R$ 
& Stage reward of the system 
\\ \hline
$D$ 
& Terminal reward of the system 
\\ \hline
$\Omega$ 
& Set of state constraints
\\ \hline
$K$ 
& Time horizon of problem
\\ \hline
$\gamma$ 
& Discount factor
\\ \hline
$\Delta t$ 
& Discretization step-size of continuous time dynamics 
\\ 
\hline
$c_{1,2,3}$ 
& Constants in exploration law of MCTS
\\ \hline
$p$ 
& Path of nodes in current rollout
\\ \hline
$H$ 
& Number of timesteps in trajectory generated by SETS
\\ \hline
$\textup{MDP}$ 
& Markov Decision Process problem data   
\\ 
\hline
$\ell$ 
& Number of simulated trajectories in MCTS tree
\\ \hline
$i,j$ 
& Nodes in MCTS tree
\\ 
\hline
$\tilde{V}(i, \ell)$ 
& Total value of node $i$ after rollout number $\ell$. 
\\ \hline
$V(i, \ell)$ 
& Average value of node $i$ after rollout number $\ell$.
\\ 
\hline
$T(i, \ell)$ 
& Total number of visits to node $i$ after rollout number $\ell$. 
\\ \hline
$C(i)$ 
& Set of children of node $i$ 
\\ \hline
$r(i)$ 
& Stage reward to node $i$
\\ \hline
$A_k, B_k, c_k$ 
& Local linearization of dynamics $F$
\\ \hline
$\mathbf{z}$ 
& State evolving with locally linear dynamics
\\ \hline
$S$ 
& Input normalization linear transform 
\\ \hline
$\mathcal{C}$ 
& Controllability matrix
\\ \hline
$\mathbf{v}_i, \lambda_i$ 
& Spectrum of controllability Gramian
\\ \hline
$\text{DARE}$ 
& Discrete Algebraic Riccati Equation
\\ \hline
$M_k$ 
& Contraction Metric
\\ \hline
$\Gamma_{x,u}$ 
& State/Input Regularization Matrix
\\ \hline
$\mathcal{K}_k$ 
& Feedback Gain
\\ \hline
$F^H, R^H$ 
& $H$-step dynamics and reward function. 
\\ \hline
$V^*$ 
& Optimal value function
\\ \hline
$C^2$ 
& Space of twice differentiable functions
\\ \hline
$\text{Lip}_1$ 
& Space of Lipschitz functions
\\ \hline
$\xi_H$ 
& Elements of the nonlinear and continuous reachable set. 
\\ \hline
$\eta_H$ 
& Elements of the linear and continuous reachable set. 
\\ \hline
$\kappa_{0,1,2,3,4}$ 
& Constants in theoretical results. 
\\ \hline
$p_{n,e,d}$
& North, east, down position of six-DOF dynamical system 
\\ \hline
$u,v,w$
& Velocities of six-DOF dynamical system 
\\ \hline
$\phi, \theta, \psi$
& Roll, pitch, yaw of six-DOF dynamical system 
\\ \hline
$p, q, r$
& Roll, pitch, yaw rates of six-DOF dynamical system 
\\ \hline
$\xi^i_k$
& Number of timesteps since viewing target $i$. 
\\ \hline
$\mathcal{G}^i$
& Goal region for $i$th target. 
\\ \hline
$f_{w,x} f_{w,y} f_{w,z}, n_w, m_w, l_w$
& Forces and torques for external wind effects. 
\\ \hline
$x,y,\theta,v,\omega$
& State of tracked vehicle. 
\\ \hline
$v_d, \omega_d$
& Action of tracked vehicle. 
\\ \hline
$\tau_v, \tau_w$
& Timescales in tracked vehicle dynamics. 
\\ \hline
$\{a_1,a_2,a_3,a_4\}$
& Adaptive control parameters in tracked vehicle
\\ \hline
$(p_x^i,p_y^i,v_x^i,v_y^i)$
& State of each particle in spacecraft experiment
\\ \hline
$F_{\text{net},i}$
& Net tension force on net segment $i$. 
\\ \hline
$l_n, k_n, c_n$
& \makecell[l]{Net parameters: nominal length, spring constant, \\ and damping constant.}
\\ \hline
$F_{\text{contact},i}$
& Contact force of net segment $i$. 
\\ \hline
$r_t, k_c, c_c$
& \makecell[l]{Contact force parameters: target radius, \\ spring constant, damping constant.}
\\ \hline
$s(d,a)$ 
& Normalization function. 
\\ \hline
$\delta_e, \delta_a, \delta_r$ 
& Actions for glider system: elevator, aileron, rudder deflections. 
\\ \hline
$V_a$ 
& Relative wind speed
\\ \hline
$S, c, b$ 
& \makecell[l]{Planar area of wind surface, mean chord length, \\ and wind span of the drone.}
\\ \hline
$C_L, C_D, C_m, C_Y, C_l, C_n$ 
& Aero model for lift, drag, and yaw force, and external moments. 
\\ \hline
$\alpha, \beta$ 
& Angle of attack and slip. 
\\ \hline
$p_{x,g}, p_{y,g}, p_{z,g}$ 
& Goal x,y,z position. 
\\ \hline
$\mathcal{T}_{1,2}$
& Bellman value operators for input sets $U_1$, $U_2$
\\ \hline
$L_V$
& Lipschitz constant of optimal value function
\\ \hline
$L_Q$
& Lipschitz constant of the state-action function
\\ \hline
$\varepsilon_H$
& \makecell[l]{Maximum distance between linearization point \\ and rollout up to horizon $H$}
\\ \hline
$\sigma_\text{max}(\cdot)$
& Largest singular value of matrix
\\ \hline
$L_{\nabla F}$
& Lipschitz constant of dynamics gradient
\\ \hline
$\Theta_k$
& Matrix decomposition of contraction metric $M_k$
\\ \hline
$b$
& Branching factor
\\ \hline
$\mu(i,\ell)$
& Child selection of node $i$ at time $\ell$
\\ \hline
$\Delta(j)$
& Gap between suboptimal node $j$ and optimal sibling node
\\ \hline
$\Delta_\text{min}(i)$
& Smallest gap among all suboptimal children of node $i$
\\ \hline
$q(i, \ell)$ 
& State-action of node $i$ associated with rollout $\ell$
\\ \hline
$Q(i, \ell)$ 
& Average state-action of node $i$ after rollout $\ell$
\\ \hline
$Q^*(i)$ 
& Optimal state-action of node $i$
\\ \hline
$L(i,\ell)$
& Visit times to node $i$, when the root has been visited $\ell$ times
\\ \hline
$c_i^d, \xi_d(j)$
& \makecell[l]{Convergence and concentration rates for node $j$ at depth $d$ \\ $i=1, \dots, 8$}
\\ \hline
$c_i^{D-1}, \xi_{D-1}(j)$
& \makecell[l]{Convergence and concentration rates for leaf nodes $j$ \\ $i=1, \dots, 8$}
\\ \hline
$\overline{m}, \underline{m}$
& \makecell[l]{Upper and lower bounds on the eigenvalues \\ of the contraction metric}
\\ \hline
$\alpha$
& Contraction rate
\\ \hline
\end{longtable}

\end{document}